\documentclass{article}

 \usepackage[main, final]{neurips_2026}

\usepackage[utf8]{inputenc} %
\usepackage[T1]{fontenc}    %
\usepackage{hyperref}       %
\usepackage{url}            %
\usepackage{booktabs}       %
\usepackage{amsfonts}       %
\usepackage{nicefrac}       %
\usepackage{microtype}      %
\usepackage{xcolor}         %

\usepackage{graphicx}
\usepackage{subfigure}

\usepackage{microtype}
\usepackage{subcaption}
\usepackage{booktabs} %

\usepackage{hyperref}
\usepackage{epsfig,tikz,caption}

\usepackage{abbreviations}
\usepackage{amsmath}
\usepackage{amssymb}
\usepackage{mathtools}
\usepackage{amsthm}

\usepackage[capitalize,noabbrev]{cleveref}
\theoremstyle{plain}
\newtheorem{theorem}{Theorem}[section]
\newtheorem{proposition}[theorem]{Proposition}
\newtheorem{lemma}[theorem]{Lemma}
\newtheorem{corollary}[theorem]{Corollary}
\theoremstyle{definition}

\newtheorem{assumption}[theorem]{Assumption}
\theoremstyle{remark}
\newtheorem{remark}[theorem]{Remark}

\usepackage[linesnumbered,ruled,vlined]{algorithm2e}  %


\title{Stochastic Matching Bandits\\ with Rare Optimization Updates}

\author{%
Jung-hun Kim \\
  FairPlay joint team, CREST, ENSAE\\
  Institut Polytechnique de Paris\\
  \texttt{junghun.kim@ensae.fr} \\
\And 
Min-hwan Oh \\
  Seoul National University\\
  \texttt{minoh@snu.ac.kr} \\
}

\begin{document}

\maketitle

\begin{abstract}
We introduce a bandit framework for stochastic matching under the
multinomial logit (MNL) choice model.
In this setting, $N$ agents on one side are assigned to $K$ arms on the other side,
where each arm stochastically selects an agent from its assigned pool according
to unknown preferences and yields a corresponding reward over a horizon $T$.
The objective is to minimize regret by maximizing the cumulative revenue from
successful matches.
A naive approach requires solving an NP-hard combinatorial optimization problem
at every round, resulting in a prohibitive computational cost.
To address this challenge, we propose batched algorithms that strategically limit the number of times matching assignments are updated to $\Theta(\log\log T)$ over the entire horizon.
By invoking expensive combinatorial optimization only on a vanishing fraction of rounds, our algorithms substantially reduce overall computational overhead while still achieving a regret bound of~$\widetilde{\mathcal{O}}(\sqrt{T})$.

\end{abstract}

\section{Introduction}

In recent years, the rapid growth of matching markets—such as ride-hailing platforms, online job boards, and labor marketplaces—has underscored the importance of {maximizing revenue} from successful matches. For example, in ride-hailing services, the platform seeks to match riders (agents) with drivers (arms) in a way that maximizes total revenue generated from completed rides.

This demand has led to extensive research on online bipartite matching problems~\citep{karp1990optimal,mehta2007adwords,mehta2013online,gamlath2019online,fuchs2005online,kesselheim2013optimal}, where two sets of vertices are considered and one side is revealed sequentially. These studies primarily focus on {maximizing the number of matches}. However, a significant gap remains between these theoretical models and practical scenarios for maximizing revenue under latent preference functions.  Specifically, these models generally assume one-to-one assignments under deterministic matching and focus solely on match count, without incorporating \emph{learning mechanisms} that adapt to observed reward feedback or aim to maximize cumulative revenue.

More recently, {stable matching bandits} have been proposed to capture
online learning dynamics in matching markets~\citep{liu2020competing,liu2021bandit,sankararaman2020dominate,basu2021beyond,zhang2022matching,kong2023player}.
In this line of work, agents are assigned to arms in each round, and the
learning objective is to infer the underlying preference/reward structure so as
to identify a stable matching~\citep{mcvitie1971stable}.

A common modeling abstraction in this literature is that, when multiple agents
are tentatively assigned to an arm, the arm resolves the contention according
to a fixed known preference order and therefore selects a match in
a deterministic manner.
This abstraction is well-suited for stability-driven formulations.
In many platform settings, however, acceptance decisions can be inherently
stochastic and reflect latent preferences. For instance, when a dispatch system
offers a driver multiple ride requests, the driver may choose among them
probabilistically rather than following a fixed rule.

Motivated by such applications, we take a complementary perspective and
introduce \emph{stochastic matching bandits} (SMB), which explicitly model
stochastic arm-side choice under unknown preferences and target revenue
maximization. In SMB, multiple agents may be assigned to the same arm in a
round, and the arm then stochastically selects (at most) one agent from its
assigned pool. This shift in modeling leads to a different and practically
motivated objective: maximizing the platform’s cumulative system-wide revenue. 

However, realizing this goal comes with substantial computational challenges.
In particular, determining the optimal assignment at each round requires solving
an NP-hard combinatorial optimization problem.
As a result, naive approaches must repeatedly re-solve this expensive optimization
to recompute the set of assortments or matchings used in subsequent rounds,
leading to a substantial overall
computational burden.
These considerations naturally raise the following fundamental question:
\begin{center}
\textit{Can we achieve no-regret learning to maximize revenue in stochastic
matching bandits while avoiding combinatorial optimization at every round?}
\end{center}

To address this challenge, we propose \emph{batched} algorithms for the SMB
framework that strategically limit the number of times matching assignments are
updated.
By invoking expensive combinatorial optimization only a small number of times
over the entire horizon, our algorithms achieve no-regret performance while
substantially reducing the overall computational burden.
Below, we summarize our main contributions.

\textbf{Summary of Our Contributions.} 

\begin{itemize}
    \item
    We introduce a practical framework of \emph{stochastic matching
    bandits} (SMB), which models stochastic arm-side choices under latent
    preferences and targets revenue maximization. While this framework captures
    important features of real-world matching systems, a naive implementation
    requires solving an NP-hard combinatorial optimization problem at every
    round, resulting in a substantial overall computational burden.

    \item
    Under the SMB framework, we develop a batched algorithm that balances
    exploration and exploitation with limited matching updates. Assuming
    knowledge of a non-linearity parameter $\kappa$, the algorithm achieves
    $\widetilde{\mathcal{O}}(\sqrt{T})$ regret while invoking expensive
    combinatorial optimization only $\Theta(\log\log T)$ times over the entire
    horizon, substantially reducing the overall computational burden compared
    to naive per-round optimization.

    \item
    We further propose a second algorithm that removes the requirement of knowing
    $\kappa$ in advance. This algorithm retains the same
    $\widetilde{\mathcal{O}}(\sqrt{T})$ regret guarantee and still requires only
    $\Theta(\log\log T)$ batch updates, preserving the key property that
    combinatorial optimization is performed only a small number of times over
    the horizon.
\end{itemize}

\section{Related Work}\label{sec:related_work}

\paragraph{Stable Matching Bandits.} 
We review the literature on matching bandits, which studies regret minimization in matching markets. This line of work was initiated by \citet{liu2020competing} and extended by \citet{sankararaman2020dominate,liu2021bandit,basu2021beyond,zhang2022matching,kong2023player}, focusing on finding optimal stable matchings through stochastic reward feedback, under  assumption that the number of agents does not exceed the number of arms ($N \leq K$).

Our proposed \emph{Stochastic Matching Bandits (SMB)} framework departs from this
literature in several key ways.
First, whereas prior work typically assumes that arms select agents
\emph{deterministically} according to known preferences, SMB models arms as making
\emph{stochastic} choices driven by unknown latent preferences that must be
learned over time. This shift moves the objective from identifying a static
stable matching to maximizing cumulative reward through adaptive learning.
Second, SMB captures richer preference structures by allowing arm-side choices
to depend on the \emph{composition of the assigned set}, so that selection
probabilities and expected rewards are determined by the \emph{relative
utilities of agents competing for the same arm}.
Third, SMB imposes no structural restrictions on market size, naturally
accommodating both $N \leq K$ and $N \geq K$ regimes.

Taken together, these modeling choices define a distinct perspective on matching
under uncertainty and broaden the applicability of SMB to
real-world platforms---such as ride-hailing and online marketplaces---where
preferences are stochastic, assignment-dependent, and inherently relative.

 \paragraph{MNL-Bandits.} In our study, we adopt the Multinomial Logit (MNL) model for arms' choice preferences in matching bandits.
 As the first MNL bandit method, \citet{agrawal2017mnl} proposed an epoch-based algorithm, followed by subsequent contributions from \citet{agrawal2017thompson, chen2023robust,oh2019thompson,oh2021multinomial,lee2024nearly}.
 However, unlike selecting an assortment at each time step, our framework for the stochastic matching market mandates choosing at most $K$ distinct assortments to assign agents to each arm. Consequently, handling $K$-multiple MNLs simultaneously results in exponential computational complexity. More recently, \citet{kim2024queueing} studied MNL-based preferences in matching bandits; however, their focus was on system stability under binary $(0/1)$ rewards, rather than revenue maximization. Additionally, their work did not address the computational intractability arising from per-round combinatorial optimization in this context.

\begin{figure*}
\centering     %
\subfigure[Each agent--arm pair is associated with an unknown utility.]
{\label{fig:a}\includegraphics[width=32mm]{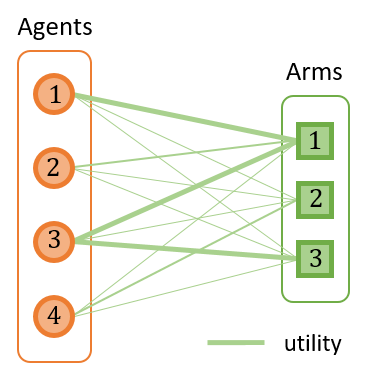}}
\hspace{0.5cm}
\subfigure[Each agent is assigned to an arm by a policy.
] {\label{fig:b}\includegraphics[width=30mm]{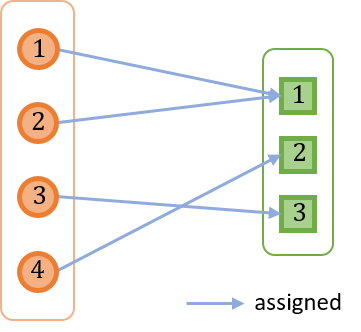}}
\hspace{0.5cm}
\subfigure[Each arm \textit{stochastically} accepts at most one assigned agent.]
{\label{fig:c}\includegraphics[width=30mm]{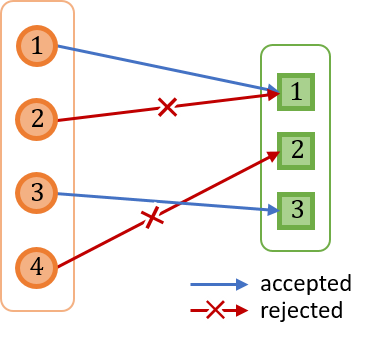}}
\caption{Illustration of the stochastic matching process with 4 agents ($N = 4$) and 3 arms ($K = 3$).}\label{fig:problem}
\end{figure*}
\paragraph{Batch Learning in Bandits.}
Batch learning in bandit problems has been extensively studied in the context of
multi-armed bandits (MAB) \citep{perchet2015batched, gao2019batched}, and later extended
to (generalized) linear bandit models
\citep{ruan2021linear, hanna2023contexts, han2020sequential, ren2024dynamic, sawarnigeneralized, ren2024optimal,chen2023robust}.
More recently, \citet{midigeshi2025achieving} study batched updates
for multinomial logistic bandits.
However, their setting considers a single-item selection problem, where the learner
chooses one arm at each round, and thus does not involve combinatorial action spaces.

Lastly, \citet{kim2025oracle} study combinatorial semi-bandits with rare update policies.
Their analysis assumes that each action is associated with a fixed (unknown) mean reward,
independent of other simultaneously selected actions.
In contrast, in the MNL bandit setting considered here, the expected reward of an
agent--arm match is a nonlinear and set-dependent function of the entire proposed assignment,
due to stochastic arm-side choice probabilities.
As a result, their framework does not capture the reward coupling induced by MNL-based
matching markets.

\section{Problem Statement}\label{sec:prob}
We study stochastic matching bandits (SMB) with $N$ agents and $K$ arms. The stochastic matching process is illustrated in Figure~\ref{fig:problem}.
 The detailed formulation is as follows:
For each agent $n\in[N]$, feature information is known as $x_{n}\in\mathbb{R}^{d}$, and each arm $k\in[K]$ is characterized by latent vector $\theta_{k}\in\mathbb{R}^{d}$. 
  We define the set of features as $X=[x_1,\dots,x_N]\in\mathbb{R}^{d\times N}$ and the rank of $X$ as $rank(X)=r(\le d)$. At each time $t\in[T]$, each agent $n$ may be assigned to an arm $k_{n,t}\in[K]$. Let assortment $S_{k,t}=\{n\in[N]: k_{n,t}=k\}$, which is the set of agents that are assigned to an arm $k$ at time $t$. Then given an assortment to each arm $k$ at time $t$, $S_{k,t}$, each arm $k$ randomly accepts an agent $n\in S_{k,t}$ and receives reward $w_{n,k}\in [0,1]$ according to the arm's preference specified as follows. 
 The probability for arm $k$ to accept  agent $n\in S_{k,t}$   follows Multinomial Logit (MNL) model \citep{agrawal2017mnl,agrawal2017thompson,oh2019thompson,oh2021multinomial,chen2023robust}  given by \[p(n|S_{k,t},\theta_k)=\frac{\exp(x_{n}^\top\theta_k)}{1+\sum_{m\in S_{k,t}}\exp(x_{m}^\top \theta_k)}.\]
  We denote $x_n^\top \theta_k$ as the latent preference utility of arm $k$ for agent $n$. Following prior work on MNL bandits~\citep{oh2019thompson, oh2021multinomial, agrawal2019mnl}, we consider that the candidate set size is bounded by $|S_{k,t}| \leq L$ for all arms $k$ and rounds $t$, and that the reward $w_{n,k}$ is known to the learner. 

  At each time step $t$, the agents receive stochastic feedback based on the assortments $\{S_{k,t}\}_{k \in [K]}$. Specifically, for each agent $n \in S_{k,t}$ and arm $k \in [K]$, the feedback is denoted by $y_{n,k,t} \in \{0,1\}$, where $y_{n,k,t} = 1$ if arm $k$ accepts agent $n$ (i.e., a successful match occurs), and $y_{n,k,t} = 0$ otherwise. Following the standard MNL model, each arm $k$ may also choose an outside option $n_0$ (i.e., reject all assigned agents) with probability
$p(n_0|S_{k,t},\theta_k)=1/(1+\sum_{m\in S_{k,t}}\exp(x_m^\top\theta_k))$, which yields zero reward.
  Then, given assortments to every arm $k$, $\{S_k\}_{k\in [K]}$, the expected reward (revenue) for the assortments at time $t$ is defined as 
\begin{align*}
\sum_{k\in[K]}R_k(S_{k})&:=\sum_{k\in[K]}\sum_{n\in S_k}w_{n,k}p(n|S_{k},\theta_k).
\end{align*}
The goal of the problem is to maximize the cumulative expected reward over a time horizon $T$ by learning the unknown parameters $\{\theta_k\}_{k \in [K]}$. We define the oracle strategy as the optimal assortment selection when the preference parameters $\theta_k$ are known. Let the set of all feasible assignments be: $\mathcal{M}=\{\{S_{k}\}_{k\in[K]}:S_{k}\subset[N], |S_k|\le L\:\;\forall k\in[K],S_{k}\cap S_{l}=\emptyset\:\; \forall k\neq l\}$. Then the oracle assortment is given by: \[\{S_k^*\}_{k\in[K]}=\argmax_{\{S_{k}\}_{k\in[K]}\in\mathcal{
M}}\sum_{k\in[K]}R_k(S_k).\]  
Given $\{S_{k,t}\}_{k\in[K]}\in\mathcal{M}$ for all $t\in[T]$, the expected cumulative regret is defined as 
\[\mathcal{R}(T)=\mathbb{E}\Big[\sum_{t\in[T]}\sum_{k\in[K]}R_k(S_k^*)-R_k(S_{k,t})\Big].\]
The objective is to design a policy that minimizes this regret over the time horizon $T$.

 Similar to previous work for logistic and MNL bandit \citep{oh2019thompson,oh2021multinomial,lee2024nearly,goyal2021dynamic,faury2020improved,abeille2021instance}, we consider the following regularity condition and non-linearity quantity.

\begin{assumption}
$\|x_n\|_2\le 1$ for all $n\in[N]$ and $\|\theta_k\|_2\le 1$ for all $k\in[K]$.\label{ass:bd}
\end{assumption}

 Then we define a problem-dependent quantity regarding non-linearity of the MNL structure as follows:
\[\kappa:=\inf_{ \substack{\theta\in\mathbb{R}^d:\|\theta\|_2\le 1\\ n\in S\subseteq [N]: |S|\le L}}p(n|S,\theta)p(n_0|S,\theta).\]

\section{Optimization in Stochastic Matching Bandits: The Curse of Complexity}\label{sec:curse_complexity}
In this work, we develop algorithms for the Stochastic Matching Bandit (SMB) problem with preference feedback. SMB can be viewed as a generalization of the standard Multinomial Logit (MNL) bandit model with a single assortment~\citep{oh2021multinomial, lee2024nearly} to a setting with $K$ simultaneous assortments—one for each arm. Applying existing MNL-based methods to this setting requires dynamically selecting $K$ assortments at each round while simultaneously learning arm preferences in an online fashion. This extension introduces significant computational challenges: the resulting combinatorial optimization problem is NP-hard. In contrast, the standard MNL bandit problem with a single assortment is known to be solvable in polynomial time~\citep{oh2021multinomial}. Thus, the SMB framework poses a substantially more complex optimization problem, highlighting the need for efficient algorithmic solutions.

Naively extending MNL bandits (e.g. \citet{oh2021multinomial, lee2024nearly}) to SMB requires defining the UCB index for the expected reward of an assortment $S_{k}$ for all $k\in[K]$ as $R^{UCB}_{k,t}(S_{k})=\sum_{n\in S_{k}}\frac{w_{n,k}\exp(h_{n,k,t})}{1+\sum_{m\in S_{k}}\exp (h_{m,k,t})},$ where $h_{n,k,t}$ is an UCB index for the utility value between $n$ and $k$ at each time $t$.
Then at each time, the algorithm determines assortments by following the UCB strategy: 
\begin{align}    \label{eq:ucb_comb}
    \{S_{k,t}\}_{k\in[K]}=\argmax_{\{S_{k}\}_{k\in[K]}\in\mathcal{M}}\sum_{k\in[K]}R^{UCB}_{k,t}(S_{k}).
\end{align} 
 While this method can achieve a regret bound of $\tilde{\mathcal{O}}( \sqrt{T})$, it suffers from severe computational limitations. Let $C_{\mathrm{opt}}$ denote the worst-case cost of solving~\eqref{eq:ucb_comb} once.
Since~\eqref{eq:ucb_comb} is NP-hard in general, exact maximization can be computationally prohibitive. A trivial worst-case upper bound is obtained by exhaustive search over feasible assignments:
when $L\ge N$, the action space has size $\mathcal{O}(K^N)$. Consequently, recomputing the maximizer at every round leads to total optimization cost
on the order of $T\cdot C_{\mathrm{opt}}=\mathcal{O}(TK^N)$.  Further details of the algorithm and regret analysis are provided in Appendix~\ref{app:ucb}.

To overcome the overall computational burden, we propose a \emph{batched learning} approach that substantially reduces the total number of combinatorial optimization updates, thereby reducing the overall computational cost. 

\begin{remark}[Approximation oracles and rare updates]
Approximation oracles~\citep{kakade2007playing, chen2013combinatorial} are a
standard way to address the computational difficulty of combinatorial
optimization. However, replacing the exact oracle with an approximation oracle
changes the benchmark from exact regret to approximation regret. In this work,
we instead target exact regret by using exact optimization only at rare batch
boundaries. These two ideas are complementary: approximation reduces the cost of
each optimization call, whereas batching reduces the number of such calls.
Therefore, batch updates can also be combined with approximation oracles to
further reduce computation, at the price of targeting the corresponding
approximation-regret benchmark. Additional details are provided in
Appendix~\ref{app:oracle}.
\end{remark}

\section{Batch Learning for Stochastic Matching Bandits}\label{sec:elim}

\begin{algorithm*}[ht]
  \caption{Batched Stochastic Matching Bandit (\texttt{B-SMB})}\label{alg:elim}
  {\textbf{Input:} $\kappa$, $M\ge 1$}; {\textbf{Initialize: }$t\leftarrow 1, T_1\leftarrow \eta_T,  \mathcal{M}_0\leftarrow \mathcal{M}, \mathcal{N}_{k,0}\leftarrow[N], \mathcal{T}_{k,0}\leftarrow \emptyset$ }

Compute the SVD \(X=U\Sigma V^\top\), form \(U_r=[u_1,\dots,u_r]\), and set
\(z_n \leftarrow U_r^\top x_n\) for all \(n\in[N]\).
  
 \For{$\tau=1,2...$}{
 \For{$k\in[K]$}{

$\widehat{\theta}_{k,\tau}\leftarrow\argmin_{\theta\in\mathbb{R}^{r}} l_{k,\tau}(\theta)$ with  \eqref{eq:log-loss} where $\mathcal{T}_{k,\tau-1}=  \mathcal{T}_{k,\tau-1}^{(1)}\cup\mathcal{T}_{k,\tau-1}^{(2)}$ and $\mathcal{T}_{k,\tau-1}^{(2)}=\bigcup_{n\in\mathcal{N}_{k,\tau-1}}\mathcal{T}_{n,k,\tau-1}^{(2)}$

\tcp{Assortments Construction \& Elimination  with  \eqref{eq:ucb_lcb} }

$\displaystyle\{S_{l,\tau}^{(n,k)}\}_{l\in[K]}\leftarrow \argmax_{\{S_l\}_{l\in[K]}\in\mathcal{M}_{\tau-1}: n\in S_k} \sum_{l\in[K]}{R}^{UCB}_{l,\tau}(S_l)$ for all $n\in\mathcal{N}_{k,\tau-1}$\label{line:construct}

$\displaystyle\mathcal{N}_{k,\tau}
\!\leftarrow\!\{n\in\mathcal{N}_{k,\tau-1}\!:\!\max_{\{S_{l}\}_{l\in[K]}\in \mathcal{M}_{\tau-1}}\sum_{l\in[K]}R^{LCB}_{l,\tau}(S_l)\le \sum_{l\in[K]}R^{UCB}_{l,\tau}(S_{l,\tau}^{(n,k)})\}$ \label{line:elim1_2}

\tcp{G-Optimal Design}

 $\pi_{k,\tau}\leftarrow \argmin_{\pi\in \mathcal{P}(\mathcal{N}_{k,\tau})} \max_{n\in \Ncal_{k,\tau}}\|z_n\|^2_{(\sum_{n\in \Ncal_{k,\tau}}\pi_{k,\tau}(n)z_nz_n^\top +(1/rT_\tau)I_r)^{-1}}$\label{line:elim3}

\tcp{Exploration}
  Run \texttt{Warm-up} (Algorithm~\ref{alg:warm}) over time steps in $\mathcal{T}_{k,\tau}^{(1)}$ (defined in Algorithm~\ref{alg:warm})\label{line:warm-up}

 \For{$n\in\mathcal{N}_{k,\tau}$}
 {$t_{n,k}\leftarrow t$, $\mathcal{T}_{n,k,\tau}^{(2)}\leftarrow [t_{n,k},t_{n,k}+\lceil r\pi_{k,\tau}(n) T_\tau\rceil-1]$

\While{$t\in \mathcal{T}_{n,k,\tau}^{(2)}$}{
Offer $\{S_{l,t}\}_{l\in[K]}=\{S_{l,\tau}^{(n,k)}\}_{l\in[K]}$ and observe  feedback $y_{m,k,t}\in\{0,1\}$ for all $m\in S_{l,t}$ and $l\in[K]$

$t\leftarrow t+1$\label{line:elim4}}}}
  $\mathcal{M}_\tau\leftarrow \{\{S_k\}_{k\in[K]}: S_k\subset \mathcal{N}_{k,\tau}, |S_k|\le L\: \forall k\in[K], S_k\cap S_l=\emptyset \: \forall k\neq l\}$\label{line:elim_update_M}
  
  $T_{\tau+1}\leftarrow \eta_{T}\sqrt{T_\tau}$
 }
\end{algorithm*}

For batch learning to reduce the computational cost, we adopt the elimination-based bandit algorithm \citep{lattimore2020bandit}. This approach presents several key challenges in the framework of SMB, including efficiently handling the large number of possible matchings between agents and arms for elimination, designing an appropriate estimator for the elimination process, and minimizing the total number of updates to reduce computational overhead. The details of our algorithm (Algorithm~\ref{alg:elim}) are described as follows.

\paragraph{Preprocessing.}
Before the online rounds, the algorithm computes the SVD
\(X=U\Sigma V^\top\) and lets \(r=\mathrm{rank}(X)\). It then projects each
feature vector onto the rank-\(r\) subspace, defining
\(z_n=U_r^\top x_n\in\mathbb{R}^r\), where \(U_r\) contains the left singular
vectors associated with nonzero singular values. With
\(\theta_k^*=U_r^\top\theta_k\), the MNL model is equivalently written in terms
of \((z_n,\theta_k^*)\). This projection does not require prior knowledge of
\(r\) and also ensures proper regularization in the warm-up phase
(Line~\ref{line:warm-up} of Algorithm~\ref{alg:elim}); see
Appendix~\ref{app:projection} for details.

\paragraph{Estimation.} In what follows, we describe the process for constructing assortments at each time step. The algorithm consists of several epochs.
For each $k\in[K]$, from observed feedback $y_{n,k,t}\in\{0,1\}$ for $n\in S_{k,t}$, $t\in\mathcal{T}_{k,\tau-1}$, where $\mathcal{T}_{k,\tau-1}$ is a set of the exploration time steps regarding arm $k$ in the $\tau-1$-th epoch, we first define the negative log-likelihood loss as 
\begin{align}
    \!\!l_{k,\tau}(\theta)\!=-\!\!\!\!\!\!\sum_{t\in \mathcal{T}_{k,\tau-1}}\sum_{n\in S_{k,t}\cup\{n_0\}}\!\!\!\!\!y_{n,k,t}\log p(n|S_{k,t},\theta)\!+\!\tfrac{1}{2}\|\theta\|_2^2,\label{eq:log-loss}
\end{align}
where, with a slight abuse of notation, $p(n|S_{k,t},\theta):=\exp(z_{n}^\top\theta)/(1+\sum_{m\in S_{k,t}}\exp(z_{m}^\top \theta)).$ 
Then at the beginning of each epoch $\tau$, the algorithm estimates $\widehat{\theta}_{k,\tau}$ from the method of Maximum Likelihood Estimation (MLE).

\paragraph{Confidence Bound Construction.}  From the estimator, we define upper and lower confidence bounds for expected reward of assortment $S_{k}$ as 
\begin{align}
    &\!\!R_{k,\tau}^{UCB}(S_{k})\!:=\!\!\sum_{n\in S_k}[w_{n,k}p(n|S_{k},\widehat{\theta}_{k,\tau})] +2\beta_T\max_{n\in S_{k}}\|z_n\|_{V_{k,\tau}^{-1}}, \cr 
&\!\!R_{k,\tau}^{LCB}(S_{k})\!:=\!\!\sum_{n\in S_k}[w_{n,k}p(n|S_{k},\widehat{\theta}_{k,\tau})]-2\beta_T\max_{n\in S_{k}}\|z_n\|_{V_{k,\tau}^{-1}},\label{eq:ucb_lcb}
\end{align}
  where confidence term $\beta_T=\frac{C_1}{\kappa}\sqrt{\log(TNK)}$ for some constant $C_1>0$ and $V_{k,\tau}=\sum_{t\in \mathcal{T}_{k,\tau-1}}\sum_{n\in S_{k,t}}z_nz_n^\top+I_r$. Unlike prior MNL bandit methods~\citep{oh2021multinomial, lee2024nearly},
which construct confidence intervals for latent utilities inside the MNL
function, we place the confidence term directly at the reward level. This is
needed to obtain valid UCB and LCB indices for the expected reward, since
lower-bounding latent utilities does not necessarily lower-bound the nonlinear
MNL reward.

 \paragraph{Efficient Elimination of Suboptimal Matches.} For batch updates, we utilize elimination for suboptimal matches. However, exploring all possible matchings naïvely for the elimination is statistically expensive. Therefore, we utilize a statistically efficient exploration strategy by assessing the eligibility of each assignment $(n,k)$ for $n\in \mathcal{N}_{k,\tau-1}$ and $k\in[K]$ as a potential optimal assortment, where $\mathcal{N}_{k,\tau-1}$ is the active set of agents regarding arm $k$ at epoch $\tau$. 
 To evaluate the assignment $(n,k)$, it constructs a representative assortment  of $
\{S_{l,\tau}^{(n,k)}\}_{l\in[K]}$ from an optimistic view (Line~\ref{line:construct} in Algorithm~\ref{alg:elim}).
Then based on the representative assortments, it obtains $\mathcal{N}_{k,\tau}$ by eliminating $n\in \mathcal{N}_{k,\tau-1}$ which satisfies an elimination condition (Line~\ref{line:elim1_2}).
From the obtained $\mathcal{N}_{k,\tau}$ for all $k\in[K]$, it constructs an active set of assortments $\mathcal{M}_{\tau}$ (Line~\ref{line:elim_update_M}),  
 which is likely to contain the optimal assortments as $\{S_k^*\}_{k\in[K]}\in \mathcal{M}_\tau$.

\paragraph{Assortment Assignment.} Following the elimination process outlined above, here we describe the policy of assigning assortment $\{S_{k,t}\}_{k\in[K]}$ at each time $t$ corresponding to Lines~\ref{line:elim3}-\ref{line:elim4} in Algorithm~\ref{alg:elim}. The algorithm initiates the warm-up stage (Line~\ref{line:warm-up} in Algorithm~\ref{alg:elim}) to apply regularization to the estimators, by uniform exploration across all agents $n\in[N]$ for each arm $k\in[K]$. Then  for each $k\in[K]$, the algorithm utilizes the G-optimal design problem \citep{lattimore2020bandit} to obtain proportion $\pi_{k,\tau}\in \mathcal{P}(\mathcal{N}_{k,\tau})$ for learning $\theta_k^*$ efficiently by exploring agents in $\Ncal_{k,\tau}$, 
where $\mathcal{P}(\mathcal{N}_{k,\tau})$ is the probability simplex with vertex set $\mathcal{N}_{k,\tau}$. Notably, the G-optimal design problem can be solved by the Frank-Wolfe algorithm \citep{damla2008linear}.
Then, for all $n\in\mathcal{N}_{k,\tau}$, it explores $\{S_{l,\tau}^{(n,k)}\}_{l\in[K]}$ several times using $\pi_{k,\tau}(n)$ which is the corresponding value of $n$ in $\pi_{k,\tau}$.

\paragraph{Scheduling.}
The algorithm proceeds in epochs $\tau$ until the time horizon $T$ is reached. 
The number of rounds in each epoch is scheduled recursively as $
T_{\tau} = \eta_T \sqrt{T_{\tau-1}}.$
This schedule limits the number of updates for assortment assignments, which is crucial for reducing the overall computational cost. We set $\eta_T = \left(\frac{T}{rK}\right)^{\frac{1}{2(1-2^{-M})}},$
where $M \ge 1$ specifies the batch update budget. Let $\tau_T$ denote the final epoch before time $T$, corresponding to the total number of batch updates. As formalized in the following proposition, the parameter $M$ directly bounds $\tau_T$ and thus limits the frequency of optimization updates (see Appendix~\ref{app:batch_bd} for the proof).

\begin{proposition}[Number of Batch Updates]\label{prop:batch_bd}
     $\tau_T\le M$.
\end{proposition}
We establish the following regret bound for our algorithm, with the proof provided in Appendix~\ref{app:thm_elim}.
\begin{theorem}
\label{thm:elim}  Algorithm~\ref{alg:elim} with {$M=\mathcal{O}(\log T)$} achieves: $$\mathcal{R}(T)=\tilde{\Ocal}\bigg(\tfrac{1}{\kappa}K^{\frac{3}{2}}\sqrt{rT}\bigg(\frac{T
}{rK}\bigg)^{\frac
{1}{2(2^M-1)}}\bigg).$$
\end{theorem}
\begin{corollary}\label{cor:elim}
    For $M=\Theta\bigl(\log\log(\frac{T}{rK})\bigr)$, Algorithm~\ref{alg:elim} achieves:  
\[\mathcal{R}(T)=\tilde{\Ocal}\left(\tfrac{1}{{\kappa}}K^{3/2}\sqrt{rT}\right).\]
\end{corollary}

\begin{remark}[Efficiency via Rare Optimization Updates]
As established in Corollary~\ref{cor:elim}, Algorithm~1 invokes combinatorial
optimization at most $M=\Theta\bigl(\log\log(\frac{T}{rK})\bigr)$ times over the entire horizon $T$,
while achieving a regret bound of $\widetilde{\mathcal{O}}(\sqrt{T})$.
Consequently, the total computational cost associated with combinatorial
optimization is $\mathcal{O}(NKC_{\mathrm{opt}}\log\log(\frac{T}{rK}))$.  This stands in sharp contrast to the naive approach discussed in
Section~\ref{sec:curse_complexity} (e.g., Algorithm~\ref{alg:ucb} in
Appendix~\ref{app:ucb}), which requires solving a combinatorial optimization
problem of cost $C_{\mathrm{opt}}$ at every round, leading to a total cost of
$\mathcal{O}(C_{\mathrm{opt}}T)$.  Thus, the dependence
on the horizon is reduced from linear in \(T\) to doubly logarithmic in \(T\),
substantially reducing the overall computational burden.

\end{remark}
\paragraph{Discussion on the Tightness of the Regret Bound.}
We begin by comparing our results to those from previous batch bandit studies under a (generalized) linear structure. Our regret bound, given as \(\tilde{\mathcal{O}}(T^{\frac{1}{2}+\frac{1}{2(2^M-1)}})=\tilde{\mathcal{O}}(T^{\frac{1}{2(1-2^{-M})}})\) for a general \(M = \mathcal{O}(\log(T))\), matches the results from \citet{han2020sequential,ren2024dynamic,sawarnigeneralized}.
 Notably, this bound also aligns with the lower bound for the linear structure, \(\Omega(T^{\frac{1}{2(1-2^{-M})}})\) \citep{han2020sequential}. For the case of \(M = \Theta\bigl(\log\log(\frac{T}{rK})\bigr)\), our bound of \(\tilde{\mathcal{O}}(\sqrt{T})\) corresponds to the findings for linear bandits in \citet{ruan2021linear,hanna2023contexts} under $\log\log(T)$ batch updates.  

Given that our problem generalizes the single-assortment MNL setting to \(K\)-multiple assortments, we can obtain the regret lower bound of \(\Omega(K\sqrt{T})\) with respect to \(K\) and \(T\) for the contextual setting, by simply extending the result of Theorem 3 in \citet{lee2024nearly} for single-assortment to $K$-multiple assortments. In comparison, our analysis indicates a regret dependence of \(K^{3/2}\) when \(M = \Theta\bigl(\log\log\bigl(\frac{T}{rK}\bigr)\bigr)\), which is worse by a factor of \(\sqrt{K}\) relative to the lower bound. This gap arises from the need to explore all potential matches during the epoch-based elimination procedure in batch updates.

Although Algorithm~\ref{alg:elim} achieves $\widetilde{\mathcal{O}}(\sqrt{T})$
regret with amortized computational efficiency due to reduced batch updates,
its regret guarantee relies on problem-specific knowledge of~$\kappa$ and,
crucially, requires this parameter to be known in advance for setting
$\beta_T$.
The resulting regret bound scales linearly with $1/\kappa$, which can be as
large as $\mathcal{O}(L^2)$ in the worst case.
In the following section, we propose an algorithm that removes the need for
prior knowledge of~$\kappa$.

\section{Without Prior Knowledge $\kappa$ }
Here we provide details of our proposed algorithm (Algorithm~\ref{alg:elim2} in Appendix~\ref{app:elim2}), focusing on the difference from the algorithm in the previous section. 

\paragraph{Local curvature.} While we follow the elimination framework of Algorithm~\ref{alg:elim}, to improve the dependence on the unknown parameter $\kappa$, we exploit the local curvature of the log-likelihood around the current estimate $\widehat{\theta}_{k,\tau}$. Specifically, we define the regularized empirical Hessian
\begin{align}
\!\!\!\!\!\!\!\!\!\!H_{k,\tau}(\widehat{\theta}_{k,\tau})& =\lambda I_r+ \!\!\!\!\!\!\sum_{t\in \Tcal_{k,\tau-1}}\!\!\!\bigg[ \sum_{n\in S_{k,t}}\!\!p(n|S_{k,t},\widehat{\theta}_{k,\tau})z_nz_n^\top \!-\! \!\!\!\!\!\sum_{n,m \in S_{k,t}}\!\!\!\!\!\!p(n|S_{k,t},\widehat{\theta}_{k,\tau})p(m|S_{k,t},\widehat{\theta}_{k,\tau})z_nz_m^\top \bigg], \label{eq:local_gram}
\end{align}
where $\lambda=C_2r\log(K)$ for some constant $C_2>0$ and we denote $H_{k,\tau}(\widehat{\theta}_{k,\tau})$ by $H_{k,\tau}$ when there is no confusion. We define $\tilde{z}_{n,k,\tau}(S_{k,t})=z_n-\sum_{m\in S_{k,t}}p(m|S_{k,t},\widehat{\theta}_{k,\tau})z_m$ and we use $\tilde{z}_{n,k,\tau}$ for it, when there is no confusion. 

\paragraph{Confidence Bound Construction.} For the confidence bound, we define 
\begin{align*}B_\tau(S_{k,t})&:=\tfrac{13}{2}\zeta_\tau^2\!\max_{n\in S_{k,t}}\!\!\|z_n\|_{H_{k,\tau}^{-1}}^2\!+2\zeta_\tau^2\!\max_{n\in S_{k,t}}\!\!\|\tilde{z}_{n,k,\tau}\|_{H_{k,\tau}^{-1}}^2+ \zeta_\tau\sum_{n\in S_{k,t}}p(n|S_{k,t},\widehat{\theta}_{k,\tau-1})\|\tilde{z}_{n,k,\tau}\|_{H_{k,\tau}^{-1}},
\end{align*} 
where $\zeta_\tau=\frac{1}{2}\sqrt{\lambda}+\frac{2r}{\sqrt{\lambda}}\log({4KT}(1+\frac{2(t_\tau-1)L}{r\lambda}))$ with the start time of $\tau$-th episode $t_\tau$. We note that the first term arises from the second-order term in the Taylor expansion for the error from estimator, while the second and last terms originate from the first-order term.  Notably, our confidence bounds for $\tau$-th episode utilize not only the current estimator $\widehat{\theta}_{k,\tau}$ but the previous one  $\widehat{\theta}_{k,\tau-1}$ (in the last term) because the historical data in $H_{k,\tau}$  is obtained from the G/D-optimal policy which is optimized under $\widehat{\theta}_{k,\tau-1}$. 
Then we  define upper and lower confidence bounds as
\begin{align}
   & \hspace{-2mm}R_{k,\tau}^{UCB}(S_{k,t})\!:=\!\!\sum_{n\in S_{k,t}}\!w_{n,k}p(n|S_{k,t},\widehat{\theta}_{k,\tau})
    +B_\tau(S_{k,t}),
    \cr 
&\hspace{-2mm}R_{k,\tau}^{LCB}(S_{k,t})\!:=\!\!\sum_{n\in S_{k,t}}\!w_{n,k}p(n|S_{k,t},\widehat{\theta}_{k,\tau})-B_\tau(S_{k,t}).\label{eq:ucb_lcb2}
\end{align}

\paragraph{Exploration Space.} For the G/D-optimal design aimed at exploring the space of arms, the algorithm must account for both \( V_{k,\tau} \) and \( H_{k,\tau}(\widehat{\theta}_{k,\tau}) \) to achieve a tight regret bound that avoids dependence on \( 1/\kappa \). This marks a key distinction from Algorithm~\ref{alg:elim}. From this, the algorithm requires two different types of procedures regarding assortment construction, elimination, and exploration. Let $\Jcal(A)$ be the set of all combinations of subset of $A$ with cardinality bound $L$ as $\Jcal(A)=\{B\subseteq A\mid |B|\le L\}$, and  let $\Kcal(A)$ be the set of all combinations of subset $A$ (with cardinality bound $L$) and its element  as $\Kcal(A)=\{(b, B) \mid b\in B\subseteq A,  |B|\le L\}$. The G/D-optimal design seeks to minimize the ellipsoidal volume under \( V_{k,\tau} \), based on arm selection probabilities within the active set of arms \( \mathcal{N}_{k,\tau} \). Additionally, since the action space in \( H_{k,\tau}(\widehat{\theta}_{k,\tau}) \) depends not only on the selection of actions but also on the selection of assortments, the G/D-optimal design incorporates assortment selection probabilities for \(\mathcal{J}(\mathcal{N}_{k,\tau})\) and  \(\mathcal{K}(\mathcal{N}_{k,\tau})\). Following this policy, the algorithm includes two separate exploration procedures regarding the selection of arms and assortments. 

\begin{remark}[Localized design for MNL assortments]
Our localized Gram matrix in \eqref{eq:local_gram} differs from prior MNL
constructions~\citep{goyal2021dynamic, lee2024nearly}: it avoids the nonconvex
optimism step used by \citet{goyal2021dynamic} and, unlike
\citet{lee2024nearly}, uses the current estimator rather than aggregating
historical estimators. This makes it naturally compatible with our G/D-optimal
design. Moreover, unlike linear or generalized linear bandits
\citep{lattimore2020bandit, sawarnigeneralized}, our design must account for
assortment-dependent MNL feedback, and therefore explores a rescaled feature
space over assortments rather than individual actions.
\end{remark}

We set $\eta_T=(\frac{T}{rK})^{\frac{1}{2(1-2^{-M})}}$ with a parameter for batch update budget $M\ge1$. Recall that $\tau_T$ is the last epoch over $T$, which indicates the number of batch updates. Then, by following the same proof of Proposition~\ref{prop:batch_bd}, we have the following bound for the number of epochs.
\begin{proposition}[Number of Batch Updates] \label{prop:batch_bd2}$\tau_T\le M$. \end{proposition}
Then, we have the following regret bounds (the proof is provided in Appendix~\ref{app:elim2}).
\begin{theorem}
\label{thm:elim2}  Algorithm~\ref{alg:elim2} with $M=\mathcal{O}(\log(T))$ achieves: $\mathcal{R}(T)=\tilde{\Ocal}\big(rK^{\frac{3}{2}}\sqrt{T}\big(\frac{T
}{rK}\big)^{\frac
{1}{2(2^M-1)}}\big).$
\end{theorem}

\begin{corollary}
    For $M=\Theta\bigl(\log\log(\frac{T}{rK})\bigr)$, Algorithm~\ref{alg:elim2} achieves:  
$\mathcal{R}(T)=\tilde{\Ocal}\big(rK^{\frac{3}{2}}\sqrt{T}\big).$
\end{corollary}

\begin{remark}[Improvement on \(\kappa\)]
Algorithm~\ref{alg:elim2} does not require prior knowledge of \(\kappa\) and
removes the leading \(1/\kappa=\mathcal O(L^2)\) dependence in
Theorem~\ref{thm:elim}, at the cost of an additional \(\sqrt r\) regret factor.
\end{remark}
\begin{remark}[Batch Complexity]
Like Algorithm~\ref{alg:elim}, this algorithm requires only
$\Theta\bigl(\log\log(\frac{T}{rK})\bigr)$ batch updates to achieve a
$\widetilde{\mathcal{O}}(\sqrt{T})$ regret bound.
That is, expensive combinatorial optimization is invoked only a small number of
times over the entire horizon, rather than at every round.
\end{remark}

\section{Conclusion}\label{sec:con}

We introduced a stochastic matching bandit framework in which naive approaches
require solving an NP-hard combinatorial optimization problem at every round.
To address this, we proposed elimination-based batched algorithms that achieve
regret
\(\widetilde{\mathcal O}(\kappa^{-1}K^{3/2}\sqrt{rT})\) with only
\(\Theta(\log\log(T/(rK)))\) optimization updates when \(\kappa\) is known.
We also developed a parameter-free variant that removes prior knowledge of
\(\kappa\) and attains regret
\(\widetilde{\mathcal O}(rK^{3/2}\sqrt T)\) under the same update budget.
Thus, no-regret learning in stochastic matching bandits can be achieved with
rare calls to expensive optimization. Experimental comparisons are provided in
Appendix~\ref{sec:exp}. \textbf{Limitations.} An important future direction is to close the remaining \(\sqrt K\) gap in the regret dependence, or to determine whether this gap is inherent under rare optimization updates.

\section*{Acknowledgments}
JK was supported in part by the French National Research Agency
(ANR) through grants ANR-20-CE23-0007 and ANR-23-CE23-0002
and the PEPR IA FOUNDRY project (ANR23-PEIA-0003),
as well as by Hi! PARIS and the ANR/France 2030 program
(ANR23-IACL-0005).
This research was supported by Basic Science Research Program
through the National Research Foundation of Korea (NRF)
funded by the Ministry of Education (RS-2026-25536608).

MO was supported by the National
Research Foundation of Korea (NRF) grant and the Institute of Information \& communications Technology Planning \& Evaluation (IITP) grant both funded by the Korea
government (MSIT) (No. RS-2022-NR071853, RS-2023-
00222663, RS-2025-25463302, RS-2026-25507282, RS-2026-25613008).

\bibliography{mybib}
\bibliographystyle{plainnat}

\appendix

\newpage

\section{Appendix}
\subsection{Algorithm Without Prior Knowledge of $\kappa$ (Algorithm~\ref{alg:elim2})} \label{app:elim2}

\begin{algorithm}[!t]
\setcounter{AlgoLine}{0}

  \caption{Batched Stochastic Matching Bandit$^+$ (\texttt{B-SMB$^+$})}\label{alg:elim2}
  \KwIn{$M\ge 1$; \textbf{Init: }$t\leftarrow 1, T_1\leftarrow C_3\log(T)\log^2(TKL)$ for some constant $C_3>0$}
 
   Compute SVD of $X=U\Sigma V^\top$ and obtain $U_r=[u_1,\dots, u_r]$; Construct $z_{n}\leftarrow U_{r}^\top  x_n$
 
 \For{$\tau=1,2...$}{
 \For{$k\in[K]$}{

 $\widehat{\theta}_{k,\tau}\leftarrow \argmin_{\theta\in \mathbb{R}^r: \|\theta\|_2\le 1} l_{k,\tau}(\theta)$ with  \eqref{eq:log-loss} where $\mathcal{T}_{k,\tau-1}=\bigcup_{n\in\mathcal{N}_{k,\tau-1}}\mathcal{T}_{n,k,\tau-1}\bigcup_{J\in\Jcal(\mathcal{N}_{k,\tau-1})}\mathcal{T}_{J,k,\tau-1}$

\tcp{Assortments Construction with  \eqref{eq:ucb_lcb2}}

$\{S_{l,\tau}^{(n,k)}\}_{l\in[K]}\leftarrow \argmax_{\{S_l\}_{l\in[K]}\in\mathcal{M}_{\tau-1}: n\in S_k} \sum_{l\in[K]}{R}^{UCB}_{l,\tau}(S_l)$ for all $n\in\mathcal{N}_{k,\tau-1}$ \label{line:construct2}

$\{S_{l,\tau}^{(J,k)}\}_{l\in[K]}\leftarrow \argmax_{\{S_l\}_{l\in[K]}\in\mathcal{M}_{\tau-1}:  S_k=J} \sum_{l\in[K]}{R}^{UCB}_{l,\tau}(S_l)$ for all $J\in\mathcal{J}(\mathcal{N}_{k,\tau-1})$ \label{line:construct2-2}

\tcp{Elimination with  \eqref{eq:ucb_lcb2}}

$\mathcal{N}_{k,\tau}'\!\leftarrow\!\{n\in\mathcal{N}_{k,\tau-1}\!:\!\max_{\{S_{l}\}_{l\in[K]}\in \mathcal{M}_{\tau-1}}\sum_{l\in[K]}R^{LCB}_{l,\tau}(S_l)\le \sum_{l\in[K]}R^{UCB}_{l,\tau}(S_{l,\tau}^{(n,k)})\}$  

$\mathcal{N}_{k,\tau}\!\leftarrow\!\{n\in J \!: J \in\mathcal{J}(\mathcal{N}_{k,\tau}'), \max_{\{S_{l}\}_{l\in[K]}\in \mathcal{M}_{\tau-1}}\sum_{l\in[K]}R^{LCB}_{l,\tau}(S_l)\le \sum_{l\in[K]}R^{UCB}_{l,\tau}(S_{l,\tau}^{(J,k)})\}$ \label{line:elim1_ad}

\tcp{G-Optimal Design}

 $\pi_{k,\tau}\leftarrow \argmin_{\pi\in \mathcal{P}(\mathcal{N}_{k,\tau})} \max_{n\in \Ncal_{k,\tau}}\|z_n\|^2_{(\sum_{n\in \Ncal_{k,\tau}}\pi(n)z_nz_n^\top +(\lambda/rT_\tau)I_r)^{-1}}$ \label{line:elim3-2}

 $\displaystyle\tilde{\pi}_{k,\tau}\leftarrow \argmin_{\pi\in \mathcal{P}(\mathcal{J}(\mathcal{N}_{k,\tau}))}\max_{J\in \Jcal(\Ncal_{k,\tau})}\Big\|\sum_{n\in J}\tilde{z}_{n,k,\tau}'(J)\Big\|_{(\sum_{J\in \Jcal(\Ncal_{k,\tau})}\pi(J)\sum_{n\in J}\tilde{z}_{n,k,\tau}'(J)\tilde{z}_{n,k,\tau}'(J)^\top +(\lambda/T_\tau r) I_r)^{-1}}^2$ where $\tilde{z}_{n,k,\tau}'(J)=\sqrt{p(n|J,\widehat{\theta}_{k,\tau}})\tilde{z}_{n,k,\tau}(J)$

 $\displaystyle\bar{\pi}_{k,\tau}\leftarrow \argmin_{\pi\in \mathcal{P}(\mathcal{K}(\mathcal{N}_{k,\tau}))} \max_{(n,J)\in\mathcal{K}(\mathcal{N}_{k,\tau}) } \|\tilde{z}_{n,k,\tau}(J)\|_{(\sum_{(n,J)\in \Kcal(\Ncal_{k,\tau})}\pi(n,J)\tilde{z}_{n,k,\tau}(J)\tilde{z}_{n,k,\tau}(J)^\top +(\lambda/T_\tau r) I_r)^{-1}}^2$  

\tcp{Exploration}
 \For{$n\in\mathcal{N}_{k,\tau}$}
 {$t_{n,k}\leftarrow t$, $\mathcal{T}_{n,k,\tau}\leftarrow [t_{n,k},t_{n,k}+\lceil r\pi_{k,\tau}(n) T_\tau\rceil-1]$

\While{$t\in \mathcal{T}_{n,k,\tau}$}{
Offer $\{S_{l,t}\}_{l\in[K]}=\{S_{l,\tau}^{(n,k)}\}_{l\in[K]}$  and observe  feedback $y_{m,k,t}\in\{0,1\}$ for all $m\in S_{l,t}$ and $l\in[K]$

$t\leftarrow t+1$ \label{line:elim4-1}}}

 \For{$J\in\Jcal(\mathcal{N}_{k,\tau})$}
 {$t_{J,k}\leftarrow t$, $\mathcal{T}_{J,k,\tau}\leftarrow [t_{J,k},t_{J,k}+\lceil r\tilde{\pi}_{k,\tau}(J)T_\tau\rceil-1]$

\While{$t\in \mathcal{T}_{J,k,\tau}$}{
Offer $\{S_{l,t}\}_{l\in[K]}=\{S_{l,\tau}^{(J,k)}\}_{l\in[K]}$  and observe  feedback $y_{m,k,t}\in\{0,1\}$ for all $m\in S_{l,t}$ and $l\in[K]$

$t\leftarrow t+1$ \label{line:elim4-2}}}

\For{$(n,J)\in\Kcal(\mathcal{N}_{k,\tau})$}
 {$t_{n,J,k}\leftarrow t$, $\mathcal{T}_{n,J,k,\tau}\leftarrow [t_{n,J,k},t_{n,J,k}+\lceil r\bar{\pi}_{k,\tau}(n,J) T_\tau\rceil-1]$

\While{$t\in \mathcal{T}_{n,J,k,\tau}$}{
Offer $\{S_{l,t}\}_{l\in[K]}=\{S_{l,\tau}^{(J,k)}\}_{l\in[K]}$  and observe  feedback $y_{m,k,t}\in\{0,1\}$ for all $m\in S_{l,t}$ and $l\in[K]$

$t\leftarrow t+1$ }}}

 $\mathcal{M}_{\tau}\leftarrow \{\{S_k\}_{k\in[K]}: S_k\subseteq \mathcal{N}_{k,\tau}, |S_k|\le L\: \forall k\in[K], S_k\cap S_l=\emptyset \: \forall k\neq l\}$\label{line:elim_update_M2} 

    $T_{\tau+1}\leftarrow \eta_{T}\sqrt{T_\tau}$
 }
\end{algorithm}

\subsection{Naive Approach by Extending  MNL Bandit}\label{app:ucb}
For our framework, we can utilize MNL bandit \cite{lee2024nearly} by extending it to $K$-multiple MNLs (Algorithm~\ref{alg:ucb}) as follows. 
 Let the negative log-likelihood $l_{k,t}(\theta)=-\sum_{n\in S_{k,t}\cup \{n_0\}}y_{n,k,t}\log p(n|S_{k,t},\theta)$ where $y_{n,k,t}\in\{0,1\}$ is observed preference feedback  ($1$ denotes a choice, and $0$ otherwise). Then we define the gradient of the likelihood  as 
\begin{align}
g_{k,t}(\theta):=\nabla_\theta l_{k,t}(\theta)=\sum_{n\in S_{k,t}}(p(n|S_{k,t},\theta)-y_{n,k,t})x_{n}.\label{eq:grad-ucb}    
\end{align} We also define gram matrices from $\nabla^2_{\theta}l_{k,t}(\theta)$ as follows:
\begin{align}
    G_{k,t}(\theta):= \sum_{n\in S_{k,t}}p(n|S_{k,t},\theta)z_{n}z_{n}^\top-\sum_{n,m\in S_{k,t}}p(n|S_{k,t},\theta)p(m|S_{k,t},\theta)z_{n}z_{m}^\top.
    \label{eq:gram-ucb}
\end{align}
 We define the UCB index for assortment $S_{k}$ as 
 \begin{align}
     R^{UCB}_{k,t}(S_{k})=\sum_{n\in S_{k}}\frac{\exp(h_{n,k,t})}{1+\sum_{m\in S_{k}}\exp (h_{m,k,t})},\label{eq:R_ucb}
 \end{align}
  where $h_{n,k,t}=z_{n}^\top \widehat{\theta}_{k,t}+\gamma_{t}\|z_n\|_{G_{k,t}^{-1}}$ with $\gamma_t=C_4\log(L)\sqrt{d\log(t)\log(KT)}$ for some $C_4>0$. We set $\lambda= C_5d\log(K)$ and $\eta=C_6\log(K)$ for some $C_5>0$ and $C_6>0$.

\begin{proposition}\label{prop:UCB}
    Algorithm~\ref{alg:ucb} achieves a regret bound of $\Rcal(T)=\tilde{\mathcal{O}}(rK\sqrt{T})$ and the computational cost per round is $\mathcal{O}(K^N)$.
\end{proposition}
\begin{proof}
    The proof is provided in Appendix~\ref{app:thm_ucb}.
\end{proof}
  
\begin{algorithm}[H]
  \caption{Extension of \texttt{OFU-MNL+} \cite{lee2024nearly}}\label{alg:ucb}

  Compute SVD of $X=U\Sigma V^\top$  and obtain  $U_r=[u_1,\dots, u_r]$; Construct $z_{n}\leftarrow U_{r}^\top  x_n$ for $n\in [N]$

 \For{$t=1,\dots,T$}{
 \For{$k\in[K]$}{
 $\tilde{\Gcal}_{k,t}\leftarrow \lambda I_{d}+\sum_{s=1}^{t-2}G_{k,s}(\widehat{\theta}_{k,s})+\eta G_{k,t-1}(\widehat{\theta}_{k,t-1}) \text{ with \eqref{eq:gram-ucb}}$

${\Gcal}_{k,t}\leftarrow \lambda I_{d}+\sum_{s=1}^{t-1}G_{k,s}(\widehat{\theta}_{k,s})$\text{ with \eqref{eq:gram-ucb}}

$\widehat{\theta}_{k,t}\leftarrow \argmin_{\theta\in\Theta}g_{k,t-1}(\widehat{\theta}_{k,t-1})^\top \theta+\frac{1}{2\eta}\|\theta-\widehat{\theta}_{k,t-1}\|_{\tilde{\Gcal}_{k,t}^{-1}}^2 \text{ with \eqref{eq:grad-ucb}}$ 

}

 $\displaystyle\{S_{k,t}\}_{k\in[K]}\leftarrow \argmax_{\{S_{k}\}_{k\in[K]}\in\mathcal{M}}\sum_{k\in[K]}R^{UCB}_{k,t}(S_{k})$ with \eqref{eq:R_ucb} \label{line:combinatorial}

Offer $\{S_{k,t}\}_{k\in[K]}$ and observe $y_{n,k,t}$ for all $n\in S_{k,t}$, $k\in[K]$
 }
\end{algorithm}

\subsection{Details Regarding Projection in Feature Space} \label{app:projection} 

 Since $x_n$ for $n\in[N]$ lies in the subspace $U_r$, we observe  that $x_n=U_r b_n$ for some $b_n\in\mathbb{R}^r$. Let $\theta^*_k=U_r^\top \theta_k$. Then we have $x_n^\top \theta_k=z_n^\top \theta_k^*$ by following  $x_n^\top\theta_k=b_n^\top U_r^\top \theta_k=b_n^\top (U_r^\top U_r) U_r^\top \theta_k=x_n^\top U_rU_r^\top \theta_k=z_n^\top \theta^*_k$ using $U_r^\top U_r=I_d$. Therefore, we can reformulate the MNL model using $r$-dimensional feature $z_n\in\mathbb{R}^r$ and latent $\theta^*_k\in\mathbb{R}^r$ in place of $d$-dimensional $x_n\in\mathbb{R}^d$ and $\theta_k\in\mathbb{R}^d$, respectively, for $n\in[N]$ and $k\in[K]$.
 We note that this procedure is beneficial not only for reducing feature dimension but also for introducing appropriate regularization for estimators without imposing any assumption about feature distributions considered in \citet{oh2021multinomial} (see Lemma~\ref{lem:lambda_min_bd}).

\subsection{Warm-up Stage for Algorithm~\ref{alg:elim}}\label{app:warm}
Let $\lambda_{\min}(A)$ denote the minimum eigenvalue of matrix $A$. Then we provide the warm-up stage for Algorithm~\ref{alg:elim} in Algorithm~\ref{alg:warm}.

\begin{algorithm}[h]
  \caption{\texttt{Round-robin Warm-up }}\label{alg:warm}
  
$\lambda_{\min}\leftarrow\lambda_{\min}(\sum_{n\in[N]}z_{n}z_{n}^\top)$

 {
 $t_k\leftarrow$ $t$, $i\leftarrow \min\{L,N\}$
 
 $T_{k}'\leftarrow\frac{C_3N}{i\kappa^2\lambda_{\min}\log(TKN)}(r+\log(TKN))^2$ 
 
 $\mathcal{T}_{k,\tau}^{(1)}\leftarrow [t_k,t_k+T_{k}'-1]$ 
 
 \For{$t\in\mathcal{T}_{k,\tau}^{(1)}$}
 {$a\leftarrow ((i(t-1)) \bmod N)+1$, $b\leftarrow ((it-1) \bmod N)+1$
 
 \eIf{$a\le b$}{$S_{k,t}\leftarrow[a,b]$}
{$S_{k,t}\leftarrow[1,b]\cup[a,N]$}

Construct any $S_{l,t}$ for $l\in[K]/\{k\}$ satisfying $\{S_{k,t}\}_{k\in[K]}\in\mathcal{M}_0$

  Offer $\{S_{k,t}\}_{k\in[K]}$ and observe feedback $y_{n,k,t}\in\{0,1\}$ for all $n\in S_{k,t},k\in[K]$}
 }
\end{algorithm}
\subsection{Proof of Proposition~\ref{prop:batch_bd}}\label{app:batch_bd}

Here we utilize the proof techniques in \cite{sawarnigeneralized}. Recall that $\tau_T$ to be the smallest $\tau\in[T]$ such that  \[\sum_{\tau'\in[
\tau] }\sum_{k\in[K]}|\mathcal{T}_{k,\tau'}^{(1)}|+|\mathcal{T}_{k,\tau'}^{(2)}|\ge T.\] In other words, $\sum_{\tau'\in[
\tau_T-1] }\sum_{k\in[K]}|\mathcal{T}_{k,\tau'}^{(1)}|+|\mathcal{T}_{k,\tau'}^{(2)}|<T$.  Then we can show that $\tau_T\le M$ by contradiction as follows. Suppose $\tau_T> M$. Then, we have 
\begin{align*}
T_{\tau_T-1}\ge(\eta_T)^{\sum_{k=1}^{\tau_T-1}(\frac{1}{2})^{k-1}}\ge (\eta_T)^{2(1-(\frac{1}{2})^{\tau_T-1})}=(T/rK)^{\frac{1-2^{1-\tau_T}}{1-2^{-M}}}\ge  T/rK,
\end{align*}
where the last inequality comes from $M+1\le \tau_T$.  This implies that $\sum_{\tau'\in[
\tau_T-1] }\sum_{k\in[K]}|\mathcal{T}_{k,\tau'}^{(1)}|+|\mathcal{T}_{k,\tau'}^{(2)}|\ge KrT_{\tau_T-1}\ge T$, which is contradiction. Thus, we can conclude that $\tau_T\le M$.

\subsection{Proof of Theorem~\ref{thm:elim}}\label{app:thm_elim}

 In the following proof, with a slight abuse of notation, we use $p(n|S,\theta)=\exp(z_n^\top \theta)/(1+\sum_{m\in S}\exp(z_m^\top \theta))$ with $z_n\in\RR^r$ instead of $x_n\in\RR^d$. We provide a lemma for a confidence bound.

\begin{lemma}\label{lem:elim_conf}
      For any $\tau\in[T]$,  $k\in[K]$, and $n\in [N]$, with probability at least $1-\delta$, for some constant $C>0$, we have \[|z_{n}^\top (\widehat{\theta}_{k,\tau}-\theta_k^*)|\le \tfrac{C}{\kappa}\sqrt{\|z_{n}\|_{V_{k,\tau}^{-1}}^2\log(T KN/\delta) }.\] 
\end{lemma}
\begin{proof}
 We define the gradient of the likelihood  as 
\begin{align*}
g_{k,\tau}(\theta):=\sum_{t\in \Tcal_{k,\tau}}\nabla_\theta l_{k,t}(\theta)=\sum_{t\in \Tcal_{k,\tau}}\sum_{n\in S_{k,t}}(p(n|S_{k,t},\theta)-y_{n,k,t})z_{n}+\theta.   
\end{align*} Then we first provide a bound in the following lemma.

\begin{lemma} \label{lem:z_theta_hat_z_gap_bd} For any $n\in[N]$, $k\in[K]$, and $\tau\in[T]$, with probability at least $1-\delta$, we have
\[|z_n^\top(\widehat{\theta}_{k,\tau}-\theta_k^*)|\le \frac{3 \sqrt{\log(T KN/\delta)}}{\kappa}\|z_n\|_{V_{k,\tau}^{-1}}+ \frac{6}{\kappa^2}\|\widehat{\theta}_{k,\tau}-\theta_{k}^*\|_2\|g_{k,\tau}(\widehat{\theta}_{k,\tau})-g_{k,\tau}(\theta_{k}^*)\|_{V_{k,\tau}^{-1}}\|z_n\|_{V_{k,\tau}^{-1}}.\]
\end{lemma}
\begin{proof}
    The proof is deferred to Appendix~\ref{app:z_theta_hat_theta_gap}
\end{proof}
Then we define
\begin{align*}
    E_1=\bigg\{&|z_n^\top(\widehat{\theta}_{k,\tau}-\theta_k^*)|\le  \frac{3 \sqrt{\log(T KN/\delta)}}{\kappa}\|z_n\|_{V_{k,\tau}^{-1}}\cr &+ \frac{6}{\kappa^2}\|\widehat{\theta}_{k,\tau}-\theta_{k}^*\|_2\|g_{k,\tau}(\widehat{\theta}_{k,\tau})-g_{k,\tau}(\theta_{k}^*)\|_{V_{k,\tau}^{-1}}\|z_n\|_{V_{k,\tau}^{-1}}\:  \forall n\in[N], k\in[K], \tau\in[T]\bigg\},
\end{align*} which holds at least $1-\delta$. 
Now we provide bounds for $\|\widehat{\theta}_{k,\tau}-\theta_{k}^*\|_2$ and $\|g_{k,\tau}(\widehat{\theta}_{k,\tau})-g_{k,\tau}(\theta_k^*)\|_{V_{k,\tau}^{-1}}$.
\begin{lemma}[Lemma 7 in \citet{li2017provably}]\label{lem:g_gap_bd}
 For all $k\in[K]$, $\tau\in[T]$, with probability at least $1-\delta$ for $\delta>0$, we have
    \[\| g_{k,\tau}(\widehat{\theta}_{k,\tau})-g_{k,\tau-1}(\theta_{k}^*)\|_{V_{k,\tau}^{-1}}\le 4\sqrt{2r+\log(KTN/\delta)}.\]
\end{lemma}
We define $V_{k,\tau}^0=\sum_{t\in \mathcal{T}_{k,\tau-1}^{(1)}}\sum_{n\in S_{k,t}}z_nz_n^\top$. Then we have the following lemma.
\begin{lemma}
    For all $k\in[K]$ and $\tau\ge 2$, we have $\lambda_{\min}(V_{k,\tau}^0)\ge \frac{C_0}{\kappa^2\log(TKN/\delta)}(r^2+\log^2(TKN/\delta)+2r\log(TKN/\delta))$.\label{lem:lambda_min_bd}
\end{lemma}
\begin{proof}
      Let $\lambda'=\frac{C_0}{\kappa^2\lambda_{\min}\log(TKN/\delta)} 
(r^2+\log^2(TKN/\delta)+2r\log(TKN/\delta))$ and recall $\lambda_{\min}=\lambda_{\min}(\sum_{n\in[N]}z_{n}z_{n}^\top)$. From the phase in the warm-up stage (Algorithm~\ref{alg:warm}), we can observe that $V_{k,\tau}^0$ contains $z_{n}z_{n}^\top$ for each $n\in [N]$ at least $\lambda'$. Since $\sum_{n\in[N]}z_nz_n^\top=\sum_{s\in [r]}\lambda_{s}u_s{u_s}^\top$, we have $V_{k,\tau}^0=\sum_{t\in \mathcal{T}_{k,\tau-1}^{(1)}}\sum_{n\in S_{k,t}}z_nz_n^\top = \sum_{s\in [r]}\lambda_{s}'u_s{u_s}^\top$ where $\lambda_{s}'\ge \lambda'\lambda_s$. Then from $\lambda_{\min}=\lambda_r$, we can conclude $\lambda_{\min}(V_k^0)\ge  \lambda'\lambda_{\min}$. 
\end{proof}
\begin{lemma}[Lemma 9 in \citet{kveton2020randomized}]    \label{lem:theta_bd_1}
 Suppose $\lambda_{\min}(V_{k,\tau}^0)\ge \max\{(1/4\kappa^2)(r\log(T/r)+2\log(KTN/\delta)),1\}$ for all $k\in[K]$.
Then, for all $\tau\in[T]$ and $k\in[K]$, we have \[\mathbb{P}(\|\widehat{\theta}_{k,\tau}-\theta_k^*\|_2\ge 1)\le \delta.\]
\end{lemma}
We define $E_2=\{\|\widehat{\theta}_{k,\tau}-\theta_k^*\|_2\le 1 \:\forall k\in[K], \tau\in[T]\}.$ 
    Then from Lemmas ~\ref{lem:lambda_min_bd},~\ref{lem:theta_bd_1}, we have $\mathbb{P}(E_2)\ge 1-\delta$. 

We also denote by $E_3$ the event of $\{\|g_{k,\tau}(\widehat{\theta}_{k,\tau})-g_{k,\tau-1}(\theta^*_{k})\|_{V_{k,\tau}^{-1}}\le 4\sqrt{2r+\log(KTN/\delta)} \: \forall \tau\in[T], k\in[K]\}$, which hold with probability at least $1-\delta$ from Lemma~\ref{lem:g_gap_bd}.
\begin{lemma}\label{lem:theta_hat_theta_gap_bd} Under $E_2$ and $E_3$, for any $\tau\in[T]$, $k\in[K]$, we have
    \[ \|\widehat{\theta}_{k,\tau}-\theta^*_k\|_2\le \frac{2}{\kappa}\sqrt{\frac{2r+\log(TNK/\delta)}{\lambda_{\min}(V_{k}^0)}}.\]
\end{lemma}
\begin{proof}
The proof is deferred to Appendix~\ref{app:theta_hat_theta_gap_bd}

\end{proof}

Finally, under $E_1\cap E_2\cap E_3$ which holds with probability at least $1-3\delta$, we have 
\begin{align*}
    &|z_n^\top (\widehat{\theta}_{k,\tau}-\theta_k^*)|\cr &
    \le \frac{2 \sqrt{\log(T KN/\delta)}}{\kappa}\|z_n\|_{V_{k,\tau}^{-1}}+(6/\kappa^2)\|z_n\|_{V_{k,\tau}^{-1}}\|\widehat{\theta}_{k,\tau}-\theta_k^*\|_2\|(g_{k,\tau}(\widehat{\theta}_{k,\tau})-g_{k,\tau}(\theta_k^*))\|_{V_{k,\tau}^{-1}}\cr 
    &\le \frac{2 \sqrt{\log(T KN/\delta)}}{\kappa}\|z_n\|_{V_{k,\tau}^{-1}}+\frac{48(2r+\log(KTN/\delta))}{\kappa^2\sqrt{\lambda_{\min}(V_{k,\tau}^0)}}\|z_n\|_{V_{k,\tau}^{-1}}\cr 
    &\le \frac{3 \sqrt{\log(T KN/\delta)}}{\kappa}\|z_n\|_{V_{k,\tau}^{-1}}\cr &=(3/\kappa)\sqrt{\|z_{n}\|_{V_{\tau,k}^{-1}}^2\log(T KN/\delta) }:=\beta(\delta)\|z_n\|_{V_{\tau,k}^{-1}},
\end{align*}
which concludes the proof.
\end{proof}

Then we define event $E=\{|z_{n}^\top (\widehat{\theta}_{k,\tau}-\theta_k^*)|\le \beta_T\|z_{n}\|_{V_{k,\tau}^{-1}} \: \forall \tau\in[T], k\in[K],  n\in [N]\}$ for some $c_1>0$, which holds at least $1-1/T$ with Lemma~\ref{lem:elim_conf} and  $\delta=1/T$. 

\begin{lemma} \label{lem:R_UCBR_bd} Under $E$, for all $\tau\in [T]$, $k\in [K]$, and $S\subseteq \Ncal_{k,\tau-1}$, we have
\begin{align*}
    0\le R_{k,\tau}^{UCB}(S)- R_{k}(S)\le 4\beta_T\max_{n\in S} \|z_n\|_{V_{k,\tau}^{-1}} \text{ and } -4\beta_T\max_{n\in S} \|z_n\|_{V_{k,\tau}^{-1}}\le R_{k,\tau}^{LCB}(S)-R_{k}(S)\le 0
\end{align*}
\end{lemma}
\begin{proof}
     Let $u_{n,k}=z_n^\top\theta_k^*$, $\widehat{u}_{n,k}=z_n^\top \widehat{\theta}_{k,\tau}$, and $\widehat{R}_{k,\tau}(S)=\frac{\sum_{n\in S}w_{n,k}\exp(\widehat{u}_{n,k})}{1+\sum_{m\in S}\exp(\widehat{u}_{m,k})}$.  
         Then by the mean value theorem, there exists $\bar{u}_{n,k}=(1-c)\widehat{u}_{n,k}+cu_{n,k}$ for some $c\in(0,1)$ satisfying, for any $S\subset \mathcal{N}_{k,\tau-1}$
         \begin{align*}
             \left|\widehat{R}_{k,\tau}(S)-R_k(S)\right|&=\left|\frac{\sum_{n\in S}w_{n,k}\exp(\widehat{u}_{n,k})}{1+\sum_{m\in S}\exp(\widehat{u}_{n,k})}-\frac{\sum_{n\in S}w_{n,k}\exp(u_{n,k})}{1+\sum_{m\in S}\exp(u_{m,k})}\right|\cr &=\left|\sum_{n\in S}\nabla_{v_n}\left(\frac{\sum_{m\in S}w_{m,k}\exp(v_m)}{1+\sum_{m\in S}\exp(v_m)}\right)\Big|_{v_n=\bar{u}_{n,k}}(\widehat{u}_{n,k}-u_{n,k})\right|
             \cr &\le \left|\frac{(1+\sum_{n\in S}\exp(\bar{u}_{n,k}))(\sum_{n\in S}w_{n,k}\exp(\bar{u}_{n,k})(\widehat{u}_{n,k}-u_{n,k}))}{(1+\sum_{n\in S}\exp(\bar{u}_{n,k}))^2}\right|  
             \cr 
             &\quad +\left|\frac{(\sum_{n\in S}\exp(\bar{u}_{n,k}))(\sum_{n\in S}w_{n,k}\exp(\bar{u}_{n,k})(\widehat{u}_{n,k}-u_{n,k}))}{(1+\sum_{n\in S}\exp(\bar{u}_{n,k}))^2}\right|  \cr 
             &\le  2\sum_{n\in S}\frac{\exp(\bar{u}_{n,k})}{1+\sum_{m\in S}\exp(\bar{u}_{m,k})}|\widehat{u}_{n,k}-u_{n,k}|\cr &\le 2\max_{n\in S} |\widehat{u}_{n,k}-u_{n,k}|\cr &\le2\beta_T\max_{n\in S} \|z_n\|_{V_{k,\tau}^{-1}},
         \end{align*}
         where the  last inequality is obtained from, under $E$, $|z_n^\top \theta_k^*-z_n^\top \widehat{\theta}_{k,\tau}|\le \beta_T\|z_n\|_{V_{k,\tau}^{-1}}$.
         Then, from the definition of $R^{UCB}_{k,\tau}(S)$ and $R^{LCB}_{k,\tau}(S)$, we can conclude the proof.
\end{proof}
    
In the following, by adopting the proof technique in \cite{chen2023robust}, we provide a lemma for showing that $\mathcal{M}_{\tau}$ is likely to contain the optimal assortment.   
\begin{lemma} Under $E$, 
$(S_1^*,\dots,S_K^*)\in\mathcal{M}_{\tau-1}$ for all $\tau\in[T]$.\label{lem:S*inM}
\end{lemma}
\begin{proof} Here we use induction for the proof.
    Suppose that for fixed $\tau$, we have $(S_1^*,\dots,S_K^*)\in\mathcal{M}_{\tau}$ for all $k\in[K]$. Recall that $\beta_T=(C_1/\kappa)\sqrt{\log(TKN) }$. From Lemma~\ref{lem:R_UCBR_bd}, we have $R_{k,\tau+1}^{UCB}(S)\ge R_{k}(S)$ and $R_{k,\tau+1}^{LCB}(S)\le R_{k}(S)$ for any $S\subset[N]$.
    Then for $k\in[K]$, $n\in S_k^*$, and any $(S_1,..,S_K)\in\mathcal{M}_\tau$,   we have
    \begin{align}
\sum_{l\in[K]}R^{UCB}_{l,\tau+1}(S_{l,\tau+1}^{(n,k)})&\ge \sum_{l\in[K]}R^{UCB}_{l,\tau+1}(S_{l}^*)\cr &\ge \sum_{l\in[K]}R_{l}(S_{l}^*)\cr &\ge  \sum_{l\in[K]}R_{l}(S_l)\cr &\ge\sum_{l\in[K]}R^{LCB}_{l,\tau+1}(S_l),\label{eq:R_ucb_R_lcb}
    \end{align}
 where the first inequality comes from the elimination condition in the algorithm and $(S_1^*,\dots S_K^*)\in \mathcal{M}_\tau$, and the third inequality comes from the optimality of $(S_1^*,\dots,S_K^*)$. This implies that $n\in \mathcal{N}_{k,\tau+1}$ from the algorithm. Then by following the same statement of \eqref{eq:R_ucb_R_lcb} for all $n\in S_k^*$ and $k\in[K]$, we have $S_k^*\subset\mathcal{N}_{k,\tau+1}$ for all $k\in[K]$, which implies $(S_1^*,\dots,S_K^*)\in\mathcal{M}_{\tau+1}$. Therefore, with $(S_1^*,\dots,S_K^*)\in\mathcal{M}_{1}$, we can conclude the proof from the induction.
\end{proof}
From the above Lemmas~\ref{lem:R_UCBR_bd} and \ref{lem:S*inM}, under $E$, we have
\begin{align}
    \sum_{l\in[K]}{R}_l(S_l^*)-\sum_{l\in[K]}R_l(S_{l,\tau}^{(n,k)})&\le \sum_{l\in[K]}R^{LCB}_{l,\tau}(S_l^*)+4\beta_T \max_{m\in S_l^*}\|z_{m}\|_{V_{l,\tau}^{-1}}\cr & \quad-\sum_{l\in[K]}R^{UCB}_{l,\tau-1}(S_{l,\tau}^{(n,k)})+4\beta_T \max_{m\in S_{l,\tau}^{(n,k)}}\|z_{m}\|_{V_{l,\tau-1}^{-1}}
    \cr  &\le 4\beta_T \sum_{l\in[K]}(\max_{m\in S_l^*}\|z_{m}\|_{V_{l,\tau-1}^{-1}}+\max_{m\in S_{l,\tau}^{(n,k)}}\|z_{m}\|_{V_{l,\tau-1}^{-1}}),\label{eq:instance_r_bd_max_z}
\end{align}
where the last inequality comes from the fact that $(S_1^*,\dots,S_K^*)\in\mathcal{M}_{\tau-1}$ and $\max_{(S_1,...,S_K)\in \mathcal{M}_{\tau-1}}\sum_{l\in[K]}R^{LCB}_{l,\tau}(S_l)\le \sum_{l\in[K]}R^{UCB}_{l,\tau}(S_{l,\tau}^{(n,k)})$ from the algorithm.

We define $V(\pi_{k,\tau})=\sum_{n\in\mathcal{N}_{k,\tau}}\pi_{k,\tau}(n) z_{n}z_{n}^\top$ and $\mathrm{supp}(\pi_{k,\tau})=\{n\in \mathcal{N}_{k,\tau}: \pi_{k,\tau}(n)\neq 0\}$. 
Then we have the following lemma from the G/D-optimal design problem. 
 \begin{lemma}[Theorem 21.1 (Kiefer-Wolfowitz) in \citet{lattimore2020bandit}] \label{lem:D_opt}
    For all $\tau\in[T]$ and $k\in[K]$,  we have 
    \begin{align*}
\max_{n\in\mathcal{N}_{k,\tau}}\|z_{n}\|_{(V(\pi_{k,\tau})+ (1/rT_\tau)I_r)^{-1}}^2\le r \text{ and } |\mathrm{supp}(\pi_{k,\tau})|\le r(r+1)/2.
    \end{align*}
 \end{lemma}
\begin{proof}
For completeness, we provide a proof in Appendix~\ref{app:proof-kw}.
\end{proof}
 
 From the definition of $V_{k,\tau}$ and $T_\tau$,  we have
\begin{align}
    V_{k,\tau}&\succeq \sum_{n\in \mathcal{N}_{k,\tau-1}}r\pi_{k,\tau-1}(n)T_{\tau-1} z_{n}z_{n}^\top +I_r\cr &=T_{\tau-1} r (V(\pi_{k,\tau-1})+(1/T_{\tau-1} r)I_r ).\label{eq:V_lower2}
\end{align}
 Then from Lemma~\ref{lem:D_opt} and  \eqref{eq:V_lower2}, for any $n\in\mathcal{N}_{k,\tau}$ we have
\begin{align}
    \beta_T\|z_n\|_{V_{k,\tau}^{-1}}&= (1/\kappa)\sqrt{\|z_{n}\|_{V_{k,\tau}^{-1}}^2\log(KNT)}\cr &=\tilde{\Ocal}\left((1/\kappa) \sqrt{1/T_{\tau-1}}  \sqrt{\|z_{n}\|_{(V(\pi_{k,\tau-1})+(1/T_{\tau-1}r)I_r)^{-1}}^2/r}\right) \cr &=\tilde{\Ocal}((1/\kappa)\sqrt{1/T_{\tau-1}}).\label{eq:sqrt_z_norm_bd}
\end{align}
Therefore under $E$,  from \eqref{eq:instance_r_bd_max_z} and \eqref{eq:sqrt_z_norm_bd}, for $\tau>1$, we have 
\[\sum_{l\in[K]}(R_l(S_l^*)-R_l(S_{l,\tau}^{(n,k)}))=\tilde{\Ocal}((1/\kappa)K\sqrt{1/T_{\tau-1}}).\] 

We have 
\begin{align}
\mathcal{R}(T)&=\mathbb{E}\left[\sum_{t\in[T]}\sum_{k\in[K]}R_k(S_{k}^*)-R_k(S_{k,t})\right]
\cr &\le\mathbb{E}\left[\sum_{\tau \in [\tau_T]}\sum_{l\in[K]}\sum_{t\in\mathcal{T}_{l,\tau}^{(1)}\bigcap \mathcal{T}_{l,\tau}^{(2)}}\sum_{k\in[K]}R_k(S_{k}^*)-R_k(S_{k,t})\right], \cr \label{eq:R_decom_warmup_main}\end{align}
which consists of regret from the stage of warming up and main. We first analyze the regret from the warming-up as follows:
\begin{align}
\mathbb{E}\left[\sum_{\tau\in [\tau_T]}\sum_{l\in[K]}\sum_{t\in\mathcal{T}_{l,\tau}^{(1)}}\sum_{k\in[K]}R_k(S_{k}^*)-R_k(S_{k,t})\right]&\le \mathbb{E}\left[\sum_{\tau\in[\tau_T]}\sum_{l\in[K]}K\left|\Tcal_{l,\tau}^{(1)}\right|\right]\cr &
=\tilde{\Ocal}(r^2K^2 N/(\min\{L,N\}\kappa^2\lambda_{\min})),\label{eq:r_warming_bd}
\end{align}
where the  first equality comes from $\tau_T\le M=\mathcal{O}(\log(\log(T/rK)))$ from Proposition~\ref{prop:batch_bd}.

For the regret bound from the main part of the algorithm, with large enough $T$, we have
\begin{align}
&\mathbb{E}\left[\sum_{\tau\in[\tau_T]}\sum_{l\in[K]}\sum_{t\in\mathcal{T}_{l,\tau_T}^{(2)}}\sum_{k\in[K]}R_k(S_{k}^*)-R_t(S_{k,t})\right]
\cr &\le \mathbb{E}\left[ \sum_{\tau\in[\tau_T]}\sum_{l\in[K]}\sum_{t\in\mathcal{T}_{l,\tau}^{(2)}}\sum_{k\in[K]}\left({R}_k(S_{k}^*)-R_k(S_{k,t})\right)\mathbf{1}(E)\right]\cr &\qquad+\mathbb{E}\left[\sum_{\tau\in[\tau_T]} \sum_{l\in[K]}\sum_{t\in\mathcal{T}_{l,\tau}^{(2)}}\sum_{k\in[K]}\left({R}_k(S_{k}^*)-R_k(S_{k,t})\right)\mathbf{1}(E^c)\right]\cr
&= \tilde{\Ocal}\left((K/\kappa)\sum_{\tau=2}^{\tau_T}\sum_{l\in[K]}\sum_{n\in\mathcal{N}_{l,\tau}}|\mathcal{T}_{n,l,\tau}^{(2)}|\sqrt{1/T_{\tau-1}}\right)+\Ocal( rK\eta_T)+\Ocal(K) 
\cr
&= \tilde{\Ocal}\left( (K/\kappa)\sum_{\tau=2}^{\tau_T}\sum_{l\in[K]}\sum_{n\in\mathcal{N}_{l,\tau}}|\mathcal{T}_{n,l,\tau}^{(2)}|\sqrt{1/T_{\tau-1}}\right)+\Ocal( rK\eta_T)
\cr
&=\tilde{\Ocal}\left((K/\kappa)\sum_{\tau=2}^{\tau_T}\sum_{l\in[K]}(rT_\tau+|\mathrm{supp}(\pi_{l,\tau})|)\sqrt{1/T_{\tau-1}}\right)+\Ocal( rK\eta_T) \cr
&=\tilde{\Ocal}\left((K^2/\kappa)\sum_{\tau=2}^{\tau_T}(r\eta_T+r^2 \sqrt{1/T_{\tau-1}})\right)\cr
&=\tilde{\Ocal}\left((K^2/\kappa)(r\eta_T+r^2)\right)\cr&= \tilde{\Ocal}\left(\tfrac{1}{\kappa}rK^2(T/rK)^{\frac{1}{2(1-2^{-M})}}\right),
\label{eq:R_main_bd}
\end{align}
where the third last equality comes from Lemma~\ref{lem:D_opt} and the second last equality comes from $\tau_T\le M=\mathcal{O}(\log(\log(T/rK)))$ from Proposition~\ref{prop:batch_bd}.
From \eqref{eq:R_decom_warmup_main}, \eqref{eq:r_warming_bd}, \eqref{eq:R_main_bd}, for $T\ge r^3KN^2/\min\{L,N\}^2\kappa^2\lambda_{\min}^2$, we can conclude the proof.

\subsection{Proof of Theorem~\ref{thm:elim2}}\label{app:thm_elim2}
Let $g_{k,\tau}(\theta)=\sum_{t\in \Tcal_{\tau-1}}\sum_{n\in S_{k,t}}p(n|S_{k,t},\theta)z_n+\lambda\theta$ and $\zeta_{\tau}(\delta)=\frac{1}{2}\sqrt{\lambda}+\frac{2r}{\sqrt{\lambda}}\log\left(\frac{4K}{\delta}\left(1+\frac{2(t_\tau-1) L}{r\lambda}\right)\right).$

\begin{lemma} [Proposition 2 in \cite{goyal2021dynamic}]With probability at least $1-\delta$, for all $\tau\ge 1$ and  $k\in[K]$, we have
\begin{align*}
\|g_{k,\tau}(\widehat{\theta}_{k,\tau})-g_{k,\tau}(\theta_k^*)\|_{H_{k,\tau}^{-1}(\theta_k^*)}\le \zeta_\tau(\delta).   
\end{align*}\label{lem:g_gap_bd_zeta}
\end{lemma}

From the above lemma, we define event $E=\{\|g_{k,\tau}(\widehat{\theta}_{k,\tau})-g_{k,\tau}(\theta_k^*)\|_{H_{k,\tau}^{-1}(\theta_k^*)}\le \zeta_{\tau}(\delta),\: \forall \tau\ge 1, k\in[K]\}$. Then we have the following lemma.

\begin{lemma}
      Under $E$, for any $\tau\ge 1$ and $k\in[K]$, we have 
    \begin{align*}
        \|\widehat{\theta}_{k,\tau}-\theta^*_k\|_{H_{k,\tau}(\widehat{\theta}_{k,\tau})}\le (1+3\sqrt{2}) \zeta_\tau(\delta).
\end{align*}\label{lem:confi2} 
\end{lemma}
\begin{proof}
   Here we utilize the proof techniques in \cite{goyal2021dynamic}. Let $G_{k,\tau}(\theta_1,\theta_2)=
    \int_{v=0}^1\nabla g_{k,\tau}(\theta_1+v(\theta_2-\theta_1))dv$. By the multivariate mean value theorem,  we have
    \begin{align}
    g_{k,\tau}(\theta_1)-g_{k,\tau}(\theta_2)=\int_{v=0}^1\nabla g_{k,\tau}(\theta_1+v(\theta_2-\theta_1))dv(\theta_1-\theta_2)=G_{k,\tau}(\theta_1,\theta_2)(\theta_1-\theta_2),
    \end{align}
    which implies \[\|g_{k,\tau}(\theta_1)-g_{k,\tau}(\theta_2)\|_{G_{k,\tau}^{-1}(\theta_1,\theta_2)}=\|\theta_1-\theta_2\|_{G_{k,\tau}(\theta_1,\theta_2)}.\]
    By following the proof steps of Proposition 3 in \cite{goyal2021dynamic} with Proposition C.1 in \cite{lee2024nearly}, we can show that 
    \[G_{k,\tau}(\theta_1,\theta_2)\succeq \frac{1}{1+3\sqrt{2}}H_{k,\tau}(\theta_1) \text{ and } G_{k,\tau}(\theta_1,\theta_2)\succeq \frac{1}{1+3\sqrt{2}}H_{k,\tau}(\theta_2).\]
Finally, we have 
\begin{align*}
    \|\theta_1-\theta_2\|_{H_{k,\tau}(\theta_1)}&\le (1+3\sqrt{2})^{1/2}\|\theta_1-\theta_2\|_{G_{k,\tau}(\theta_1,\theta_2)}
    \cr &=(1+3\sqrt{2})^{1/2} \|g_{k,\tau}(\theta_1)-g_{k,\tau}(\theta_2)\|_{G_{k,\tau}^{-1}(\theta_1,\theta_2)}
    \cr &\le (1+3\sqrt{2})\|g_{k,\tau}(\theta_1)-g_{k,\tau}(\theta_2)\|_{H_{k,\tau}^{-1}(\theta_2)},
\end{align*}
    which concludes the proof with $E$.

\end{proof}

From the above lemma and $E$ with $\delta=1/T$, with probability at least $1-(1/T)$, for all $\tau\ge 1$ and $k\in[K]$, we have
\[|z_{n}^\top (\widehat{\theta}_{k,\tau}-\theta_k^*)|\le \|z_n\|_{H_{k,\tau}^{-1}(\widehat{\theta}_{k,\tau})}\|\widehat{\theta}_{k,\tau}-\theta_k^*\|_{H_{k,\tau}(\widehat{\theta}_{k,\tau})}\le \zeta_{\tau}\|z_n\|_{H^{-1}_{k,\tau}(\widehat{\theta}_{k,\tau})}.\]

In the following proof, with a slight abuse of notation, we define  $E=\{|z_{n}^\top (\widehat{\theta}_{k,\tau}-\theta_k^*)|\le \zeta_{\tau}\|z_n\|_{H_{k,\tau}^{-1}(\widehat{\theta}_{k,\tau})} \: \forall \tau\ge 1, k\in[K],  n\in [N]\}$, which holds at least $1-(1/T)$. We also  use $p(n|S,\theta)=\exp(z_n^\top \theta)/(1+\sum_{m\in S}\exp(z_m^\top \theta))$ with $z_n$ instead of $x_n$.

\begin{lemma}\label{lem:R_UCBR_bd_ad}
     Under $E$, for all $k\in[K]$ and $\tau\in[T]$, for any $S\subset \mathcal{N}_{k,\tau-1}$, we have 
     \begin{align*}
     0&\le R_{k,\tau}^{UCB}(S)-R_k(S)\cr &\le  13\zeta_\tau^2\max_{n\in S}\|z_n\|_{H_{k,\tau}^{-1}(\widehat{\theta}_{k,\tau})}^2+4\zeta_\tau^2\max_{n\in S}\|\tilde{z}_{n,k,\tau}\|_{H_{k,\tau}^{-1}(\widehat{\theta}_{k,\tau})}^2+ 2\zeta_\tau\sum_{n\in S}p(n|S,\widehat{\theta}_{k,\tau-1})\|\tilde{z}_{n,k,\tau}\|_{H_{k,\tau}^{-1}(\widehat{\theta}_{k,\tau})}, 
     \end{align*}
     \begin{align*}
         0&\le R_k(S)-R^{LCB}_{k,\tau}(S)\cr &\le 13\zeta_\tau^2\max_{n\in S}\|z_n\|_{H_{k,\tau}^{-1}(\widehat{\theta}_{k,\tau})}^2+4\zeta_\tau^2\max_{n\in S}\|\tilde{z}_{n,k,\tau}\|_{H_{k,\tau}^{-1}(\widehat{\theta}_{k,\tau})}+ 2\zeta_\tau\sum_{n\in S}p(n|S,\widehat{\theta}_{k,\tau-1})\|\tilde{z}_{n,k,\tau}\|_{H_{k,\tau}^{-1}(\widehat{\theta}_{k,\tau})}.
     \end{align*}
 \end{lemma}
 \begin{proof} 
     Let $u_{n,k}=z_n^\top\theta_k^*$, $\widehat{u}_{n,k}=z_n^\top \widehat{\theta}_{k,\tau}$, and $\widehat{R}_{k,\tau}(S)=\frac{\sum_{n\in S}w_{n,k}\exp(\widehat{u}_{n,k})}{1+\sum_{m\in S}\exp(\widehat{u}_{m,k})}$. We also define $u_{n,k}=z_n^\top\theta_k^*$, $\textbf{u}_k=(u_{n,k}: n\in S)$, $\widehat{\textbf{u}}_{k,\tau}=(\widehat{u}_{n,k,\tau}:n\in S)$, and $Q(\textbf{v})=\sum_{n\in S}\frac{
              w_{n,k}\exp(v_n)}{1+\sum_{m\in S}\exp(v_m)}$. 
         Then by a second-order Taylor expansion, we have
         \begin{align}
             \left|\widehat{R}_{k,\tau}(S)-R_k(S)\right|&=\left|Q(\widehat{\textbf{u}}_{k,\tau})-Q(\textbf{u}_k)\right|
\cr&=\left|\nabla Q(\textbf{u}_k)^\top (\widehat{\textbf{u}}_{k,\tau}-\textbf{u}_k)\right| +\left|\frac{1}{2} (\widehat{\textbf{u}}_{k,\tau}-\textbf{u}_k)^\top \nabla^2 Q(\bar{\textbf{u}}_k) (\widehat{\textbf{u}}_{k,\tau}-\textbf{u}_k)\right|,\label{eq:taylor}
\end{align}
where $\bar{\textbf{u}}_k$ is the convex combination of $\widehat{\textbf{u}}_{k,\tau}$ and $\textbf{u}_k$. Let $e_{n,k,\tau}=\widehat{u}_{n,k,\tau}-u_{n,k}$, $e_{n_0,k,\tau}=0$, $\bar{e}_{n,k,\tau}=e_{n,k,\tau}-\sum_{m\in S\cup \{n_0\}}p(m|S,\theta_k^*)e_{m,k,\tau}=e_{n,k,\tau}-\EE_{\theta_k^*}[e_{m,k,\tau}]$, and $\tilde{e}_{n,k,\tau}=e_{n,k,\tau}-\sum_{m\in S\cup \{n_0\}}p(m|S,\widehat{\theta}_{k,\tau})e_{m,k,\tau}=e_{n,k,\tau}-\EE_{\widehat{\theta}_{k,\tau}}[e_{m,k,\tau}]$.
         Then the first-order term in the above is bounded as 
         \begin{align*}
             &\left|\nabla Q(\textbf{u}_k)^\top (\widehat{\textbf{u}}_{k,\tau}-\textbf{u}_k)\right|\cr &= \left|\frac{\sum_{n\in S}w_{n,k}\exp({u}_{n,k})(\widehat{u}_{n,k,\tau}-u_{n,k})}{1+\sum_{n\in S}\exp({u}_{n,k})} -\frac{(\sum_{n\in S}w_{n,k}\exp({u}_{n,k}))(\sum_{n\in S}\exp({u}_{n,k})(\widehat{u}_{n,k,\tau}-u_{n,k}))}{(1+\sum_{n\in S}\exp({u}_{n,k}))^2}\right|
             \cr &= \left|\sum_{n\in S}w_{n,k}p(n|S,\theta_k^*)(\widehat{u}_{n,k,\tau}-u_{n,k})-\sum_{n,m\in S}w_{m,k}p(n|S,\theta_k^*)p(m|S,\theta_k^*)(\widehat{u}_{n,k,\tau}-u_{n,k})\right|
               \cr &= \left|\sum_{n\in S}w_{n,k}p(n|S,\theta_k^*)\left((\widehat{u}_{n,k,\tau}-u_{n,k})-\sum_{m\in S}p(m|S,\theta_k^*)(\widehat{u}_{m,k,\tau}-u_{m,k})\right)\right|
                   \cr & \le \sum_{n\in S}w_{n,k}p(n|S,\theta_k^*)\left|e_{n,k,\tau}-\EE_{\theta_k^*}[e_{m,k,\tau}]\right|
                   \cr & \le \sum_{n\in S}p(n|S,\theta_k^*)\left|e_{n,k,\tau}-\EE_{\theta_k^*}[e_{m,k,\tau}]\right|
                        \cr & = \sum_{n\in S}p(n|S,\theta_k^*)\left|\bar{e}_{n,k,\tau}\right|
                         \cr & \le \sum_{n\in S}p(n|S,\theta_k^*)\left|\bar{e}_{n,k,\tau}-\tilde{e}_{n,k,\tau}\right|+ \sum_{n\in S}p(n|S,\theta_k^*)\left|\tilde{e}_{n,k,\tau}\right|\end{align*}
              For the first term above, we have
              \begin{align*}
                  &\sum_{n\in S}p(n|S,\theta_k^*)\left|\bar{e}_{n,k,\tau}-\tilde{e}_{n,k,\tau}\right|\cr &=\sum_{n\in S}p(n|S,\theta_k^*)\left|\EE_{ \theta_k^*}[{e}_{m,k,\tau}]-\EE_{\widehat{\theta}_{k,\tau}}[{e}_{m,k,\tau}]\right|\cr &=\sum_{n\in S}p(n|S,\theta_k^*)\left|\sum_{m\in S}(p(m|S,\theta_k^*)-p(m|S,\widehat{\theta}_{k,\tau})){e}_{m,k,\tau}\right|\cr& \le
                  2\zeta_\tau^2\sum_{n\in S}p(n|S,\theta_k^*)\|z_n\|_{H_{k,\tau}^{-1}}^2
                  \cr& \le
                  2\zeta_\tau^2\max_{n\in S}\|z_n\|_{H_{k,\tau}^{-1}}^2,
              \end{align*}
                         where the first inequality is obtained by using the mean value theorem. Then for the second term, we have

    \begin{align*}
        \sum_{n\in S}p(n|S,\theta_k^*)|\tilde{e}_{n,k,\tau}|&\le \sum_{n \in S}(p(n|S,\theta_k^*)-p(n|S,\widehat{\theta}_{k,\tau-1}))|\tilde{e}_{n,k,\tau}|+\sum_{n\in S}p(n|S,\widehat{\theta}_{k,\tau-1})|\tilde{e}_{n,k,\tau}|
                             \cr & \le 2\zeta_\tau\max_{n\in S}\|z_n\|_{H_{k,\tau}^{-1}} |(\widehat{\theta}_{k,\tau}-\theta_k^*)^\top(z_n-\EE_{\widehat{\theta}_{k,\tau}}[z_n])| \cr &\qquad+ \sum_{n\in S}p(n|S, \widehat{\theta}_{k,\tau-1})|(\widehat{\theta}_{k,\tau}-\theta_k^*)^\top(z_n-\EE_{\widehat{\theta}_{k,\tau}}[z_n])|
                             \cr &\le 2\zeta_\tau^2 (\max_{n\in S}\|z_n\|_{H_{k,\tau}^{-1}}^2+\max_{n\in S}\|\tilde{z}_{n,k,\tau}\|_{H_{k,\tau}^{-1}}^2)+ \zeta_\tau\sum_{n\in S}p(n|S, \widehat{\theta}_{k,\tau-1})\|\tilde{z}_{n,k,\tau}\|_{H_{k,\tau}^{-1}}.
                         \end{align*}
From the above inequalities, we have
\begin{align}
    \left|\nabla Q(\textbf{u}_k)^\top (\widehat{\textbf{u}}_{k,\tau}-\textbf{u}_k)\right|\le 4\zeta_\tau^2\max_{n\in S}\|z_n\|_{H_{k,\tau}^{-1}}^2+2\zeta_\tau^2\max_{n\in S}\|\tilde{z}_{n,k,\tau}\|_{H_{k,\tau}^{-1}}^2+\zeta_\tau\sum_{n\in S}p(n|S, \widehat{\theta}_{k,\tau-1})\|\tilde{z}_{n,k,\tau}\|_{H_{k,\tau}^{-1}}.\label{eq:first-order-bd}
\end{align}

         Now we focus on the second-order term which is bounded as 
         \begin{align}
             &\left|\frac{1}{2} (\widehat{\textbf{u}}_{k,\tau}-\textbf{u}_k)^\top \nabla^2 Q(\bar{\textbf{u}}_k) (\widehat{\textbf{u}}_{k,\tau}-\textbf{u}_k)\right| \cr &= \left|\frac{1}{2}\sum_{n,m\in S} (\widehat{u}_{n,k,\tau}-u_{n,k})\frac{\partial^2 Q(\bar{\textbf{u}}_k)}{\partial_n\partial_m}(\widehat{u}_{m,k,\tau}-u_{m,k})\right|
\cr &=\left|\frac{1}{2}\sum_{n,m\in S} (\widehat{u}_{n,k,\tau}-u_{n,k})\frac{\partial^2 Q(\bar{\textbf{u}}_k)}{\partial_n\partial_m}(\widehat{u}_{m,k,\tau}-u_{m,k})+ \frac{1}{2}\sum_{n,m\in S} (\widehat{u}_{n,k,\tau}-u_{n,k})\frac{\partial^2 Q(\bar{\textbf{u}}_k)}{\partial_n\partial_m}(\widehat{u}_{m,k,\tau}-u_{m,k})\right|\cr &\le 
\sum_{n,m\in S}|\widehat{u}_{n,k,\tau}-u_{n,k}| \frac{\exp(\bar{u}_{n,k})}{1+\sum_{l\in S}\exp(\bar{u}_{l,k})}\frac{\exp(\bar{u}_{m,k})}{1+\sum_{l\in S}\exp(\bar{u}_{l,k})}|\widehat{u}_{m,k,\tau}-u_{m,k}| \cr & \qquad+\frac{3}{2}\sum_{n\in S}(\widehat{u}_{n,k,\tau}-u_{n,k})^2\frac{\exp(\bar{u}_{n,k})}{1+\sum_{l\in S}\exp(\bar{u}_{l,k})}\cr 
&\le \frac{5}{2}\sum_{n\in S}(\widehat{u}_{n,k,\tau}-u_{n,k})^2 \frac{\exp(\bar{u}_{n,k})}{1+\sum_{l\in S}\exp(\bar{u}_{l,k})}
\cr &\le
\frac{5}{2}\zeta_\tau^2\max_{n\in S}\|z_n\|_{H_{k,\tau}^{-1}(\widehat{\theta}_{k,\tau})}^2,\label{eq:second-order-bd}
         \end{align}
         where the first inequality is obtained from Lemma~\ref{lem:grad_Q_bd} and the second inequality is obtained from AM-GM inequality. Then from \eqref{eq:taylor}, \eqref{eq:first-order-bd}, \eqref{eq:second-order-bd}, and with the definition of $R^{UCB}_{k,\tau}(S)$ and $R^{LCB}_{k,\tau}(S)$, we can conclude the proof.
 \end{proof}

In the following, similar to Lemma~\ref{lem:S*inM}, we provide a lemma for showing that $\mathcal{M}_{\tau}$ is likely to contain the optimal assortment.   

\begin{lemma} Under $E$,
$(S_1^*,\dots,S_K^*)\in\mathcal{M}_{\tau-1}$ for all $\tau\in[T]$.\label{lem:S*inM2}
\end{lemma}
\begin{proof} Here we use induction for the proof.
    Suppose that for fixed $\tau$, we have $(S_1^*,\dots,S_K^*)\in\mathcal{M}_{\tau}$ for all $k\in[K]$.  From $E$, we have $R_{k,\tau+1}^{UCB}(S)\ge R_{k}(S)$ and $R_{k,\tau+1}^{LCB}(S)\le R_{k}(S)$ for any $S\subset[N]$.
    Then for $k\in[K]$, $n\in S_k^*$, and any $(S_1,..,S_K)\in\mathcal{M}_\tau$,   we have
    \begin{align}
\sum_{l\in[K]}R^{UCB}_{l,\tau+1}(S_{l,\tau+1}^{(n,k)})&\ge \sum_{l\in[K]}R^{UCB}_{l,\tau+1}(S_{l}^*)\cr &\ge \sum_{l\in[K]}R_{l}(S_{l}^*)\cr &\ge  \sum_{l\in[K]}R_{l}(S_l)\cr &\ge\sum_{l\in[K]}R^{LCB}_{l,\tau+1}(S_l),\label{eq:R_ucb_R_lcb_ad}
    \end{align}
 where the first inequality comes from the elimination condition in the algorithm and $(S_1^*,\dots S_K^*)\in \mathcal{M}_\tau$, and the third inequality comes from the optimality of $(S_1^*,\dots,S_K^*)$. This implies that $n\in \mathcal{N}_{k,\tau+1}'$ from the algorithm. Then by following the same statement of \eqref{eq:R_ucb_R_lcb_ad} for all $n\in S_k^*$ and $k\in[K]$, we have $S_k^*\subseteq\mathcal{N}_{k,\tau+1}'$ for all $k\in[K]$.
 
   Then for $k\in[K]$, $J=S_k^*$, and any $(S_1,..,S_K)\in\mathcal{M}_\tau$,   we have
    \begin{align}
\sum_{l\in[K]}R^{UCB}_{l,\tau+1}(S_{l,\tau+1}^{J})&\ge \sum_{l\in[K]}R^{UCB}_{l,\tau+1}(S_{l}^*)\cr &\ge \sum_{l\in[K]}R_{l}(S_{l}^*)\cr &\ge  \sum_{l\in[K]}R_{l}(S_l)\cr &\ge\sum_{l\in[K]}R^{LCB}_{l,\tau+1}(S_l),\label{eq:R_ucb_R_lcb_ad2}
    \end{align}
 where the first inequality comes from the elimination condition in the algorithm and $(S_1^*,\dots S_K^*)\in \mathcal{M}_\tau$, and the third inequality comes from the optimality of $(S_1^*,\dots,S_K^*)$. This implies that $J(=S_k^*)\in \mathcal{J}(\mathcal{N}_{k,\tau+1}')$ from the algorithm. Then by following the same statement of \eqref{eq:R_ucb_R_lcb_ad2} for all $k\in[K]$, we have $S_k^*\subseteq\mathcal{N}_{k,\tau+1}$ for all $k\in[K]$, which implies $(S_1^*,\dots,S_K^*)\in\mathcal{M}_{\tau+1}$. Therefore, with $(S_1^*,\dots,S_K^*)\in\mathcal{M}_{1}$, we can conclude the proof from the induction.

\end{proof}

We define $\bar{V}(\bar{\pi}_{k,\tau})=\sum_{n\in J\in\mathcal{J}_{k,\tau}}\bar{\pi}_{k,\tau}(n,J) \tilde{z}_{n,k,\tau}(J)\tilde{z}_{n,k,\tau}(J)^\top$ and  $\tilde{V}(\tilde{\pi}_{k,\tau})=\sum_{J\in\mathcal{J}_{k,\tau}}\tilde{\pi}_{k,\tau}(J) \sum_{n\in J}p(n|J,\widehat{\theta}_{k,\tau}) \tilde{z}_{n,k,\tau}(J)\tilde{z}_{n,k,\tau}(J)^\top$.
Then we have the following lemma from the G/D-optimal design problem. 
 \begin{lemma}[Kiefer-Wolfowitz] \label{lem:D_opt2}
    For all $\tau\in[T]$ and $k\in[K]$, we have 
    \begin{align*}
&\max_{n\in J\in\mathcal{J}(\mathcal{N}_{k,\tau})}\|\tilde{z}_{n,k,\tau}(J)\|_{(\bar{V}(\bar{\pi}_{k, \tau})+(\lambda/T_\tau r) I_r)^{-1}}^2\le r \text{ and } |\mathrm{supp}(\bar{\pi}_{k,\tau})|\le r(r+1)/2,\cr
&\max_{J\in\mathcal{J}(\mathcal{N}_{k,\tau})}\sum_{n\in J}p(n|J,\widehat{\theta}_{k,\tau})\|\tilde{z}_{n,k,\tau}(J)\|_{(\tilde{V}(\tilde{\pi}_{k, \tau})+(\lambda/T_\tau r) I_r)^{-1}}^2\le r \text{ and } |\mathrm{supp}(\tilde{\pi}_{k,\tau})|\le r(r+1)/2.\end{align*}
 \end{lemma}
\begin{proof}
Fix $\tau\in[T]$ and $k\in[K]$, and let
\[
\alpha := \frac{\lambda}{T_\tau r}.
\]

\paragraph{(1) The bound for $\bar\pi_{k,\tau}$.}
Recall the design space
\[
\mathcal K(\mathcal N_{k,\tau})
:=\{(n,J): J\in\mathcal J(\mathcal N_{k,\tau}),\ n\in J\}.
\]
For a distribution $\bar\pi$ over $\mathcal K(\mathcal N_{k,\tau})$, define
\[
\bar V(\bar\pi)
:= \sum_{(n,J)\in\mathcal K(\mathcal N_{k,\tau})}
\bar\pi(n,J)\,\tilde z_{n,k,\tau}(J)\tilde z_{n,k,\tau}(J)^\top,
\qquad
\bar W(\bar\pi) := \bar V(\bar\pi)+\alpha I_r,
\]
and the criterion
\[
\bar g(\bar\pi)
:=\max_{(n,J)\in\mathcal K(\mathcal N_{k,\tau})}
\tilde z_{n,k,\tau}(J)^\top \bar W(\bar\pi)^{-1}\tilde z_{n,k,\tau}(J)
=
\max_{n\in J\in\mathcal J(\mathcal N_{k,\tau})}
\|\tilde z_{n,k,\tau}(J)\|_{\bar W(\bar\pi)^{-1}}^2 .
\]
Let $\bar\pi_{k,\tau}$ be a minimizer of $\bar g(\cdot)$.

\emph{Step 1 (equalization on the support).}
As in Lemma~A.10, we may assume that for every $(n,J)\in\mathrm{supp}(\bar\pi_{k,\tau})$,
\begin{equation}\label{eq:bar_support_tight}
\tilde z_{n,k,\tau}(J)^\top \bar W(\bar\pi_{k,\tau})^{-1}\tilde z_{n,k,\tau}(J)
=\bar g(\bar\pi_{k,\tau}).
\end{equation}
(Otherwise, if some $(n_0,J_0)\in\mathrm{supp}(\bar\pi_{k,\tau})$ is strictly below the maximum,
one can redistribute an arbitrarily small amount of its probability mass to points
attaining the maximum and (by continuity of $\bar g$) obtain another feasible design
with strictly smaller objective, contradicting optimality. This is the same argument
used in Lemma~A.10.)

\emph{Step 2 (trace bound, giving $\bar g(\bar\pi_{k,\tau})\le r$).}
Using \eqref{eq:bar_support_tight} and $\sum_{(n,J)}\bar\pi_{k,\tau}(n,J)=1$,
\begin{align*}
\bar g(\bar\pi_{k,\tau})
&=\sum_{(n,J)}\bar\pi_{k,\tau}(n,J)\,
\tilde z_{n,k,\tau}(J)^\top \bar W(\bar\pi_{k,\tau})^{-1}\tilde z_{n,k,\tau}(J)\\
&=\tr\!\left(\sum_{(n,J)}\bar\pi_{k,\tau}(n,J)\,
\tilde z_{n,k,\tau}(J)\tilde z_{n,k,\tau}(J)^\top
\bar W(\bar\pi_{k,\tau})^{-1}\right)\\
&=\tr\!\left(\bar V(\bar\pi_{k,\tau})\,\bar W(\bar\pi_{k,\tau})^{-1}\right)\\
&=\tr\!\left((\bar W(\bar\pi_{k,\tau})-\alpha I_r)\,\bar W(\bar\pi_{k,\tau})^{-1}\right)\\
&=\tr(I_r)-\alpha\,\tr(\bar W(\bar\pi_{k,\tau})^{-1})
\;\le\; r.
\end{align*}
Hence
\[
\max_{n\in J\in\mathcal J(\mathcal N_{k,\tau})}
\|\tilde z_{n,k,\tau}(J)\|_{(\bar V(\bar\pi_{k,\tau})+\alpha I_r)^{-1}}^2
=\bar g(\bar\pi_{k,\tau})\le r.
\]

\emph{Step 3 (support size bound).}
Let $S:=\mathrm{supp}(\bar\pi_{k,\tau})$ and define
$A_{(n,J)}:=\tilde z_{n,k,\tau}(J)\tilde z_{n,k,\tau}(J)^\top\in\mathbb S^r$.
If $|S|>\dim(\mathbb S^r)=r(r+1)/2$, then the matrices $\{A_{(n,J)}:(n,J)\in S\}$
are linearly dependent, so there exists a nonzero $v:S\to\mathbb R$ such that
\[
\sum_{(n,J)\in S} v(n,J)\,A_{(n,J)}=0.
\]
Taking the trace after multiplying by $\bar W(\bar\pi_{k,\tau})^{-1}$ gives
\begin{align*}
0
&=\tr\!\left(\bar W(\bar\pi_{k,\tau})^{-1}\sum_{(n,J)\in S}v(n,J)A_{(n,J)}\right)\\
&=\sum_{(n,J)\in S} v(n,J)\,
\tilde z_{n,k,\tau}(J)^\top \bar W(\bar\pi_{k,\tau})^{-1}\tilde z_{n,k,\tau}(J)\\
&=\bar g(\bar\pi_{k,\tau})\sum_{(n,J)\in S}v(n,J),
\end{align*}
where the last equality uses \eqref{eq:bar_support_tight}. Thus $\sum_{(n,J)\in S}v(n,J)=0$.
Define $\bar\pi(t)=\bar\pi_{k,\tau}+t v$. Then $\sum_{(n,J)\in S}\bar\pi(t)(n,J)=1$ for all $t$ and
\[
\bar V(\bar\pi(t))=\bar V(\bar\pi_{k,\tau})+t\sum_{(n,J)\in S}v(n,J)A_{(n,J)}=\bar V(\bar\pi_{k,\tau}),
\]
so $\bar W(\bar\pi(t))=\bar W(\bar\pi_{k,\tau})$ and hence $\bar g(\bar\pi(t))=\bar g(\bar\pi_{k,\tau})$.
Let
\[
t^*:=\sup\{t>0:\ \bar\pi_{k,\tau}(n,J)+t\,v(n,J)\ge 0\ \ \forall (n,J)\in S\}.
\]
At $t=t^*$ at least one coordinate becomes $0$, giving an equally optimal design
with strictly smaller support. Iterating yields an optimal design with
$|\mathrm{supp}(\bar\pi_{k,\tau})|\le r(r+1)/2$.

This completes the first claim.

\paragraph{(2) The bound for $\tilde\pi_{k,\tau}$.}
Now the design space is $\mathcal J(\mathcal N_{k,\tau})$. For $\tilde\pi$ on this set define
\[
\tilde V(\tilde\pi)
:=\sum_{J\in\mathcal J(\mathcal N_{k,\tau})}\tilde\pi(J)
\sum_{n\in J}p(n\mid J,\hat\theta_{k,\tau})\,
\tilde z_{n,k,\tau}(J)\tilde z_{n,k,\tau}(J)^\top,
\qquad
\tilde W(\tilde\pi):=\tilde V(\tilde\pi)+\alpha I_r,
\]
and
\[
\tilde g(\tilde\pi)
:=\max_{J\in\mathcal J(\mathcal N_{k,\tau})}
\sum_{n\in J}p(n\mid J,\hat\theta_{k,\tau})\,
\tilde z_{n,k,\tau}(J)^\top \tilde W(\tilde\pi)^{-1}\tilde z_{n,k,\tau}(J).
\]
Let $\tilde\pi_{k,\tau}$ be a minimizer of $\tilde g(\cdot)$.

\emph{Step 1 (equalization on the support).}
As before, for every $J\in\mathrm{supp}(\tilde\pi_{k,\tau})$ we may assume
\begin{equation}\label{eq:tilde_support_tight}
\sum_{n\in J}p(n\mid J,\hat\theta_{k,\tau})\,
\tilde z_{n,k,\tau}(J)^\top \tilde W(\tilde\pi_{k,\tau})^{-1}\tilde z_{n,k,\tau}(J)
=\tilde g(\tilde\pi_{k,\tau}).
\end{equation}

\emph{Step 2 (trace bound).}
Using \eqref{eq:tilde_support_tight},
\begin{align*}
\tilde g(\tilde\pi_{k,\tau})
&=\sum_{J}\tilde\pi_{k,\tau}(J)
\sum_{n\in J}p(n\mid J,\hat\theta_{k,\tau})\,
\tilde z_{n,k,\tau}(J)^\top \tilde W(\tilde\pi_{k,\tau})^{-1}\tilde z_{n,k,\tau}(J)\\
&=\tr\!\left(\tilde V(\tilde\pi_{k,\tau})\,\tilde W(\tilde\pi_{k,\tau})^{-1}\right)\\
&=\tr\!\left((\tilde W(\tilde\pi_{k,\tau})-\alpha I_r)\,\tilde W(\tilde\pi_{k,\tau})^{-1}\right)\\
&=\tr(I_r)-\alpha\,\tr(\tilde W(\tilde\pi_{k,\tau})^{-1})
\;\le\; r.
\end{align*}
Therefore,
\[
\max_{J\in\mathcal J(\mathcal N_{k,\tau})}
\sum_{n\in J}p(n\mid J,\hat\theta_{k,\tau})\,
\|\tilde z_{n,k,\tau}(J)\|_{(\tilde V(\tilde\pi_{k,\tau})+\alpha I_r)^{-1}}^2
=\tilde g(\tilde\pi_{k,\tau})\le r.
\]

\emph{Step 3 (support size bound).}
Let $S:=\mathrm{supp}(\tilde\pi_{k,\tau})$ and define the symmetric matrices
\[
B_J:=\sum_{n\in J}p(n\mid J,\hat\theta_{k,\tau})\,
\tilde z_{n,k,\tau}(J)\tilde z_{n,k,\tau}(J)^\top\in\mathbb S^r.
\]
If $|S|>r(r+1)/2$, then $\{B_J:J\in S\}$ are linearly dependent, so $\exists\,v:S\to\mathbb R\setminus\{0\}$
with $\sum_{J\in S}v(J)B_J=0$. Multiplying by $\tilde W(\tilde\pi_{k,\tau})^{-1}$ and taking traces yields
\[
0=\sum_{J\in S}v(J)\tr(\tilde W(\tilde\pi_{k,\tau})^{-1}B_J)
=\sum_{J\in S}v(J)\,\tilde g(\tilde\pi_{k,\tau}),
\]
where we used \eqref{eq:tilde_support_tight}. Hence $\sum_{J\in S}v(J)=0$.
Define $\tilde\pi(t)=\tilde\pi_{k,\tau}+t v$. Then $\sum_J\tilde\pi(t)(J)=1$ and
$\tilde V(\tilde\pi(t))=\tilde V(\tilde\pi_{k,\tau})$, so $\tilde g(\tilde\pi(t))=\tilde g(\tilde\pi_{k,\tau})$.
Choosing $t^*$ as before to hit the boundary removes at least one support element without changing optimality.
Iterating yields $|\mathrm{supp}(\tilde\pi_{k,\tau})|\le r(r+1)/2$.

This completes the second claim and the proof.
\end{proof}
    
From the above Lemmas~\ref{lem:S*inM2} and \ref{lem:R_UCBR_bd}, under $E$,  we have
\begin{align}
    &\sum_{l\in[K]}{R}_l(S_l^*)-\sum_{l\in[K]}R_l(S_{l,\tau}^{(n,k)})\cr &\le \sum_{l\in[K]}\left[R^{LCB}_{l,\tau}(S_l^*)+13\zeta_\tau^2\max_{m\in S_l^*}\|z_m\|_{H_{l,\tau}^{-1}(\widehat{\theta}_{l,\tau})}^2+4\zeta_\tau^2\max_{m\in S_l^*}\|\tilde{z}_{m,l,\tau}\|_{H_{l,\tau}^{-1}(\widehat{\theta}_{l,\tau})}^2\right.\cr &\qquad\left. +2\zeta_\tau\sum_{m\in S_l^*}p(m|S_l^*,\widehat{\theta}_{l,\tau-1})\|\tilde{z}_{m,l,\tau}\|_{H_{l,\tau}^{-1}(\widehat{\theta}_{l,\tau})}\right] \cr & \quad-\sum_{l\in[K]}\left[R^{UCB}_{l,\tau}(S_{l,\tau}^{(n,k)})-13\zeta_\tau^2\max_{m\in S_{l,\tau}^{(n,k)}}\|z_m\|_{H_{l,\tau}^{-1}(\widehat{\theta}_{l,\tau})}^2-4\zeta_\tau^2\max_{m\in S_{l,\tau}^{(n,k)}}\|z_m\|_{H_{l,\tau-1}^{-1}(\widehat{\theta}_{l,\tau-1})}^2\right.\cr&\qquad\left.- 2\zeta_\tau\sum_{m\in S_{l,\tau}^{(n,k)}}p(m|S_{l,\tau}^{(n,k)},\widehat{\theta}_{l,\tau-1})\|\tilde{z}_{m,l,\tau}\|_{H_{l,\tau}^{-1}(\widehat{\theta}_{l,\tau})}\right]\cr
      &\lesssim \sum_{l\in[K]}\left[\zeta_\tau^2\max_{m\in S_l^*}\|z_m\|_{H_{l,\tau}^{-1}(\widehat{\theta}_{l,\tau})}^2+ \zeta_\tau^2\max_{m\in S_l^*}\|\tilde{z}_{m,l,\tau}\|_{H_{l,\tau}^{-1}(\widehat{\theta}_{l,\tau})}^2+\zeta_\tau\sum_{m\in S_l^*}p(m|S_l^*,\widehat{\theta}_{l,\tau-1})\|\tilde{z}_{m,l,\tau}\|_{H_{l,\tau}^{-1}(\widehat{\theta}_{l,\tau})} \right.\cr&\left.\qquad\qquad+\zeta_\tau^2\max_{m\in S_{l,\tau}^{(n,k)}}\|z_m\|_{H_{l,\tau}^{-1}(\widehat{\theta}_{l,\tau})}^2+\zeta_\tau^2\max_{m\in S_{l,\tau}^{(n,k)}}\|\tilde{z}_{m,l,\tau}\|_{H_{l,\tau}^{-1}(\widehat{\theta}_{l,\tau})}^2\right.\cr &\left.\qquad\qquad\qquad+\zeta_\tau\sum_{m\in S_{l,\tau}^{(n,k)}}p(m|S_{l,\tau}^{(n,k)},\widehat{\theta}_{l,\tau-1})\|\tilde{z}_{m,l,\tau}\|_{H_{l,\tau}^{-1}(\widehat{\theta}_{l,\tau})}\right] \cr  &\le  \sum_{l\in[K]}\left[\zeta_\tau^2\max_{m\in S_l^*}\|z_m\|_{H_{l,\tau}^{-1}(\widehat{\theta}_{l,\tau})}^2+ \zeta_\tau^2\max_{m\in S_l^*}\|\tilde{z}_{m,l,\tau}\|_{H_{l,\tau}^{-1}(\widehat{\theta}_{l,\tau})}+\zeta_\tau^2\max_{m\in S_{l,\tau}^{(n,k)}}\|z_m\|_{H_{l,\tau}^{-1}(\widehat{\theta}_{l,\tau})}^2\right.\cr &\left.\qquad+\zeta_\tau^2\max_{m\in S_{l,\tau}^{(n,k)}}\|\tilde{z}_{m,l,\tau}\|_{H_{l,\tau}^{-1}(\widehat{\theta}_{l,\tau})} +\zeta_\tau\sqrt{\sum_{m\in S_l^*}p(m|S_l^*,\widehat{\theta}_{l,\tau-1})}\sqrt{\sum_{m\in S_l^*}p(m|S_l^*,\widehat{\theta}_{l,\tau-1})\|\tilde{z}_{m,l,\tau}\|_{H_{l,\tau}^{-1}(\widehat{\theta}_{l,\tau})}^2}\right.\cr &\left.\qquad+\zeta_\tau\sqrt{\sum_{m\in S_{l,\tau}^{(n,k)}}p(m|S_{l,\tau}^{(n,k)},\widehat{\theta}_{l,\tau-1})}\sqrt{\sum_{m\in S_{l,\tau}^{(n,k)}}p(m|S_{l,\tau}^{(n,k)},\widehat{\theta}_{l,\tau-1})\|\tilde{z}_{m,l,\tau}\|^2_{H_{l,\tau}^{-1}(\widehat{\theta}_{l,\tau})}}\right],\cr \label{eq:instance_r_bd_max_z_ad}
\end{align}
where the second inequality comes from the fact that $(S_1^*,\dots,S_K^*)\in\mathcal{M}_{\tau-1}$ and $\max_{(S_1,...,S_K)\in \mathcal{M}_{\tau-1}}\sum_{l\in[K]}R^{LCB}_{l,\tau}(S_l)\le \sum_{l\in[K]}R^{UCB}_{l,\tau}(S_{l,\tau}^{(n,k)})$ from the algorithm.

Likewise, we also have
\begin{align}
    &\sum_{l\in[K]}{R}_l(S_l^*)-\sum_{l\in[K]}R_l(S_{l,\tau}^{(J,k)})\cr &\lesssim  \sum_{l\in[K]}\left[\zeta_\tau^2\max_{m\in S_l^*}\|z_m\|_{H_{l,\tau}^{-1}(\widehat{\theta}_{l,\tau})}^2+ \zeta_\tau^2\max_{m\in S_l^*}\|\tilde{z}_{m,l,\tau}\|_{H_{l,\tau}^{-1}(\widehat{\theta}_{l,\tau})}^2+\zeta_\tau^2\max_{m\in S_{l,\tau}^{(J,k)}}\|z_m\|_{H_{l,\tau}^{-1}(\widehat{\theta}_{l,\tau})}^2\right.\cr &\left.\qquad+\zeta_\tau^2\max_{m\in S_{l,\tau}^{(J,k)}}\|\tilde{z}_{m,l,\tau}\|_{H_{l,\tau}^{-1}(\widehat{\theta}_{l,\tau})} ^2 +\zeta_\tau\sqrt{\sum_{m\in S_l^*}p(m|S_l^*,\widehat{\theta}_{l,\tau-1})}\sqrt{\sum_{m\in S_l^*}p(m|S_l^*,\widehat{\theta}_{l,\tau-1})\|\tilde{z}_{m,l,\tau}\|_{H_{l,\tau}^{-1}(\widehat{\theta}_{l,\tau})}^2}\right.\cr &\left.\qquad+\zeta_\tau\sqrt{\sum_{m\in S_{l,\tau}^{(n,k)}}p(m|S_{l,\tau}^{(J,k)},\widehat{\theta}_{l,\tau-1})}\sqrt{\sum_{m\in S_{l,\tau}^{(J,k)}}p(m|S_{l,\tau}^{(J,k)},\widehat{\theta}_{l,\tau-1})\|\tilde{z}_{m,l,\tau}\|^2_{H_{l,\tau}^{-1}(\widehat{\theta}_{l,\tau})}}\right].\cr\label{eq:instance_r_bd_max_z_ad2}
\end{align}

We can show that 
\begin{align}
      &H_{k,\tau}(\widehat{\theta}_{k,\tau})\cr & = \lambda I_r + \sum_{t\in \Tcal_{k,\tau-1}}\sum_{n\in S_{k,t}}p(n|S_{k,t}, \widehat{\theta}_{k,\tau})z_nz_n^\top-\sum_{t\in \Tcal_{k,\tau-1}}\sum_{n\in S_{k,t}}\sum_{m\in S_{k,t}} p(n|S_{k,t}, \widehat{\theta}_{k,\tau})p(m|S_{k,t}, \widehat{\theta}_{k,\tau})z_nz_m^\top \cr
    & = \lambda I_r + \sum_{t\in \Tcal_{k,\tau-1}}\sum_{n\in S_{k,t}}p(n|S_{k,t}, \widehat{\theta}_{k,\tau})z_nz_n^\top-\frac{1}{2}\sum_{t\in \Tcal_{k,\tau-1}}\sum_{n\in S_{k,t}}\sum_{m\in S_{k,t}} p(n|S_{k,t}, \widehat{\theta}_{k,\tau})p(m|S_{k,t}, \widehat{\theta}_{k,\tau})(z_nz_m^\top+z_nz_m^\top)\cr 
     & \succeq \lambda I_r + \sum_{t\in \Tcal_{k,\tau-1}}\sum_{n\in S_{k,t}}p(n|S_{k,t}, \widehat{\theta}_{k,\tau})z_nz_n^\top-\frac{1}{2}\sum_{t\in \Tcal_{k,\tau-1}}\sum_{n\in S_{k,t}}\sum_{m\in S_{k,t}} p(n|S_{k,t}, \widehat{\theta}_{k,\tau})p(m|S_{k,t}, \widehat{\theta}_{k,\tau})(z_nz_n^\top+z_mz_m^\top)\cr 
     & =\lambda I_r + \sum_{t\in \Tcal_{k,\tau-1}}\sum_{n\in S_{k,t}}p(n|S_{k,t}, \widehat{\theta}_{k,\tau})z_nz_n^\top-\sum_{t\in \Tcal_{k,\tau-1}}\sum_{n\in S_{k,t}}\sum_{m\in S_{k,t}} p(n|S_{k,t}, \widehat{\theta}_{k,\tau})p(m|S_{k,t}, \widehat{\theta}_{k,\tau})z_nz_n^\top\cr
       & =\lambda I_r + \sum_{t\in \Tcal_{k,\tau-1}}\sum_{n\in S_{k,t}}p(n|S_{k,t}, \widehat{\theta}_{k,\tau})\left(1-\sum_{m\in S_{k,t}}p(m|S_{k,t},\widehat{\theta}_{k,\tau})\right)z_nz_n^\top\cr  &= \lambda I_r + \sum_{t\in \Tcal_{k,\tau-1}}\sum_{n\in S_{k,t}}p(n|S_{k,s}, \widehat{\theta}_{k,s})p(n_0|S_{k,s},\widehat{\theta}_{k,})z_nz_n^\top \succeq \lambda I_r +\sum_{t\in \Tcal_{k,\tau-1}}\sum_{n\in S_{k,t}}\kappa z_nz_n^\top\cr
     &\succeq \lambda I_r +\sum_{n\in \mathcal{N}_{k,\tau-1}}\kappa r\pi_{k,\tau-1}(n)T_{\tau-1} z_{n}z_{n}^\top= \kappa T_{\tau-1} r(V(\pi_{k,\tau-1})+(\lambda /\kappa r T_{\tau-1})I_r) \cr &\succeq \kappa T_{\tau-1} r(V(\pi_{k,\tau-1})+(\lambda / r T_{\tau-1})I_r).\label{eq:V_lower1_ad}
\end{align}

From Lemma~\ref{lem:D_opt} and  \eqref{eq:V_lower1_ad}, we also have, for any $n\in \Ncal_{k,\tau}$
\begin{align}
    \|z_n\|_{H_{k,\tau}^{-1}(\widehat{\theta}_{k,\tau})}^2&=\Ocal\left( \frac{\|z_{n}\|_{({V}({\pi}_{k,\tau-1})+(\lambda / r T_{\tau-1})I_r)^{-1}}^2}{\kappa rT_{\tau-1} }\right) \cr &=\Ocal\left(\frac{1}{
    \kappa T_{\tau-1}}\right).\label{eq:sqrt_z_norm_bd_ad1}
\end{align}

 We have
\begin{align}
    &H_{k,\tau}(\widehat{\theta}_{k,\tau})\cr & = \lambda I_r + \sum_{t\in \Tcal_{k,\tau-1}}\sum_{n\in S_{k,t}}p(n|S_{k,t}, \widehat{\theta}_{k,\tau})z_nz_n^\top-\sum_{t\in \Tcal_{k,\tau-1}}\sum_{n\in S_{k,t}}\sum_{m\in S_{k,t}} p(n|S_{k,t}, \widehat{\theta}_{k,\tau})p(m|S_{k,t}, \widehat{\theta}_{k,\tau})z_nz_m^\top \cr
     & =\lambda I_r + \sum_{t\in \Tcal_{k,\tau-1}}\EE_{\widehat{\theta}_{k,\tau}}[z_nz_n^\top]-\EE_{\widehat{\theta}_{k,\tau}}[z_n]\EE_{\widehat{\theta}_{k,\tau}}[z_n]^\top\cr 
     & =\lambda I_r + \sum_{t\in \Tcal_{k,\tau-1}}\EE_{\widehat{\theta}_{k,\tau}}[\tilde{z}_{n,k,\tau}\tilde{z}_{n,k,\tau}^\top]\cr 
     & =\lambda I_r + \sum_{t\in \Tcal_{k,\tau-1}}\sum_{n\in S_{k,t}}p(n|S_{k,t}, \widehat{\theta}_{k,\tau})\tilde{z}_{n,k,\tau}\tilde{z}_{n,k,\tau}^\top\cr  
     & \succeq \lambda I_r + \sum_{J\in \Jcal(\Ncal_{k,\tau-1})}\sum_{t\in \Tcal_{J,k,\tau-1}}\sum_{n\in J}p(n|J, \widehat{\theta}_{k,\tau})\tilde{z}_{n,k,\tau}\tilde{z}_{n,k,\tau}^\top
     \cr
     &\succeq \lambda I_r +\sum_{J\in \Jcal(\Ncal_{k,\tau-1})} r\bar{\pi}_{k,\tau-1}(J)T_{\tau-1} \sum_{n\in J}\kappa\tilde{z}_{n,k,\tau}\tilde{z}_{n,k,\tau}^\top \cr 
     &\succeq  \kappa T_{\tau-1} r \left(\bar{V}(\bar{\pi}_{k,\tau-1})+(\lambda/T_{\tau-1}r)I_r\right).\label{eq:V_lower2_ad}
\end{align}
 
From Lemma~\ref{lem:D_opt2} and  \eqref{eq:V_lower2_ad} with $\Ncal_{k,\tau}\subseteq \Ncal_{k,\tau-1}$, we also have, for any $n\in J\in \Jcal(\Ncal_{k,\tau})$
\begin{align}
    \|\tilde{z}_{n,k,\tau}(J)\|_{H_{k,\tau}^{-1}(\widehat{\theta}_{k,\tau})}^2&=\Ocal\left( \frac{\|\tilde{z}_{n,k,\tau}(J))\|_{(\bar{V}(\bar{\pi}_{k,\tau-1})+(\lambda / r T_{\tau-1})I_r)^{-1}}^2}{\kappa rT_{\tau-1} }\right) \cr &=\Ocal\left(\frac{1}{
    \kappa T_{\tau-1}}\right).\label{eq:sqrt_z_norm_bd_ad2}
\end{align}

 We have
\begin{align*}
    &H_{k,\tau}(\widehat{\theta}_{k,\tau})\cr & = \lambda I_r + \sum_{t\in \Tcal_{k,\tau-1}}\sum_{n\in S_{k,t}}p(n|S_{k,t}, \widehat{\theta}_{k,\tau})z_nz_n^\top-\sum_{t\in \Tcal_{k,\tau-1}}\sum_{n\in S_{k,t}}\sum_{m\in S_{k,t}} p(n|S_{k,t}, \widehat{\theta}_{k,\tau})p(m|S_{k,t}, \widehat{\theta}_{k,\tau})z_nz_m^\top \cr
     & =\lambda I_r + \sum_{t\in \Tcal_{k,\tau-1}}\EE_{\widehat{\theta}_{k,\tau}}[z_nz_n^\top]-\EE_{\widehat{\theta}_{k,\tau}}[z_n]\EE_{\widehat{\theta}_{k,\tau}}[z_n]^\top\cr 
     & =\lambda I_r + \sum_{t\in \Tcal_{k,\tau-1}}\EE_{\widehat{\theta}_{k,\tau}}[\tilde{z}_{n,k,\tau}\tilde{z}_{n,k,\tau}^\top]\cr 
     & =\lambda I_r + \sum_{t\in \Tcal_{k,\tau-1}}\sum_{n\in S_{k,t}}p(n|S_{k,t}, \widehat{\theta}_{k,\tau})\tilde{z}_{n,k,\tau}\tilde{z}_{n,k,\tau}^\top\cr  
     & \succeq \lambda I_r + \sum_{J\in \Jcal(\Ncal_{k,\tau-1})}\sum_{t\in \Tcal_{J,k,\tau-1}}\sum_{n\in J}p(n|J, \widehat{\theta}_{k,\tau})\tilde{z}_{n,k,\tau}\tilde{z}_{n,k,\tau}^\top
     \end{align*}
    \begin{align}
     &\succeq \lambda I_r +\sum_{J\in \Jcal(\Ncal_{k,\tau-1})} r\tilde{\pi}_{k,\tau-1}(J)T_{\tau-1} \sum_{n\in J}p(n|J, \widehat{\theta}_{k,\tau})\tilde{z}_{n,k,\tau}\tilde{z}_{n,k,\tau}^\top \cr 
        &\succeq \lambda I_r +\sum_{J\in \Jcal(\Ncal_{k,\tau-1})} r\tilde{\pi}_{k,\tau-1}(J)T_{\tau-1} \sum_{n\in J}p(n|J, \widehat{\theta}_{k,\tau-1})\tilde{z}_{n,k,\tau}\tilde{z}_{n,k,\tau}^\top \cr &\qquad\qquad-2\zeta_\tau\sum_{J\in \Jcal(\Ncal_{k,\tau-1})}r\tilde{\pi}_{k,\tau-1}(J)T_{\tau-1}\max_{n\in J} (\|z_n\|_{H_{k,\tau}^{-1}(\widehat{\theta}_{k,\tau})}+\|z_n\|_{H_{k,\tau-1}^{-1}(\widehat{\theta}_{k,\tau-1})})\tilde{z}_{n,k,\tau}\tilde{z}_{n,k,\tau}^\top\cr
     &=  T_{\tau-1} r \left(\tilde{V}(\tilde{\pi}_{k,\tau-1})+(\lambda/T_{\tau-1}r)I_r\right.\cr &\qquad \left.-2\zeta_\tau\sum_{J\in \Jcal(\Ncal_{k,\tau-1})}\tilde{\pi}_{k,\tau-1}(J)\max_{n\in J} (\|z_n\|_{H_{k,\tau}^{-1}(\widehat{\theta}_{k,\tau})}+\|z_n\|_{H_{k,\tau-1}^{-1}(\widehat{\theta}_{k,\tau-1})})\tilde{z}_{n,k,\tau}\tilde{z}_{n,k,\tau}^\top\right),\cr\label{eq:V_lower3_ad}
\end{align}
 
where the last inequality is obtained from, using the mean value theorem,
\begin{align}
    &\sum_{n\in J}(p(n|J, \widehat{\theta}_{k,\tau})-p(n|J, \widehat{\theta}_{k,\tau-1})\tilde{z}_{n,k,\tau}\tilde{z}_{n,k,\tau}^\top\cr &= \sum_{n\in J}(p(n|J, \widehat{\theta}_{k,\tau})-p(n|J, {\theta}_{k}^*)+p(n|J, {\theta}_{k}^*)-p(n|J, \widehat{\theta}_{k,\tau-1}))\tilde{z}_{n,k,\tau}\tilde{z}_{n,k,\tau}^\top\cr &\succeq -2\zeta_\tau(\max_{n\in J}\|z_n\|_{H_{k,\tau}^{-1}(\widehat{\theta}_{k,\tau})}+\max_{n\in J}\|z_n\|_{H_{k,\tau-1}^{-1}(\widehat{\theta}_{k,\tau-1})})\tilde{z}_{n,k,\tau}\tilde{z}_{n,k,\tau}^\top.
\end{align}

Let $B=2\zeta_\tau\sum_{J\in \Jcal(\Ncal_{k,\tau-1})}\tilde{\pi}_{k,\tau-1}(J)\max_{n\in J} (\|z_n\|_{H_{k,\tau}^{-1}(\widehat{\theta}_{k,\tau})}+\|z_n\|_{H_{k,\tau-1}^{-1}(\widehat{\theta}_{k,\tau-1})})\tilde{z}_{n,k,\tau}\tilde{z}_{n,k,\tau}^\top$ and we have
$B\preceq 4\zeta_\tau \sqrt{\frac{1}{\kappa T_{\tau-2}}}\sum_{J\in \Jcal(\mathcal{N}_{k,\tau-1})}\tilde{\pi}_{k,\tau-1}(J)\max_{n\in J} \tilde{z}_{n,k,\tau}\tilde{z}_{n,k,\tau}^\top$ from \eqref{eq:sqrt_z_norm_bd_ad1}. Then for $\tau\ge 3$, we have 
\begin{align}
    &\tilde{V}(\tilde{\pi}_{k,\tau-1})-B\cr &\succeq \frac{1}{2}\tilde{V}(\tilde{\pi}_{k,\tau-1})+\frac{1}{2}\tilde{V}(\tilde{\pi}_{k,\tau-1})-B\cr &\succeq \frac{1}{2}\tilde{V}(\tilde{\pi}_{k,\tau-1})+\frac{1}{2}\sum_{J\in\Jcal(\Ncal_{k,\tau})}\tilde{\pi}_{k,\tau}(J)\sum_{n\in J}\kappa \tilde{z}_{n,k,\tau}\tilde{z}_{n,k,\tau}^\top-4\zeta_\tau\sqrt{\frac{1}{\kappa T_{\tau-2}}}\sum_{J\in \Jcal(\Ncal_{k,\tau-1})}\tilde{\pi}_{k,\tau-1}(J)\max_{n\in J}\tilde{z}_{n,k,\tau}\tilde{z}_{n,k,\tau}^\top\cr&\succeq \frac{1}{2}\tilde{V}(\tilde{\pi}_{k,\tau-1}),
\end{align}
where the last inequality is obtained from  $\frac{1}{2}\kappa\ge 4\zeta_\tau\sqrt{\frac{1}{\kappa T_{\tau-2}}}$ because $T_{\tau-2}\ge \min\{T_1,\eta_T\}$ with large enough $T$ such that $T\ge \max\{\frac{r^3K}{\kappa^6}\log^4(KTL),\exp(\frac{r}{\kappa^3})\}$.

Then, we have
\begin{align*}
   \|\tilde{z}_{n,k,\tau}\|_{H_{k,\tau}^{-1}(\widehat{\theta}_{k,\tau})}^2&\le rT_{\tau-1}\|\tilde{z}_{n,k,\tau}\|_{(\tilde{V}(\tilde{\pi}_{k,\tau-1})+(\lambda/T_{\tau-1}r)I_r-B)^{-1}}^2 \cr 
    &\le rT_{\tau-1}\|\tilde{z}_{n,k,\tau}\|_{(\frac{1}{2}\tilde{V}(\tilde{\pi}_{k,\tau-1})+\frac{1}{2}(\lambda/T_{\tau-1}r)I_r)^{-1}}^2\cr 
    &\le 2rT_{\tau-1}\|\tilde{z}_{n,k,\tau}\|_{(\tilde{V}(\tilde{\pi}_{k,\tau-1})+(\lambda/T_{\tau-1}r)I_r)^{-1}}^2.
\end{align*}
 Then from the above, Lemma~\ref{lem:D_opt2}, and  \eqref{eq:V_lower3_ad} with $\mathcal{N}_{k,\tau}\subseteq\mathcal{N}_{k,\tau-1}$, we have, for any $J\in \mathcal{J}(\mathcal{N}_{k,\tau})$
\begin{align}
    &\sum_{n\in J}p(n|J, \widehat{\theta}_{k,\tau-1})\|\tilde{z}_{n,k,\tau}\|_{H_{k,\tau}^{-1}(\widehat{\theta}_{k,\tau})}^2\cr &=\Ocal\left( \frac{ \sum_{n\in J}p(n|J, \widehat{\theta}_{k,\tau-1})\|\tilde{z}_{n,k,\tau}\|_{(\tilde{V}(\tilde{\pi}_{k,\tau-1})+(\lambda/T_{\tau-1}r)I_r)^{-1}}^2}{rT_{\tau-1} }\right) \cr &=\Ocal\left(\frac{1}{T_{\tau-1}}\right).\label{eq:sqrt_z_norm_bd_ad_tilde}
\end{align}

Therefore under $E$,  from \eqref{eq:instance_r_bd_max_z_ad}, \eqref{eq:instance_r_bd_max_z_ad2}, \eqref{eq:sqrt_z_norm_bd_ad1},\eqref{eq:sqrt_z_norm_bd_ad2}, and  \eqref{eq:sqrt_z_norm_bd_ad_tilde}, we have the following.

 For $t\in\bigcup_{n\in \Ncal_{k,\tau}, k\in[K]}\Tcal_{n,k,\tau}\bigcup_{J\in \Jcal(\Ncal_{k,\tau}), k\in[K]}\Tcal_{J,k,\tau}\bigcup_{n\in J\in \Jcal(\Ncal_{k,\tau}), k\in[K]}\Tcal_{n,J,k,\tau}$, 

\[\sum_{k\in[K]}(R_k(S_k^*)-R_k(S_{k,t}))=\Ocal\left(K\left(\sqrt{\frac{r}{T_{\tau-1}}}+\frac{r}{{T_{\tau-1}}\kappa}\right)\right).\] 

For the regret bound, we have
\begin{align}
&\mathbb{E}\left[\sum_{t\in[T]}\sum_{k\in[K]}R_k(S_{k}^*)-R_t(S_{k,t})\right]
\cr &\le \mathbb{E}\left[ \sum_{t\in[T]}\sum_{k\in[K]}\left({R}_k(S_{k}^*)-R_k(S_{k,t})\right)\mathbf{1}(E)\right]+\mathbb{E}\left[\sum_{t\in[T]}\sum_{k\in[K]}\left({R}_k(S_{k}^*)-R_k(S_{k,t})\right)\mathbf{1}(E^c)\right]\cr
&= \tilde{\Ocal}\left(K\sum_{\tau=3}^{\tau_T}\sum_{k\in[K]}\left(\sum_{J\in\Jcal(\mathcal{N}_{k,\tau})}|\mathcal{T}_{J,k,\tau}|+\sum_{n\in\mathcal{N}_{k,\tau}}|\mathcal{T}_{n,k,\tau}|\right)\left(\sqrt{\frac{r}{T_{\tau-1}}}+\frac{r}{{T_{\tau-1}}\kappa}\right)\right)+\tilde{\Ocal}(rK\eta_T)+\Ocal(K) 
\cr
&= \tilde{\Ocal}\left( K\sum_{\tau=3}^{\tau_T}\sum_{k\in[K]}\left(\sum_{J\in\Jcal(\mathcal{N}_{k,\tau})}|\mathcal{T}_{J,k,\tau}|+\sum_{n\in\mathcal{N}_{k,\tau}}|\mathcal{T}_{n,k,\tau}|\right)\left(\sqrt{\frac{r}{T_{\tau-1}}}+\frac{r}{{T_{\tau-1}}\kappa}\right)\right)+\tilde{\Ocal}(rK\eta_T)
\cr
&=\tilde{\Ocal}\left(K\sum_{\tau=3}^{\tau_T}\sum_{k\in[K]}(rT_\tau+|Supp(\pi_{k,\tau})|+|Supp(\tilde{\pi}_{k,\tau})|)\left(\sqrt{\frac{r}{T_{\tau-1}}}+\frac{r}{{T_{\tau-1}}\kappa}\right)\right)+\tilde{\Ocal}(rK\eta_T) \cr
&=\tilde{\Ocal}\left(K^2\sum_{\tau=3}^{\tau_T}\left(r^{3/2}\eta_T+r^2\frac{1}{\kappa\sqrt{T_{\tau-1}}}\eta_T\right)\right)\cr
&=\tilde{\Ocal}\left(K^2r^{3/2}\eta_T\right)\cr
&=\tilde{\Ocal}\left( r^{3/2}K^2(T/rK)^{\frac{1}{2(1-2^{-M})}}\right),\label{eq:R_main_bd2}
\end{align}
where the third last equality comes from Lemma~\ref{lem:D_opt} and the second last equality comes from $\tau_T\le M=\tilde{\mathcal{O}}(1)$ and $T_{\tau-1}\ge \eta_T$ for $\tau\ge 3$.

\subsection{Approximation Oracle}\label{app:oracle}

Here we discuss the combinatorial optimization in our algorithm. 
We can utilize an $\alpha$-approximation oracle with $0\le \alpha\le 1$, first introduced in \citet{kakade2007playing}.
Instead of obtaining the exact optimal assortment solution, the $\alpha$-approximation oracle, denoted by $\mathbb{O}^\alpha$, outputs $\{S_k^\alpha\}_{k\in[K]}$ satisfying $\sum_{k\in[K]}f_k(S_k^\alpha)\ge  \max_{\{S_k\}_{k\in[K]}\in \mathcal{M}}\sum_{k\in[K]}\alpha f_k(S_k).$ 

{We introduce an algorithm (Algorithm~\ref{alg:elim_alpha} in Appendix~\ref{app:oracle}) by modifying Algorithm~\ref{alg:elim} to incorporate  $\alpha$-approximation oracles for the optimization. Due to the redundancy, we explain only the distinct parts of the algorithm here. (Approximation oracles can also be applied to Algorithm~\ref{alg:elim2} similarly, but we omit it in this discussion.) For testing the assignment $(n,k)$, the algorithm constructs assortment  $\{S_{l,\tau}^{\alpha,(n,k)}\}_{l\in[K]}$ (where $n\in S_{k,\tau}^{\alpha,(n,k)}$) in an optimistic view with an $\alpha$-approximation oracle to resolve computation issue as follows. We define an approximation oracle $\mathbb{O}^{\alpha, (n,k)}_{UCB}$ which outputs $\{S_{l,\tau}^{\alpha,(n,k)}\}_{l\in[K]}$ satisfying  
\begin{align}
\max_{\{S_l\}_{l\in[K]}\in\mathcal{M}_{\tau-1}: n\in S_k} \sum_{l\in[K]}\alpha {R}^{UCB}_{l,\tau}(S_l) &\le \sum_{l\in[K]}{R}^{UCB}_{l,\tau}(S_{l,\tau}^{\alpha,(n,k)}) ,\label{eq:construct_Snk_alpha}
\end{align} which replaces Line~\ref{line:construct} in Algorithm~\ref{alg:elim}.
For the elimination procedure, we define another $\beta$-approximation oracle, denoted by $\mathbb{O}_{LCB}^{\beta}$, which outputs $\{S_{l,\tau}^\beta\}_{l\in[K]}$ 
satisfying
\begin{align} \max_{\{S_{l}\}_{l\in[K]}\in \mathcal{M}_{\tau-1}}\sum_{l\in[K]}\beta R^{LCB}_{l,\tau}(S_l)\le \sum_{l\in[K]}R^{LCB}_{l,\tau}(S_{l,\tau}^\beta).\label{eq:oracle_elim_alpha}\end{align}
Then it updates $\mathcal{N}_{k,\tau}$ 
by eliminating $n\in \mathcal{N}_{k,\tau-1}$
which satisfies the elimination  condition of \begin{align}\sum_{l\in[K]}\alpha R^{LCB}_{l,\tau}(S_{l,\tau}^\beta)> \sum_{l\in[K]}R^{UCB}_{l,\tau}(S_{l,\tau}^{\alpha,(n,k)}),\label{eq:elim_con_alpha}\end{align}
which replaces Line~\ref{line:elim1_2} in Algorithm~\ref{alg:elim}.
We note that the algorithm utilizes the two different types of approximation oracles, $\mathbb{O}^{\alpha,(n,k)}_{UCB}$ and $\mathbb{O}^{\beta}_{LCB}$. Then the algorithm achieves a regret bound for $\gamma$-regret defined as $\mathcal{R}^\gamma(T)=\mathbb{E}[\sum_{t\in[T]}\sum_{k\in[K]}\gamma R_k(S_k^*)-R_k(S_{k,t})]$ in the following theorem.}
\begin{theorem}
Algorithm~\ref{alg:elim_alpha} with $M=\mathcal{O}(\log(T))$ achieves a regret bound with $\gamma=\alpha\beta$  as 
    \[\mathcal{R}^\gamma(T)=\tilde{\Ocal}\left(\tfrac{1}{\kappa}K^{\frac{3}{2}}\sqrt{rT}\left(\frac{T
}{rK}\right)^{\frac
{1}{2(2^M-1)}}\right).\]\label{thm:alph_oracle}
\end{theorem}
\begin{proof}
    The proof is provided in Appendix~\ref{app:oracle_proof}.
\end{proof}

 \subsubsection{$\alpha$-Approximated Algorithm (Algorithm~\ref{alg:elim_alpha})}

\begin{algorithm}[ht]
  \caption{Batched Stochastic Matching Bandit with $\beta$-Approximation Oracle }\label{alg:elim_alpha}
  \KwIn{$\beta$, $\kappa$, $M\ge 1$;   \textbf{Init:} $t\leftarrow 1, T_1\leftarrow \eta_T$}

  Compute SVD of $X=U\Sigma V^\top$ and obtain $U_r=[u_1,\dots, u_r]$; Construct $z_{n}\leftarrow U_{r}^\top  x_n$ for $n\in [N]$
  
 \For{$\tau=1,2...$}{
 \For{$k\in[K]$}{

\tcp{Estimation}

$\widehat{\theta}_{k,\tau}\leftarrow\argmin_{\theta\in\mathbb{R}^{r}} l_{k,\tau}(\theta)$ with  \eqref{eq:log-loss} where $\mathcal{T}_{k,\tau-1}=  \mathcal{T}_{k,\tau-1}^{(1)}\cup\mathcal{T}_{k,\tau-1}^{(2)}$ and $\mathcal{T}_{k,\tau-1}^{(2)}=\bigcup_{n\in\mathcal{N}_{k,\tau-1}}\mathcal{T}_{n,k,\tau-1}^{(2)}$

\tcp{Assortments Construction }

$\{S_{l,\tau}^{\alpha,(n,k)}\}_{l\in[K]}\leftarrow \mathbb{O}_{UCB}^{\alpha, (n,k)}$ from \eqref{eq:construct_Snk_alpha} for all $n\in\mathcal{N}_{k,\tau-1}$ with  \eqref{eq:ucb_lcb}

\tcp{Elimination}

$\{S_{l,\tau}^\beta\}_{l\in [K]}\leftarrow \mathbb{O}^{\beta}_{LCB}$ from \eqref{eq:oracle_elim_alpha}

 $\mathcal{N}_{k,\tau}\leftarrow\{n\in\mathcal{N}_{k,\tau}:\sum_{l\in[K]}\alpha R^{LCB}_{l,\tau}(S_{l,\tau}^\beta)\le \sum_{l\in[K]}R^{UCB}_{l,\tau}(S_{l,\tau}^{\alpha,(n,k)}) \}$ for $k\in[K]$

\tcp{G/D-optimal design}

 $\pi_{k,\tau}\leftarrow \argmax_{\pi\in \mathcal{P}(\mathcal{N}_{k,\tau})} \log\det(\sum_{n\in\mathcal{N}_{k,\tau}}\pi_{k,\tau}(n) z_{n}z_{n}^\top+(1/rT_\tau)I_r)$ 

\tcp{Exploration}
  Run \texttt{Warm-up} (Algorithm~\ref{alg:warm}) over time steps in $\mathcal{T}_{k,\tau}^{(1)}$ (defined in Algorithm~\ref{alg:warm})

 \For{$n\in\mathcal{N}_{k,\tau}$}
 {$t_{n,k}\leftarrow t$, $\mathcal{T}_{n,k,\tau}^{(2)}\leftarrow [t_{n,k},t_{n,k}+\lceil r\pi_{k,\tau}(n) T_\tau\rceil-1]$

\While{$t\in \mathcal{T}_{n,k,\tau}^{(2)}$}{
Offer $\{S_{l,t}\}_{l\in[K]}=\{S_{l,\tau}^{(n,k)}\}_{l\in[K]}$ and observe  feedback $y_{m,k,t}\in\{0,1\}$ for all $m\in S_{l,t}$ and $l\in[K]$

$t\leftarrow t+1$}}}
  $\mathcal{M}_\tau\leftarrow \{\{S_k\}_{k\in[K]}: S_k\subset \mathcal{N}_{k,\tau}, |S_k|\le L\: \forall k\in[K], S_k\cap S_l=\emptyset \: \forall k\neq l\}$; $T_{\tau+1}\leftarrow \eta_{T}\sqrt{T_\tau}$
 }
\end{algorithm}

 \subsubsection{Proof of Theorem~\ref{thm:alph_oracle}}\label{app:oracle_proof}
 In this proof, we provide only the parts that are different from the proof of Theorem~\ref{thm:elim}.

\begin{lemma} Under $E$,
$(S_1^*,\dots,S_K^*)\in\mathcal{M}_{\tau-1}$ for all $\tau\in[T]$.\label{lem:S*inM_alpha}
\end{lemma}
\begin{proof} Here we use induction for the proof.
    Suppose that for fixed $\tau$, we have $(S_1^*,\dots,S_K^*)\in\mathcal{M}_{\tau}$ for all $k\in[K]$. Recall that $\beta_T=(C_1/\kappa)\sqrt{\log(TKN) }$. From Lemma~\ref{lem:R_UCBR_bd}, we have $R_{k,\tau+1}^{UCB}(S)\ge R_{k}(S)$ and $R_{k,\tau+1}^{LCB}(S)\le R_{k}(S)$ for any $S\subset[N]$.
    Then for $k\in[K]$, $n\in S_k^*$, and any $(S_1,..,S_K)\in\mathcal{M}_\tau$,   we have
    \begin{align}
\sum_{l\in[K]}R^{UCB}_{l,\tau+1}(S_{l,\tau+1}^{{\alpha},(n,k)})&\ge\max_{\{S_k\}_{k\in[K]}\in\mathcal{M}_{\tau}: n\in S_k} \sum_{l\in[K]}{\alpha} {R}^{UCB}_{l,\tau+1}(S_l) \cr &\ge \sum_{l\in[K]}{\alpha} R^{UCB}_{l,\tau+1}(S_{l}^*)\cr &\ge \sum_{l\in[K]}{\alpha} R_{l}(S_{l}^*)\cr &\ge  \sum_{l\in[K]}{\alpha} R_{l}(S_{l,\tau+1}^{{\beta}})\cr &\ge\sum_{l\in[K]}{\alpha} R^{LCB}_{l,\tau+1}(S_{l,\tau+1}^{{\beta}}),\label{eq:R_ucb_R_lcb_alpha}
    \end{align}
 where the first inequality comes from \eqref{eq:construct_Snk_alpha}, the second inequality comes from $(S_1^*,\dots S_K^*)\in \mathcal{M}_\tau$, and the firth inequality comes from the optimality of $(S_1^*,\dots,S_K^*)$. This implies that $n\in \mathcal{N}_{k,\tau+1}$ from the algorithm. Then by following the same statement of \eqref{eq:R_ucb_R_lcb_alpha} for all $n\in S_k^*$ and $k\in[K]$, we have $S_k^*\subset\mathcal{N}_{k,\tau+1}$ for all $k\in[K]$, which implies $(S_1^*,\dots,S_K^*)\in\mathcal{M}_{\tau+1}$. Therefore, with $(S_1^*,\dots,S_K^*)\in\mathcal{M}_{1}$, we can conclude the proof from the induction.
\end{proof}

From Lemmas~\ref{lem:S*inM_alpha} and \ref{lem:R_UCBR_bd}, under $E$, we have
\begin{align}
    \sum_{l\in[K]}\alpha\beta{R}_l(S_l^*)-\sum_{l\in[K]}R_l(S_{l,\tau}^{{\alpha},(n,k)})&\le \sum_{l\in[K]}\alpha\beta R^{LCB}_{l,\tau}(S_l^*)+4\beta_T \max_{m\in S_l^*}\|z_{m}\|_{V_{l,\tau}^{-1}}\cr & \quad-\sum_{l\in[K]}R^{UCB}_{l,\tau}(S_{l,\tau}^{{\alpha},(n,k)})+4\beta_T \max_{m\in S_{l,\tau}^{(n,k)}}\|z_{m}\|_{V_{l,\tau}^{-1}}
    \cr  &\le \sum_{l\in[K]}
    {\alpha} R^{LCB}_{l,\tau}(S_{l,\tau}^{\beta} )+4\beta_T\max_{m\in S_l^*}\|z_{m}\|_{V_{l,\tau}^{-1}}\cr & \quad-\sum_{l\in[K]}R^{UCB}_{l,\tau}(S_{l,\tau}^{{\alpha},(n,k)})+4\beta_T \max_{m\in S_{l,\tau}^{(n,k)}}\|z_{m}\|_{V_{l,\tau}^{-1}}
    \cr  &\le 4\beta_T \sum_{l\in[K]}(\max_{m\in S_l^*}\|z_{m}\|_{V_{l,\tau}^{-1}}+\max_{m\in S_{l,\tau}^{(n,k)}}\|z_{m}\|_{V_{l,\tau-1}^{-1}}),\cr\label{eq:instance_r_bd_max_z_alpha}
\end{align}
where the second inequality comes from \eqref{eq:oracle_elim_alpha} and last inequality comes from the fact that $(S_1^*,\dots,S_K^*)\in\mathcal{M}_{\tau-1}$ and $\sum_{l\in[K]}{\alpha} R^{LCB}_{l,\tau}(S_{l,\tau}^{\beta})\le \sum_{l\in[K]}R^{UCB}_{l,\tau}(S_{l,\tau}^{{\alpha},(n,k)})$ from the algorithm. Then,  by following the proof in Theorem~\ref{alg:elim}, we can conclude the proof.

\subsection{Proof of Lemmas}

\subsubsection{Proof of Lemma~\ref{lem:z_theta_hat_z_gap_bd}}\label{app:z_theta_hat_theta_gap}
For the poof, we follow the proof steps in (Bounding the Prediction Error) \citet{oh2021multinomial}.   We define \[H_{k,\tau}({\theta})=\sum_{t\in \mathcal{T}_{k,\tau}}\left(\sum_{n\in S_{k,t}}p(n|S_{k,t},{\theta})z_{n}z_{n}^\top-\sum_{n\in S_{k,t}}\sum_{m\in S_{k,t}}p(n|S_{k,t},{\theta})p(m|S_{k,t},{\theta})z_{n}z_{m}^\top\right)+I_r.\] We note that $g_{k,\tau}(\theta_1)-g_{k,\tau}(\theta_2)=\sum_{t\in\mathcal{T}_{k,\tau}}\sum_{n\in S_{k,t}}(p(n,|S_{k,t},\theta_1)-p(n,|S_{k,t},\theta_2))z_{n}+(\theta_1-\theta_2)$. 
    Then from the mean value theorem, there exists $\bar{\theta}=c\theta_1 +(1-c)\theta_2$ with some $c\in(0,1)$ such that 
\begin{align}
    &g_{k,\tau}(\theta_1)-g_{k,\tau}(\theta_2)\cr&=\nabla_\theta g_{k,\tau}(\theta)\big|_{\theta=\bar{\theta}}(\theta_1-\theta_2)\cr &=\left(\sum_{t\in\mathcal{T}_{k,\tau}}\left(\sum_{n\in S_{k,t}}p(n|S_{k,t},\bar{\theta})z_{n}z_{n}^\top-\sum_{n\in S_{k,t}}\sum_{m\in S_{k,t}}p(n|S_{k,t},\bar{\theta})p(m|S_{k,t},\bar{\theta})z_{n}z_{m}^\top\right)+I_r\right)(\theta_1-\theta_2)\cr &=H_{k,\tau}(\bar{\theta})(\theta_1-\theta_2)\label{eq:mean_value2}\end{align}

We define 
$L_{k,\tau}=H_{k,\tau}(\theta^*_k)$ and $E_{k,\tau}=H_{k,\tau}(\bar{\theta}_k)-H_{k,\tau}(\theta_k^*)$ where $\bar{\theta}_k=c\theta_k^*+(1-c)\widehat{\theta}_{k,\tau}$ for some constant $c\in(0,1)$. 

  From \eqref{eq:mean_value2}, we have 
$g_{k,\tau}(\widehat{\theta}_{k,\tau})-g_{k,\tau}(\theta_k^*)=(L_{k,\tau}+E_{k,\tau})(\widehat{\theta}_{k,\tau}-\theta^*_k).$
Then, for any $z\in\mathbb{R}^{r},$ we have
\begin{align}
z^\top (\widehat{\theta}_{k,\tau}-\theta_k^*) &=z^\top (L_
{k,\tau}+E_{k,\tau})^{-1}(g_{k,\tau}(\widehat{\theta}_{k,\tau})-g_{k,\tau}(\theta_k^*))\cr &=z^\top L_{k,\tau}^{-1}(g_{k,\tau}(\widehat{\theta}_{k,\tau})-g_{k,\tau}(\theta_k^*))-z^\top L_{k,\tau}^{-1}E_{k,\tau}(L_{k,\tau}+E_{k,\tau})^{-1}(g_{k,\tau}(\widehat{\theta}_{k,\tau})-g_{k,\tau}(\theta_k^*)).\cr\label{eq:z_theta_decom} \end{align}
For obtaining a bound for $|z^\top(\widehat{\theta}_{k,\tau}-\theta_k^*)|$, we analyze the two terms in \eqref{eq:z_theta_decom}.
 We first provide a bound for $|z^\top L_{k,\tau}^{-1}(g_{k,\tau}(\widehat{\theta}_{k,\tau})-g_{k,\tau}(\theta_k^*))|$.
Let $\epsilon_{n,t}=y_{n,k,t}-p(n|S_{k,t},\theta_k^*)$ for $n\in S_{k,t}$. Since $\widehat{\theta}_{k,\tau}$ is the solution from MLE such that $\sum_{t\in\mathcal{T}_{k,\tau}}\sum_{n\in S_{k,t}}(p(n|S_{k,t},\widehat{\theta}_{k,\tau})-y_{n,k,\tau})z_{n}=0$, we have
\begin{align}
   &g_{k,\tau}(\widehat{\theta}_{k,\tau})-g_{k,\tau}(\theta_k^*)\cr&=\sum_{t\in\mathcal{T}_{k,\tau}}\sum_{n\in S_{k,t}}\left(p(n|S_{k,t},\widehat{\theta}_{k,\tau})-p(n|S_{k,t},\theta_k^*)\right)z_{n}+(\widehat{\theta}_{k,\tau}-\theta_k^*)\cr 
&=\sum_{t\in\mathcal{T}_{k,\tau}}\sum_{n\in S_{k,t}}\left(p(n|S_{k,t},\widehat{\theta}_{k,\tau})-y_{n,k,t}\right)z_{n}+\widehat{\theta}_{k,\tau}+\sum_{t\in\mathcal{T}_{k,\tau}}\sum_{n\in S_{k,t}}\left(y_{n,k,\tau}-p(n|S_{k,t},\theta^*_k)\right)z_{n}-\theta_k^*\cr 
&=0+\sum_{t\in\mathcal{T}_{k,\tau}}\sum_{n\in S_{k,t}}\epsilon_{n,t}z_{n}-\theta_k^*
\end{align}

We define 
\[Z_{k,t}=[z_{n}: n\in S_{k,t}]^\top\in\mathbb{R}^{|S_{k,t}|\times r} \text{ for } t\in\mathcal{T}_{k,\tau},\]
\[D_{k,\tau}=[Z_{k,t}: t\in \mathcal{T}_{k,\tau}]^\top\in\mathbb{R}^{(\sum_{t\in\mathcal{T}_{k,\tau}}|S_{k,t}|)\times r},\]
\[\mathcal{E}_{k,t}=[\epsilon_{n,t}:n\in S_{k,t}]^\top\in \mathbb{R}^{|S_{k,t}|}.\] 
Then using Hoeffding inequality, we have
\begin{align}
    \mathbb{P}(|z^\top L_{k,\tau}^{-1}(g_{k,\tau}(\widehat{\theta}_{k,\tau})-g_{k,\tau}(\theta_k^*))|\ge \nu)&\le \mathbb{P}\left(\left|\sum_{t\in\mathcal{T}_{k,\tau}}z^\top L_{k,\tau}^{-1}Z_{k,t}^\top \mathcal{E}_{k,t}\right|\ge \nu-|z^\top L_{k,\tau}^{-1}\theta_k^*|\right) \cr&\le \mathbb{P}\left(\left|\sum_{t\in\mathcal{T}_{k,\tau}}z^\top L_{k,\tau}^{-1}Z_{k,t}^\top \mathcal{E}_{k,t}\right|\ge \nu-1\right) \cr &\le 2\exp\left(-\frac{2(\nu-1)^2}{\sum_{t\in\mathcal{T}_{k,\tau}}(2\sqrt{2}\|z^\top L_{k,\tau}^{-1}Z_{k,t}^\top\|_2)^2}\right)\cr 
    &=2\exp\left(-\frac{(\nu-1)^2}{4\|z^\top L_{k,\tau}^{-1}D_{k,\tau}^\top\|_2^2}\right)\cr 
    &\le 2\exp\left(-\frac{\kappa^2(\nu-1)^2}{4\|z\|_{V_{k,\tau}^{-1}}^2}\right),\label{eq:z_theta_hoe}
\end{align}
where the last inequality is obtained from the fact that
\begin{align*}
L_{k,\tau}&=\sum_{t\in\mathcal{T}_{k,\tau}}\left(\sum_{n\in S_{k,t}}p(n|S_{k,t},\theta_k^*)z_{n}z_{n}^\top - \sum_{n\in S_{k,t}}\sum_{m\in S_{k,t}}p(n|S_{k,t},\theta_k^*)p(m|S_{k,t},\theta_k^*)z_{n}z_{m}^\top\right)\cr &=\sum_{t\in\mathcal{T}_{k,\tau}}\left(\sum_{n\in S_{k,t}}p(n|S_{k,t},\theta_k^*)z_{n}z_{n}^\top - \frac{1}{2}\sum_{n\in S_{k,t}}\sum_{m\in S_{k,t}}p(n|S_{k,t},\theta_k^*)p(m|S_{k,t},\theta_k^*)(z_{n}z_{m}^\top+z_{m}z_{n}^\top)\right)\cr 
& \succeq \sum_{t\in\mathcal{T}_{k,\tau}}\left(\sum_{n\in S_{k,t}}p(n|S_{k,t},\theta_k^*)z_{n}z_{n}^\top - \frac{1}{2}\sum_{n\in S_{k,t}}\sum_{m\in S_{k,t}}p(n|S_{k,t},\theta_k^*)p(m|S_{k,t},\theta_k^*)(z_{n}z_{n}^\top+z_{m}z_{m}^\top)\right)\cr 
& = \sum_{t\in\mathcal{T}_{k,\tau}}\left(\sum_{n\in S_{k,t}}p(n|S_{k,t},\theta_k^*)z_{n}z_{n}^\top - \sum_{n\in S_{k,t}}\sum_{m\in S_{k,t}}p(n|S_{k,t},\theta_k^*)p(m|S_{k,t},\theta_k^*)z_{n}z_{n}^\top\right)\cr 
& = \sum_{t\in\mathcal{T}_{k,\tau}}\left(\sum_{n\in S_{k,t}}p(n|S_{k,t},\theta_k^*)p(n_0|S_{k,t},\theta_k^*)z_{n}z_{n}^\top\right) \cr &\succeq \kappa D_\tau^\top D_\tau(=\kappa V_{k,\tau}),
\end{align*}
where the first inequality is obtained from $(z_n-z_m)(z_n-z_m)^\top=z_nz_n^\top +z_mz_m^\top -z_nz_m^\top-z_mz_n^\top \succeq 0.$

Then from \eqref{eq:z_theta_hoe} using $\nu=(2/\kappa)\sqrt{\log(2TKN/\delta)}\|z\|_{V_{k,\tau}^{-1}}+1$ and the union bound, with probability at least $1-\delta$, for all $\tau\in[T],k\in[K]$, we have
\begin{align}
|z^\top L_{k,\tau}^{-1}(g_{k,\tau}(\widehat{\theta}_{k,\tau})-g_{k,\tau}(\theta_k^*))|\le \frac{3\sqrt{\log(T KN/\delta)}}{\kappa}\|z\|_{V_{k,\tau}^{-1}}.\label{eq:zLg_bd}    
\end{align}

Now we provide a bound for the second term in \eqref{eq:z_theta_decom} of $|z^\top L_{k,\tau}^{-1}E_{k,\tau}(L_{k,\tau}+E_{k,\tau})^{-1}(g_{k,\tau}(\widehat{\theta}_{k,\tau})-g_{k,\tau}(\theta_k^*))|$. We have 
\begin{align}
    &|z^\top L_{k,\tau}^{-1}E_{k,\tau}(L_{k,\tau}+E_{k,\tau})^{-1}(g_{k,\tau}(\widehat{\theta}_{k,\tau})-g_{k,\tau}(\theta_k^*))|\cr &\le \|z\|_{L_{k,\tau}^{-1}}\|L_{k,\tau}^{-1/2}E_{k,\tau}(L_{k,\tau}+E_{k,\tau})^{-1}L^{1/2}\|_2\|g_{k,\tau}(\widehat{\theta}_{k,\tau})-g_{k,\tau}(\theta_k^*)\|_{L_{k,\tau}^{-1}}\cr 
    &\le  (1/\kappa)\|z\|_{V_{k,\tau}^{-1}}\|L_{k,\tau}^{-1/2}E_{k,\tau}(L_{k,\tau}+E_{k,\tau})^{-1}L^{1/2}\|_2\|g_{k,\tau}(\widehat{\theta}_{k,\tau})-g_{k,\tau}(\theta_k^*)\|_{V_{k,\tau}^{-1}}.\label{eq:z_Lg_bd}
\end{align}
Then it follows that 
\begin{align*}
    &\|L_{k,\tau}^{-1/2}E_{k,\tau}(L_{k,\tau}+E_{k,\tau})^{-1}L^{1/2}\|_2\cr &= \|L_{k,\tau}^{-1/2}E_{k,\tau}(L_{k,\tau}^{-1}-L_{k,\tau}^{-1}E_{k,\tau}(L_{k,\tau}+E_{k,\tau})^{-1}L^{1/2}\|_2\cr &
    \le \|L_{k,\tau}^{-1/2}E_{k,\tau}L_{k,\tau}^{-1/2}\|_2+\|L_{k,\tau}^{-1/2}E_{k,\tau}L_{k,\tau}^{-1/2}\|_2\|L_{k,\tau}^{-1/2}E_{k,\tau}(L_{k,\tau}+E_{k,\tau})^{-1}L_{k,\tau}^{1/2}\|_2,
\end{align*}
which implies
\begin{align}
    \|L_{k,\tau}^{-1/2}E_{k,\tau}(L_{k,\tau}+E_{k,\tau})^{-1}L_{k,\tau}^{1/2}\|_2 &\le \frac{\|L_{k,\tau}^{-1/2}E_{k,\tau}L_{k,\tau}^{-1/2}\|_2}{1-\|L_{k,\tau}^{-1/2}E_{k,\tau}L_{k,\tau}^{-1/2}\|_2}\cr 
    &\le 2 \|L_{k,\tau}^{-1/2}E_{k,\tau}L_{k,\tau}^{-1/2}\|_2\cr 
    &\le \frac{6}{\kappa}\|\widehat{\theta}_{k,\tau}-\theta_{k}^*\|_2,\label{eq:LE_bd}
\end{align}
where the last inequality is obtained from (17) and (18) in \citet{oh2021multinomial}. 
Then from \eqref{eq:z_Lg_bd}, \eqref{eq:LE_bd}, we have 
\begin{align}
    &|z^\top L_{k,\tau}^{-1}E_{k,\tau}(L_{k,\tau}+E_{k,\tau})^{-1}(g_{k,\tau}(\widehat{\theta}_{k,\tau})-g_{k,\tau}(\theta_k^*))|\cr &\le \frac{6}{\kappa^2}\|\widehat{\theta}_{k,\tau}-\theta_{k}^*\|_2\|g_{k,\tau}(\widehat{\theta}_{k,\tau})-g_{k,\tau}(\theta_k^*)\|_{V_{k,\tau}^{-1}}\|z\|_{V_{k,\tau}^{-1}}.\label{eq:zLE_bd}
\end{align}
We can conclude the proof from \eqref{eq:zLg_bd} and \eqref{eq:zLE_bd}.

\subsubsection{Proof of Lemma~\ref{lem:theta_hat_theta_gap_bd}}\label{app:theta_hat_theta_gap_bd}

    We note that $g_{k,\tau}(\theta_1)-g_{k,\tau}(\theta_2)=\sum_{t\in \mathcal{T}_{k,\tau}}\sum_{n\in S_{k,t}}(p(n,|S_{k,t},\theta_1)-p(n,|S_{k,t},\theta_2))z_{n}+(\theta_1-\theta_2)$. Define $H_{k,\tau}({\theta})=\sum_{t\in\mathcal{T}_{k,\tau}}\left(\sum_{n\in S_{k,t}}p(n|S_{k,t},{\theta})z_{n}z_{n}^\top-\sum_{n\in S_{k,t}}\sum_{m\in S_{k,t}}p(n|S_{k,t},{\theta})p(m|S_{k,t},{\theta})z_{n}z_{m}^\top\right)+I_r$.
    Then we can show that there exists $\bar{\theta}=c\theta_1 +(1-c)\theta_2$ with some $c\in(0,1)$ such that 
\begin{align}
    &g_{k,\tau}(\theta_1)-g_{k,\tau}(\theta_2)\cr&=\nabla_\theta g_{k,\tau}(\theta)\big|_{\theta=\bar{\theta}}(\theta_1-\theta_2)\cr &=\left(\sum_{t\in\mathcal{T}_{k,\tau}}\left(\sum_{n\in S_{k,t}}p(n|S_{k,t},\bar{\theta})z_{n}z_{n}^\top-\sum_{n\in S_{k,t}}\sum_{m\in S_{k,t}}p(n|S_{k,t},\bar{\theta})p(m|S_{k,t},\bar{\theta})z_{n}z_{m}^\top\right)+I_r\right)(\theta_1-\theta_2)\cr &=H_{k,\tau}(\bar{\theta})(\theta_1-\theta_2).\end{align}

    Define $\bar{H}_{k,\tau}(\bar{\theta})=\sum_{t\in\mathcal{T}_{k,\tau}}\sum_{n \in S_{k,t}} p(n|S_{k,t},\bar{\theta})p(n_0|S_{k,t},\bar{\theta})z_{n}z_{n}^\top+I_r$. Then we have $H_{k,\tau}(\bar{\theta})\succeq \bar{H}_{k,\tau}(\bar{\theta})$ from the following.
\begin{align}
    &\sum_{t\in\mathcal{T}_{k,\tau}}\left(\sum_{n\in S_{k,t}}p(n|S_{k,t},\bar{\theta})z_{n}z_{n}^\top-\sum_{n\in S_{k,t}}\sum_{m\in S_{k,t}}p(n|S_{k,t},\bar{\theta})p(m|S_{k,t},\bar{\theta})z_{n}z_{m}^\top\right)
    \cr&=\sum_{t\in\mathcal{T}_{k,\tau}}\left(\sum_{n\in S_{k,t}}p(n|S_{k,t},\bar{\theta})z_{n}z_{n}^\top - \sum_{n\in S_{k,t}}\sum_{m\in S_{k,t}}p(n|S_{k,t},\bar{\theta})p(m|S_{k,t},\bar{\theta})z_{n}z_{m}^\top\right)\cr &=\sum_{t\in\mathcal{T}_{k,\tau}}\left(\sum_{n\in S_{k,t}}p(n|S_{k,t},\bar{\theta})z_{n}z_{n}^\top - \frac{1}{2}\sum_{n\in S_{k,t}}\sum_{m\in S_{k,t}}p(n|S_{k,t},\bar{\theta})p(m|S_{k,t},\bar{\theta})(z_{n}z_{m}^\top+z_{m}z_{n}^\top)\right)\cr 
& \succeq \sum_{t\in\mathcal{T}_{k,\tau}}\left(\sum_{n\in S_{k,t}}p(n|S_{k,t},\bar{\theta})z_{n}z_{n}^\top - \frac{1}{2}\sum_{n\in S_{k,t}}\sum_{m\in S_{k,t}}p(n|S_{k,t},\bar{\theta})p(m|S_{k,t},\bar{\theta})(z_{n}z_{n}^\top+z_{m}z_{m}^\top)\right)\cr 
& = \sum_{t\in\mathcal{T}_{k,\tau}}\left(\sum_{n\in S_{k,t}}p(n|S_{k,t},\bar{\theta})z_{n}z_{n}^\top - \sum_{n\in S_{k,t}}\sum_{m\in S_{k,t}}p(n|S_{k,t},\bar{\theta})p(m|S_{k,t},\bar{\theta})z_{n}z_{n}^\top\right)\cr 
& = \sum_{t\in\mathcal{T}_{k,\tau}}\left(\sum_{n\in S_{k,t}}p(n|S_{k,t},\bar{\theta})p(n_0|S_{k,t},\bar{\theta})z_{n}z_{n}^\top\right), 
    \label{eq:mean_value}
\end{align}
where the inequality is obtained from $(z_n-z_m)(z_n-z_m)^\top \succeq 0$. Under $E_1$,  we have $\|\widehat{\theta}_{k,\tau}\|_2-\|\theta_k^*\|_2\le 1$ implying $\|\widehat{\theta}_{k,\tau}\|_2\le 1+\|\theta_k^*\|_2= 1+\|U_{r}^\top \theta_k\|_2\le 2$. Then for $\bar{\theta}=c\widehat{\theta}_{k,\tau}+(1-c)\theta_k^*$ for some $c\in (0,1)$, we have $\|U_{r} \bar{\theta}\|_2\le 2$. Then from $p(n|S_{k,t},\bar{\theta})=\exp(z_n^\top \bar{\theta})/(1+\sum_{m\in S_{k,t}}\exp(z_m^\top \bar{\theta}))=\exp(x_n^\top (U_r\bar{\theta}))/(1+\sum_{m\in S_{k,t}}\exp(x_m^\top (U_r\bar{\theta})))$, we can show that $\bar{H}_{k,\tau}(\bar{\theta})\succeq \kappa V_{k,\tau}$, which implies 
 $H_{k,\tau}(\bar{\theta})\succeq\bar{H}_{k,\tau}(\bar{\theta})\succeq \kappa V_{k,\tau}$. 

Then we have 
\begin{align}
\|\widehat{\theta}_{k,\tau}-\theta_k^*\|_2^2 &\le (1/\lambda_{\min}(V_{k,\tau}))(\widehat{\theta}_{k,\tau}-\theta_k^*)^\top V_{k,\tau}(\widehat{\theta}_{k,\tau}-\theta_k^*)\cr 
& \le (1/\kappa\lambda_{\min}(V_{k,\tau}^0))(\widehat{\theta}_{k,\tau}-\theta_k^*)^\top H_{k,\tau}(\bar{\theta})(\widehat{\theta}_{k,\tau}-\theta_k^*)\cr 
& \le (1/\kappa\lambda_{\min}(V_{k,\tau}^0))(\widehat{\theta}_{k,\tau}-\theta_k^*)^\top H_{k,\tau}(\bar{\theta}) H_{k,\tau}(\bar{\theta})^{-1} H_{k,\tau}(\bar{\theta})(\widehat{\theta}_{k,\tau}-\theta_k^*)\cr &
\le (1/\kappa^2\lambda_{\min}(V_{k,\tau}^0))(g_{k,\tau}(\widehat{\theta}_{k,\tau})-g_{k,\tau}(\theta_k^*))^\top V_{k,\tau}^{-1} (g_{k,\tau}(\widehat{\theta}_{k,\tau})-g_{k,\tau}(\theta_k^*))\cr 
&\le (1/\kappa^2\lambda_{\min}(V_{k,\tau}^0))\|g_{k,\tau}(\widehat{\theta}_{k,\tau})-g_{k,\tau}(\theta_k^*))\|_{V_{k,\tau}^{-1}}^2.
\end{align}
Then from $E_2$, we can conclude that
\[\|\widehat{\theta}_{k,\tau}-\theta_k^*\|_2\le \frac{4}{\kappa}\sqrt{\frac{2r+\log(KTN/\delta)}{\lambda_{\min}(V_{k,\tau}^0)}}.\]

\subsection{Proof of Proposition~\ref{prop:UCB}}\label{app:thm_ucb}

We first provide a lemma for a confidence bound. Let $\gamma_t(\delta)=c_1\sqrt{d}\log(L)\left(\log(t)+\sqrt{\log(t)\log(K/\delta)}\right)$ for some $c_1>0$.

\begin{lemma}[Lemma 1 in \cite{lee2024nearly}]
     With probability at least  $1-\delta$, for all $t\ge 1$ and $k\in[K]$ we have 
    \begin{align*}
        \|\widehat{\theta}_{k,t}-\theta^*_k\|_{\Gcal_{k,t}}\le \gamma_t(\delta).
\end{align*}\label{lem:ucb_confi} 
\end{lemma}

Let $\delta=1/T$. From the above lemma, we define event $E=\{\|\widehat{\theta}_{k,t}-\theta_k^*\|_{\Gcal_{k,t}}\le \gamma_t \: \forall k\in[K] \text{ and } t\ge 1\}$, which holds with probability at least $1-1/T$. Then we provide a lemma for the optimism. 

\begin{lemma}
     Under $E$, for all $t\ge 1$,  we have
    $$\sum_{k\in[K]}R_k(S_k^*)\le \sum_{k\in[K]}R^{UCB}_{k,t}(S_{k,t}).$$\label{eq:UCB_optimism}
\end{lemma}
\begin{proof}
    Under $E$, we have 
    \begin{align*}
        |z_n^\top \widehat{\theta}_{k,t}-z_n^\top \theta_{k}^*|\le \|z_n\|_{\Gcal_{k,t}^{-1}}\|\widehat{\theta}_{k,t}-\theta_{k}^*\|_{\Gcal_{k,t}}\le \gamma_{t}\|z_n\|_{\Gcal_{k,t}^{-1}},
    \end{align*}
    which implies $z_n^\top \theta_{k}^*\le z_n^\top \widehat{\theta}_{k,t}+\gamma_{t}\|z_n\|_{\Gcal_{k,t}^{-1}}=h_{n,k,t}$. Therefore, from Lemma A.3 in \cite{agrawal2017mnl}, we have
    $R_k(S_k^*)\le R^{UCB}_{k,t}(S_k^*).$
    Then using definition of $S_{k,t}$ in the algorithm, we can conclude that
    \[\sum_{k\in[K]}R_k(S_k^*)\le \sum_{k\in[K]}R^{UCB}_{k,t}(S_k^*)\le\sum_{k\in[K]} R^{UCB}_{k,t}(S_{k,t}).\]
\end{proof}
Now we provide a lemma which is critical to bound regret under optimism.

\begin{lemma}\label{lem:R_UCBR_bd_naive}
     Under $E$, for all $k\in[K]$, we have \[\sum_{t=1}^TR_{k,t}^{UCB}(S_{k,t})-R_k(S_{k,t})=O\left(r\sqrt{T}+\frac{1}{\kappa}r^2\right)\]
 \end{lemma}
 \begin{proof} 
 By following the proof steps in Theorem 4 in \cite{lee2024nearly}, we can show this lemma. 

\end{proof}

Then from Lemmas~\ref{lem:ucb_confi} and \ref{lem:R_UCBR_bd_naive}, we can conclude the proof for the regret as follows.
\begin{align*}
\mathcal{R}(T)&=\mathbb{E}\left[\sum_{t\in[T]}\sum_{k\in[K]}R_k(S_{k,t}^*)-R_k(S_{k,t})\right]
\cr &\le\mathbb{E}\left[ \sum_{t=1}^T\sum_{k\in[K]}\left({R}_k(S_{k,t}^*)-R_k(S_{k,t})\right)\mathbf{1}(E)\right]+\mathbb{E}\left[ \sum_{t=1}^T\sum_{k\in[K]}\left({R}_k(S_{k,t}^*)-R_k(S_{k,t})\right)\mathbf{1}(E^c)\right]\cr
&\le \mathbb{E}\left[ \sum_{t=1}^T\sum_{k\in[K]}\left(R^{UCB}_{k,t}(S_{k,t})-R_k(S_{k,t})\right)\mathbf{1}(E)\right]+ \sum_{t=1}^T\sum_{k\in[K]}\mathbb{P}(E^c)\cr &= \tilde{\Ocal}\left(rK\sqrt{T}+\frac{1}{\kappa}r^2K\right)   
=\tilde{\Ocal}\left(rK\sqrt{T}\right). 
\end{align*}

Now we discuss the computational cost. Since there exists $\mathcal{O}(K^N)$ number of assortment candidate in $\Mcal$, especially for $L\ge N$, the cost per round is $\mathcal{O}(K^N)$ from Line~\ref{line:combinatorial}. 

\subsection{Proof of Lemma~\ref{lem:D_opt}}\label{app:proof-kw}

Fix an epoch $\tau$ and an arm $k$. Recall that the projected features satisfy $z_n\in\mathbb{R}^r$ and define, for any
$\pi\in \mathcal{P}(\mathcal{N}_{k,\tau})$,
\[
V(\pi):=\sum_{n\in \mathcal{N}_{k,\tau}}\pi(n)\,z_n z_n^\top,\qquad
W(\pi):=V(\pi)+\alpha I_r,\quad \alpha:=\frac{1}{rT_\tau},
\]
and
\[
g(\pi):=\max_{n\in \mathcal{N}_{k,\tau}} z_n^\top W(\pi)^{-1} z_n
= \max_{n\in \mathcal{N}_{k,\tau}} \|z_n\|_{(V(\pi)+\alpha I_r)^{-1}}^2.
\]
Let $\pi_{k,\tau}$ be a (regularized) $G$-optimal design, i.e.,
$\pi_{k,\tau}\in\arg\min_{\pi\in\mathcal{P}(\mathcal{N}_{k,\tau})} g(\pi)$, and denote
$W_*:=W(\pi_{k,\tau})$ and $g_*:=g(\pi_{k,\tau})$.

\paragraph{Step 1: Active constraints on the support.}
Let $S:=\mathrm{supp}(\pi_{k,\tau})$. We claim that for every $n\in S$,
\begin{equation}\label{eq:active}
z_n^\top W_*^{-1} z_n = g_*.
\end{equation}
This is a standard optimality (equivalence/KKT) condition. Indeed, consider the equivalent regularized $D$-optimal
problem
\[
\max_{\pi\in\mathcal{P}(\mathcal{N}_{k,\tau})} \ \log\det(W(\pi)).
\]
Since $\log\det(\cdot)$ is concave and $W(\pi)$ is affine in $\pi$, KKT conditions are necessary and sufficient.
The gradient is
\[
\frac{\partial}{\partial \pi(n)} \log\det(W(\pi))
= \mathrm{trace}\!\left(W(\pi)^{-1} z_n z_n^\top\right)
= z_n^\top W(\pi)^{-1} z_n.
\]
Thus there exists $\lambda\in\mathbb{R}$ such that
\[
z_n^\top W_*^{-1} z_n \le \lambda\quad \forall n\in \mathcal{N}_{k,\tau},
\qquad\text{and}\qquad
z_n^\top W_*^{-1} z_n = \lambda\quad \forall n\in S.
\]
By definition, $\lambda=\max_n z_n^\top W_*^{-1}z_n=g_*$, proving~\eqref{eq:active}.

\paragraph{Step 2: Bound $g(\pi_{k,\tau})\le r$.}
Using~\eqref{eq:active} and $\sum_{n\in S}\pi_{k,\tau}(n)=1$, we obtain
\[
\sum_{n\in \mathcal{N}_{k,\tau}} \pi_{k,\tau}(n)\, z_n^\top W_*^{-1} z_n
=\sum_{n\in S}\pi_{k,\tau}(n)\, g_* = g_*.
\]
On the other hand,
\begin{align*}
\sum_{n\in \mathcal{N}_{k,\tau}} \pi_{k,\tau}(n)\, z_n^\top W_*^{-1} z_n
&= \mathrm{trace}\!\left(\sum_{n\in \mathcal{N}_{k,\tau}}\pi_{k,\tau}(n)\, z_n z_n^\top\, W_*^{-1}\right)\\
&= \mathrm{trace}\!\left(V(\pi_{k,\tau})\, W_*^{-1}\right)\\
&= \mathrm{trace}\!\left((W_*-\alpha I_r)W_*^{-1}\right)
= \mathrm{trace}(I_r) - \alpha\, \mathrm{trace}(W_*^{-1})\\
&= r - \alpha\,\mathrm{trace}(W_*^{-1})
\le r.
\end{align*}
Therefore,
\[
g(\pi_{k,\tau})=g_* \le r.
\]

\paragraph{Step 3: Support-size bound $|S|\le r(r+1)/2$.}
Each matrix $z_n z_n^\top$ is symmetric $r\times r$, hence lies in the vector space $\mathbb{S}^r$ of symmetric matrices
with dimension $\dim(\mathbb{S}^r)=r(r+1)/2$.
If $|S|>r(r+1)/2$, then the set $\{z_n z_n^\top : n\in S\}$ is linearly dependent, so there exists a nonzero
$v:S\to\mathbb{R}$ such that
\[
\sum_{n\in S} v(n)\, z_n z_n^\top = 0.
\]
Multiplying by $W_*^{-1}$ and taking trace yields
\[
0
= \mathrm{trace}\!\left(W_*^{-1}\sum_{n\in S} v(n)\, z_n z_n^\top\right)
= \sum_{n\in S} v(n)\, \mathrm{trace}\!\left(W_*^{-1} z_n z_n^\top\right)
= \sum_{n\in S} v(n)\, z_n^\top W_*^{-1} z_n.
\]
Using~\eqref{eq:active}, $z_n^\top W_*^{-1} z_n=g_*$ for all $n\in S$, hence
\[
0 = g_* \sum_{n\in S} v(n).
\]
Since $W_*\succ 0$ and $g_*=\max_n z_n^\top W_*^{-1} z_n\ge 0$ (and is positive whenever some $z_n\neq 0$),
we obtain
\[
\sum_{n\in S} v(n)=0.
\]
Define, for $t\in\mathbb{R}$, a perturbed design $\pi(t)$ by
\[
\pi(t)(n)=
\begin{cases}
\pi_{k,\tau}(n) + t\,v(n), & n\in S,\\
0, & n\notin S.
\end{cases}
\]
Because $\sum_{n\in S} v(n)=0$, we have $\sum_{n}\pi(t)(n)=1$ for all $t$.
Moreover,
\[
V(\pi(t)) = \sum_{n\in S}(\pi_{k,\tau}(n)+t v(n)) z_n z_n^\top
= V(\pi_{k,\tau}) + t\sum_{n\in S} v(n) z_n z_n^\top
= V(\pi_{k,\tau}),
\]
so $W(\pi(t))=W(\pi_{k,\tau})=W_*$ for all $t$ such that $\pi(t)\ge 0$. Hence
$g(\pi(t))=\max_n z_n^\top W(\pi(t))^{-1} z_n=\max_n z_n^\top W_*^{-1} z_n=g_*$.

Let
\[
t'=\sup\Bigl\{t>0:\ \pi_{k,\tau}(n)+t v(n)\ge 0\ \ \forall n\in S\Bigr\}.
\]
Then $t'>0$ and at $t=t'$ at least one coordinate in $S$ becomes zero; otherwise we could increase $t$ further and stay
nonnegative, contradicting the definition of $t'$. Thus $\pi(t')$ is feasible, achieves the same value $g_*$, and satisfies
$|\mathrm{supp}(\pi(t'))|\le |S|-1$.

Iterating this support-reduction argument yields an optimal design $\pi$ such that
\[
|\mathrm{supp}(\pi)| \le \frac{r(r+1)}{2}.
\]
This completes the proof.

\clearpage
\subsection{Experiments}\label{sec:exp}
We compare \texttt{B-SMB} and \texttt{B-SMB}$^+$ with three per-round
baselines: the adapted MNL bandit algorithm \texttt{OFU-MNL}$^+$
\citep{lee2024nearly}, and the stable-MNL matching bandit algorithms
\texttt{UCB-QMB} and \texttt{TS-QMB}~\citep{kim2024queueing}. Details on the
adaptation of \texttt{OFU-MNL}$^+$ are provided in Appendix~\ref{app:ucb}. We
sample \(x_n,\theta_k\sim\mathrm{Unif}([-1,1]^d)\) and normalize them, and draw
\(w_{n,k}\sim\mathrm{Unif}[0,1]\). We use \(N=3\), \(K=2\), \(r=2\), and
\(T=5000\) in Figure~\ref{fig:exp}, and \(N=7\), \(K=4\) in
Figure~\ref{fig:exp2}.

Figures~\ref{fig:exp} and~\ref{fig:exp2} report the cumulative number of
optimization updates, runtime, and cumulative regret. As discussed in
Section~\ref{sec:curse_complexity}, the per-round baselines repeatedly solve a
combinatorial optimization problem, so their optimization updates grow linearly
with time. In contrast, our batched algorithms require only a few updates,
consistent with Propositions~\ref{prop:batch_bd} and~\ref{prop:batch_bd2}, and
therefore achieve substantial runtime savings, while maintaining regret comparable to \texttt{OFU-MNL}$^+$. Specifically, the runtime is only
about \(25\%\) of the benchmark runtime for \(N=3,K=2\), and only about \(2\%\)
for \(N=7,K=4\). Equivalently, this corresponds to roughly \(75\%\) and \(98\%\)
reductions in runtime, respectively, even in these relatively small instances. This efficiency gap becomes more pronounced as the problem size increases.
Additional results are reported in Appendix~\ref{app:add_exp}.

\begin{figure*}[h]
\centering\includegraphics[width=0.29\linewidth]{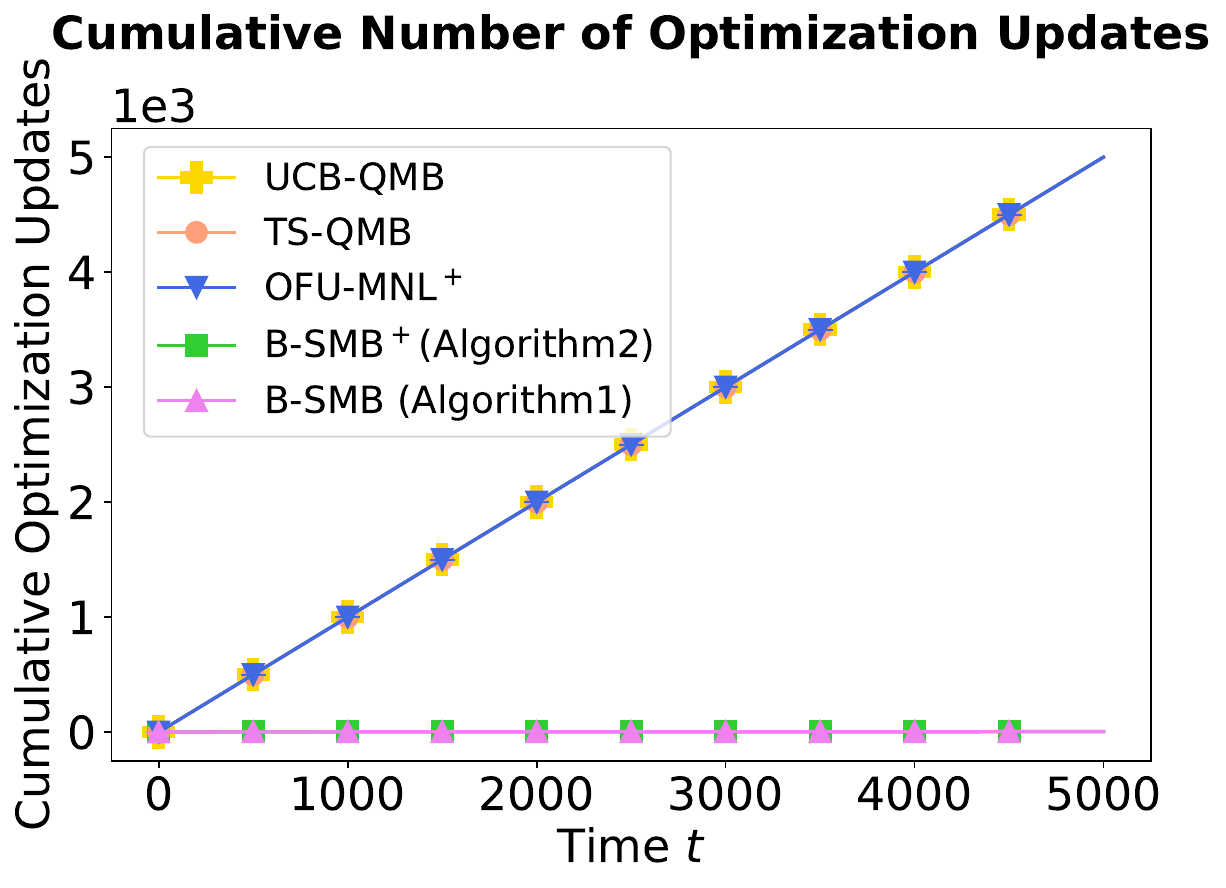}
\includegraphics[width=0.29\linewidth]{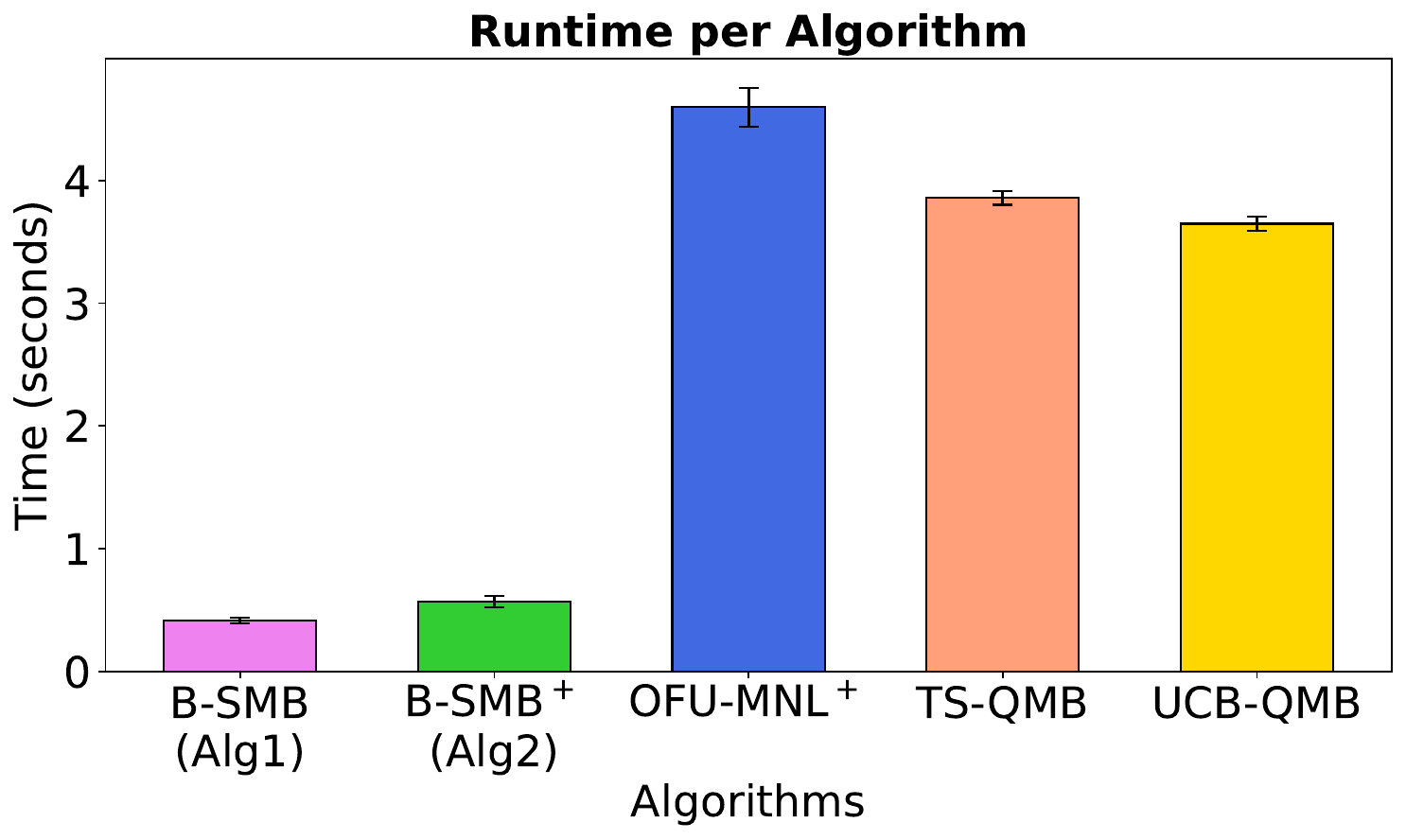}
\includegraphics[width=0.29\linewidth]{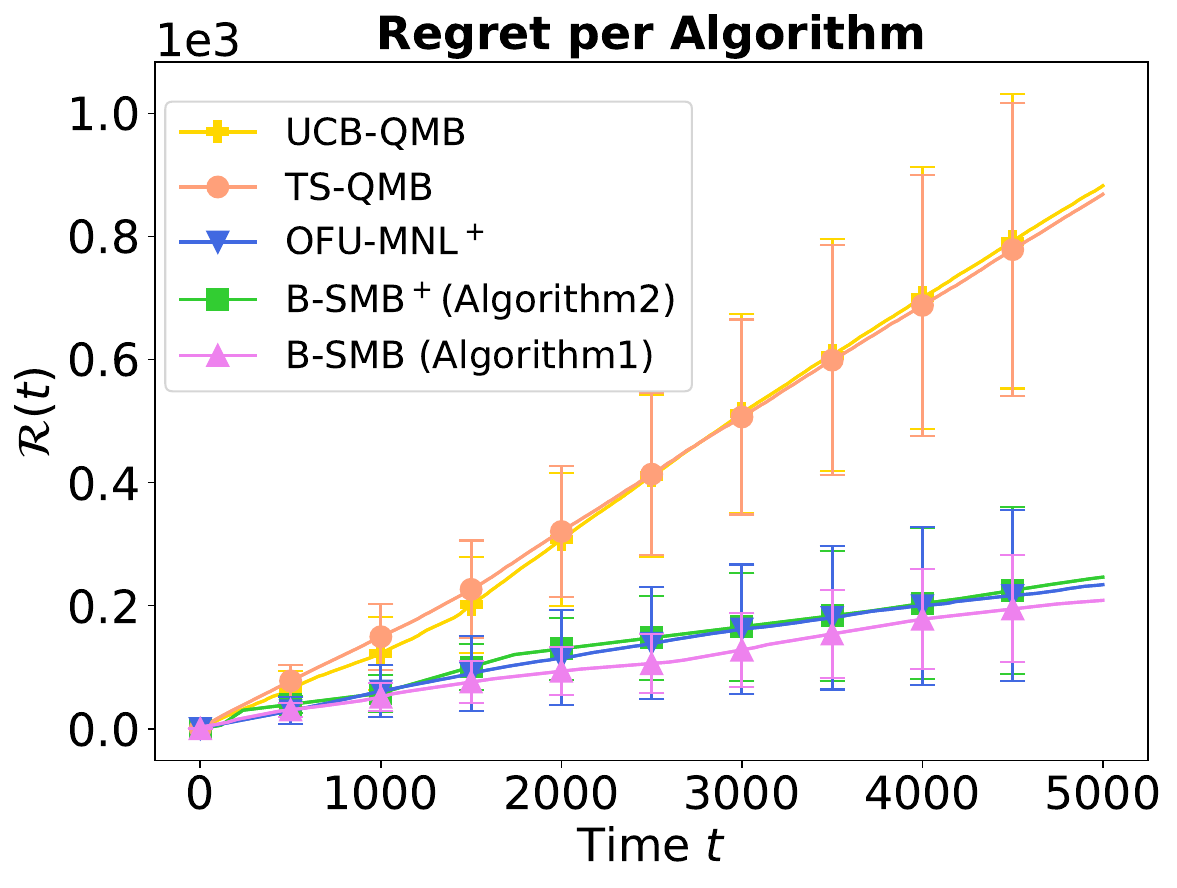}

\caption{Results for \(N=3\), \(K=2\): optimization updates, runtime, and cumulative regret.}

\label{fig:exp}
\end{figure*}

\begin{figure*}[h]
\centering
\includegraphics[width=0.29\linewidth]{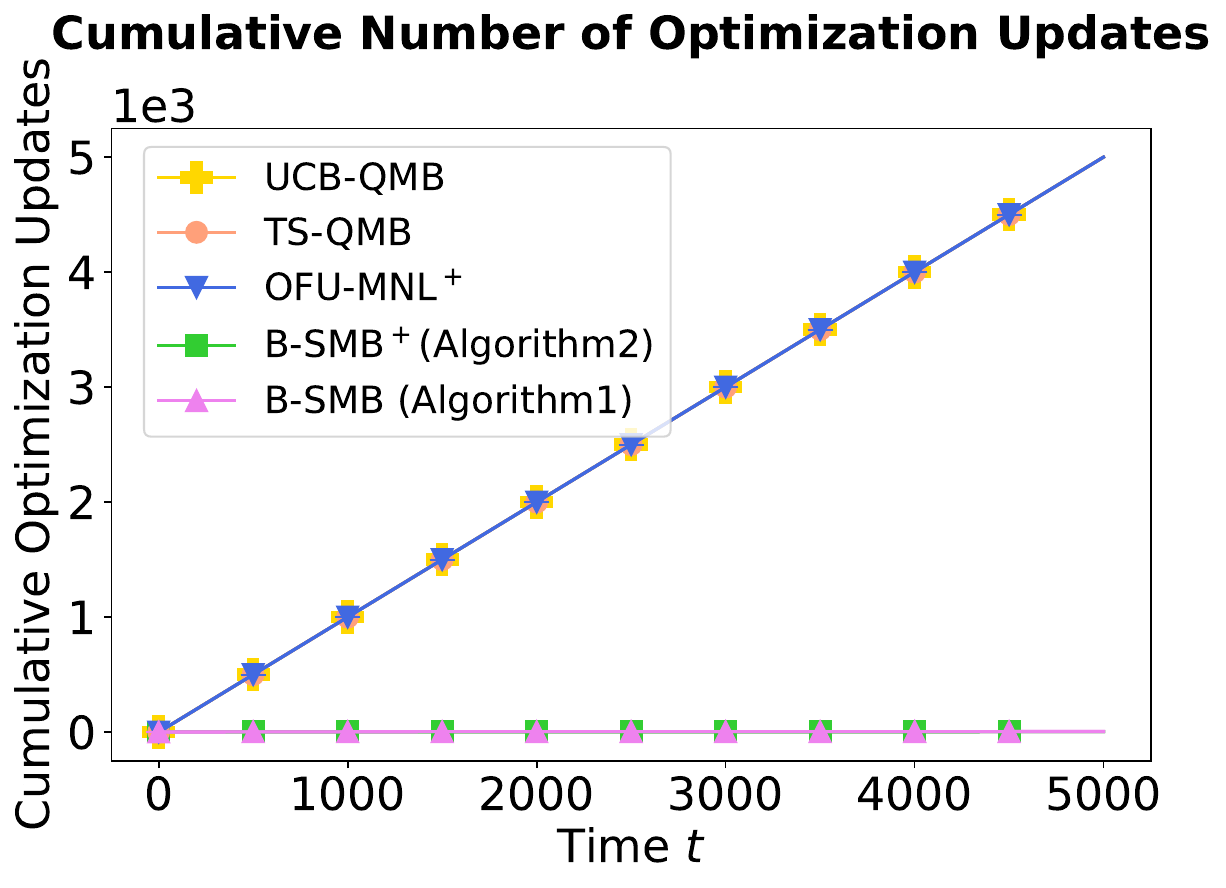}
\includegraphics[width=0.29\linewidth]{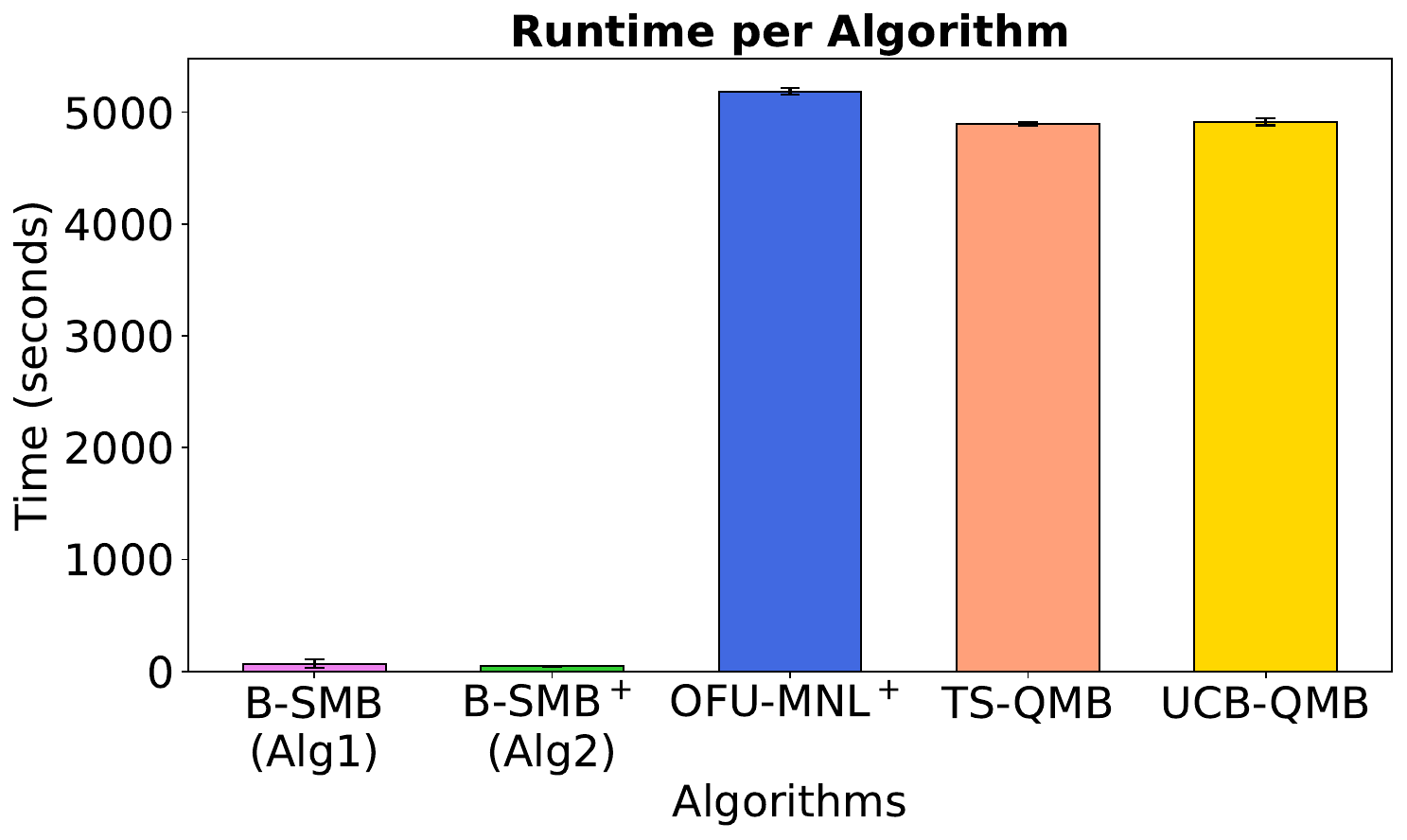}
\includegraphics[width=0.29\linewidth]{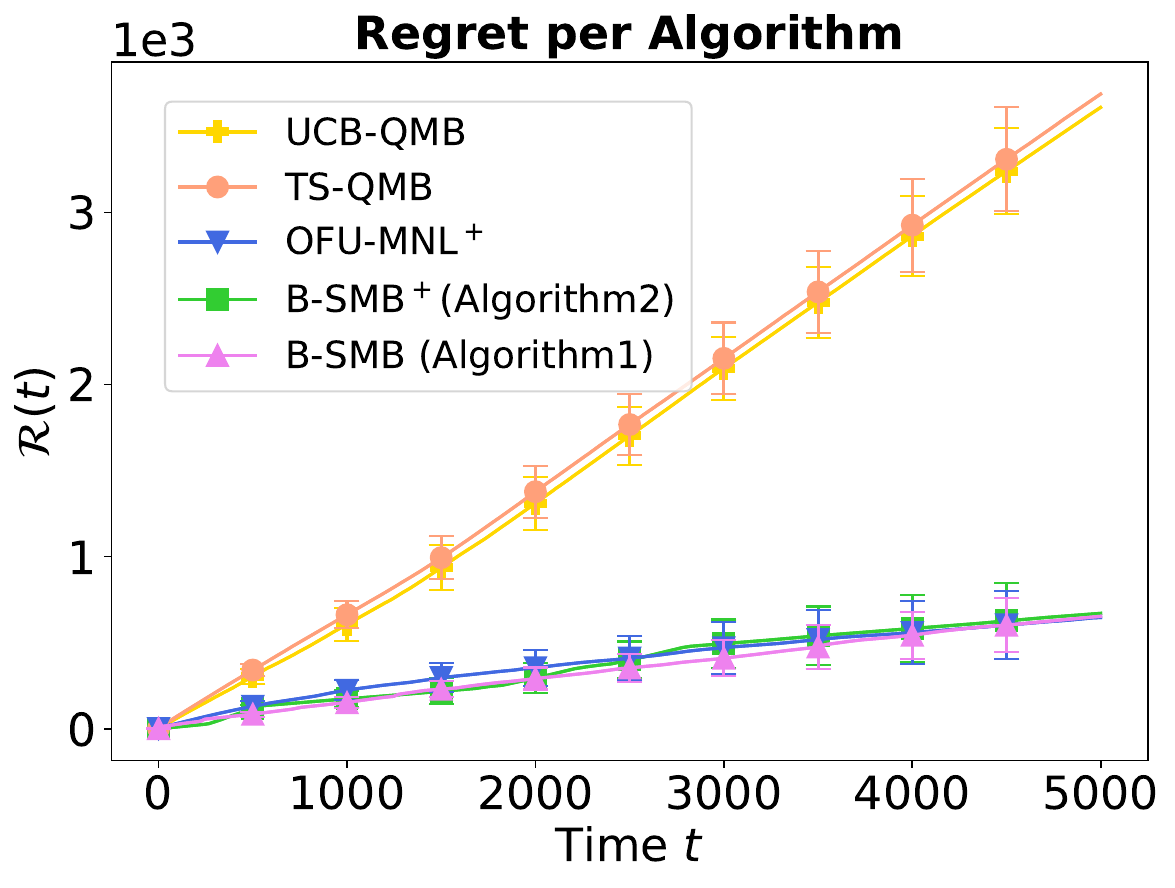}

\caption{Results for $N=7$, $K=4$:  cumulative optimization updates, runtime, cumulative regret.}
\label{fig:exp2}
\end{figure*}

\clearpage
\subsubsection{Additional Scalability Experiment}\label{app:add_exp}

We also conducted a larger experiment with \(N=18\) and \(K=5\). This experiment
is intended as a computational stress test rather than a full regret comparison,
since the per-round exact baseline did not complete the \(5{,}000\)-round run
within 24 hours on the same setup. In this instance, \texttt{B-SMB} and
\texttt{B-SMB}$^+$ completed all \(5{,}000\) rounds in 1419 and 3961 seconds,
respectively, while requiring only 4 and 3 optimization updates. By contrast,
within the same 3961-second time budget, the exact baseline
\texttt{OFU-MNL}$^+$ completed only 44 rounds, which is just \(0.88\%\) of the
rounds completed by our methods. 
\begin{figure}[h]
    \centering
    \includegraphics[width=0.48\linewidth]{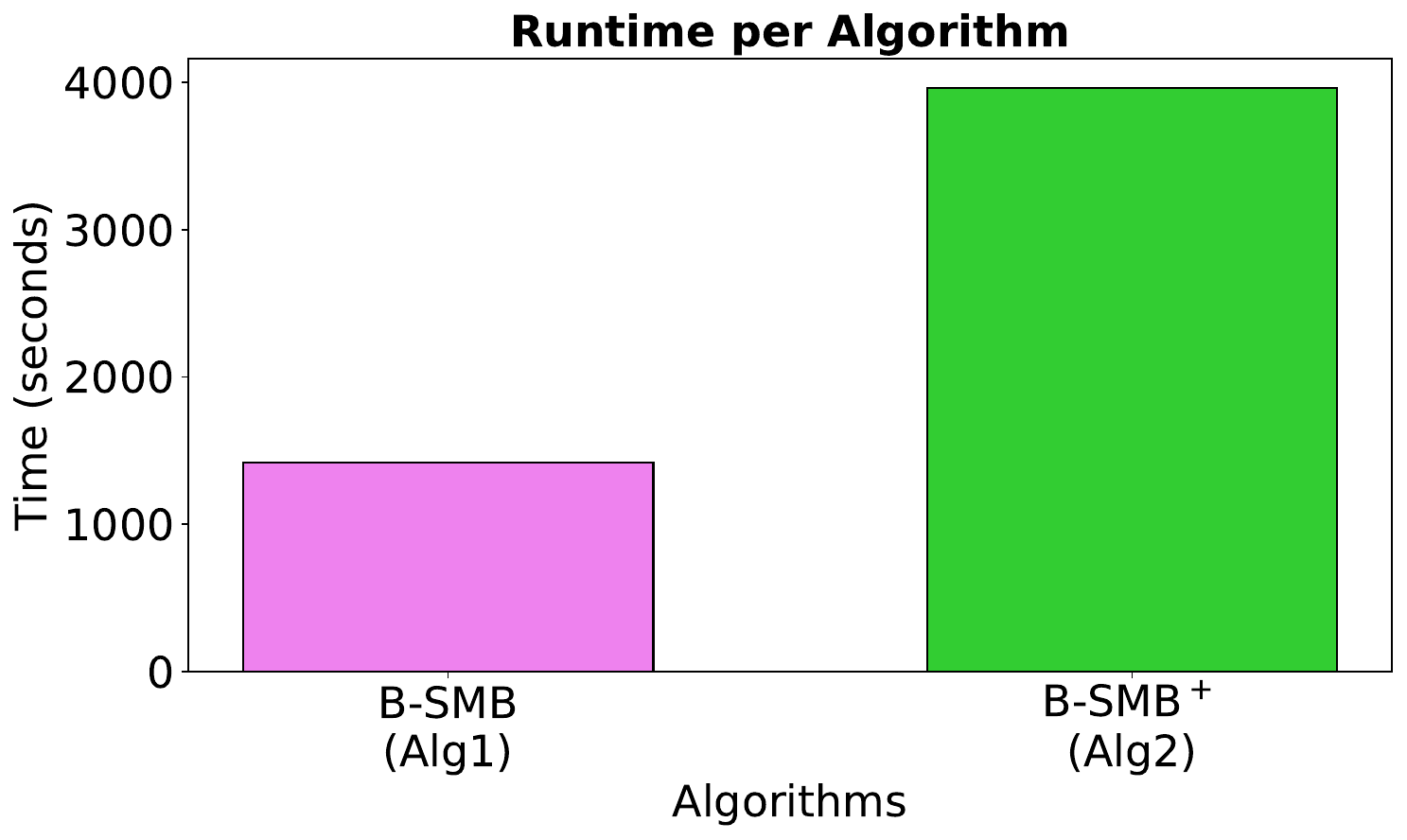}
    \hfill
    \includegraphics[width=0.48\linewidth]{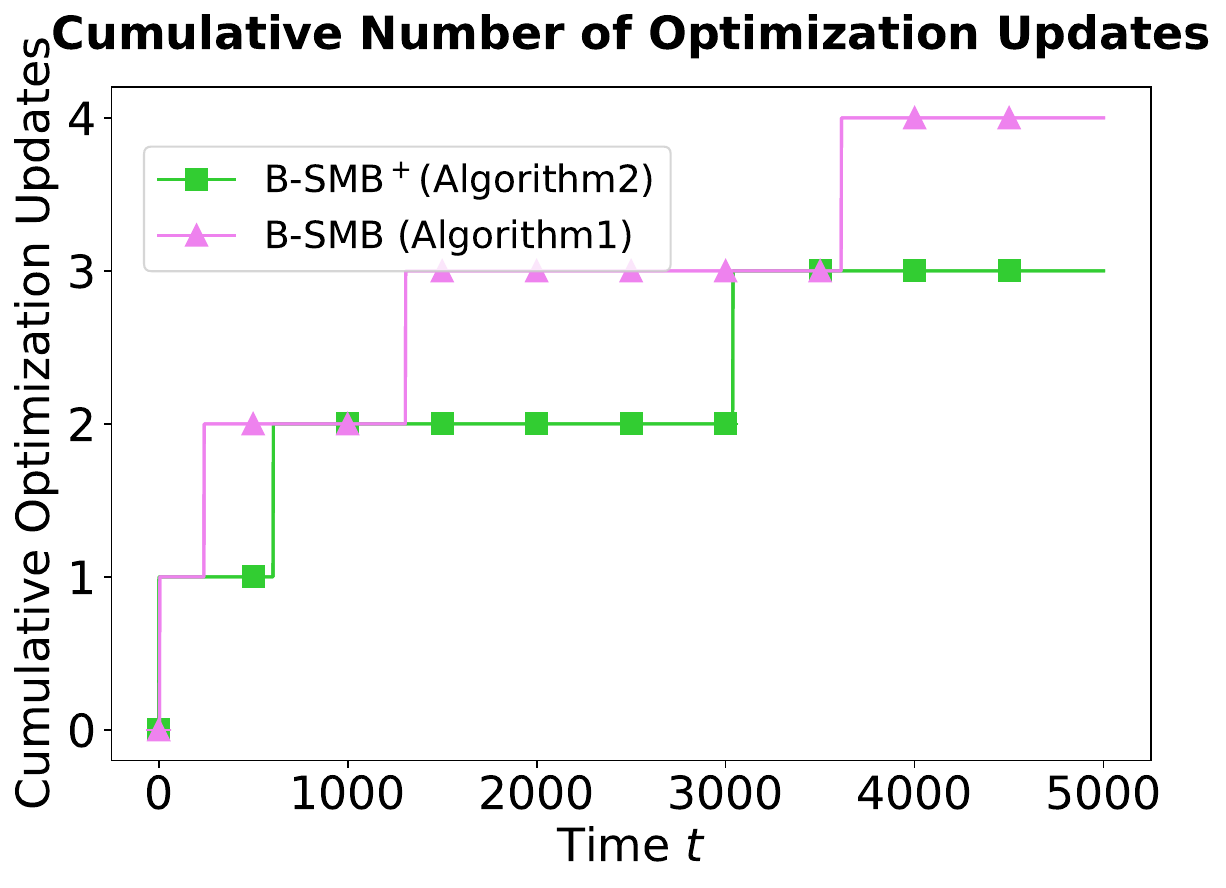}
    \caption{Additional scalability experiment with \(N=18\), \(K=5\). Left: runtime comparison. Right: number of optimization updates.}
    \label{fig:stress_n18k5}
\end{figure}

\subsection{Auxiliary Lemmas}

\begin{lemma}[Lemma E.2 in \cite{lee2024nearly}] For all $t\ge 1$ and $k\in[K]$, we have
\begin{align*} 
    &(i) \quad\sum_{s=1}^t\sum_{n\in S_{k,s}}p(n|S_{k,s},\widehat{\theta}_{k,s})p(n_0|S_{k,s},\widehat{\theta}_{k,s})\|z_n\|_{H_{k,s}^{-1}}^2\le 2r\log\left(1+\tfrac{t}{r\lambda}\right),\cr
&(ii) \quad\sum_{s=1}^t\max_{n\in S_{k,s}}\|z_n\|_{H_{k,s}^{-1}}^2\le \tfrac{1}{\kappa}2r\log\left(1+\tfrac{t}{r\lambda}\right).\label{lem:elip-bd}
\end{align*}
\end{lemma}
\begin{lemma}[Lemma E.3 in \cite{lee2024nearly}]
Define $\tilde{Q}:\RR^{|S|} \rightarrow \RR$ for $S\in [N]$, such that for any $\textbf{u}=(u_1,\dots,u_{|S|})\in \RR^{|S|},$ $\tilde{Q}(\textbf{u})=\sum_{n\in S}\frac{\exp(u_n)}{1+\sum_{m\in S}\exp(u_m)}$. Let $p_n(\textbf{u})=\frac{\exp(u_n)}{1+\sum_{m\in S}\exp(u_m)}$. Then for all $n\in S$, we have 
\[\left|\frac{\partial^2\tilde{Q}}{\partial u_n \partial u_m}\right|\le \begin{cases}
    3p_n(\textbf{u}),& \text{if } n=m\\
    2p_n(\textbf{u})p_m(\textbf{u}),              & \text{if } n\neq m
\end{cases}\]\label{lem:grad_Q_bd}
\end{lemma}

\subsection{Broader Impacts}

  The work may positively impact online matching and recommendation systems by reducing computational burden. Since the paper is primarily theoretical and uses only synthetic experiments, we do not identify direct negative societal impacts.

\newpage
\section*{NeurIPS Paper Checklist}

\begin{enumerate}

\item {\bf Claims}
    \item[] Question: Do the main claims made in the abstract and introduction accurately reflect the paper's contributions and scope?
     \item[] Answer: \answerYes{}.
    \item[] Justification: The abstract and introduction state the main contributions: the stochastic matching bandit framework, rare optimization updates, and \(\widetilde{\mathcal O}(\sqrt T)\)-type regret guarantees. 


\item {\bf Limitations}
    \item[] Question: Does the paper discuss the limitations of the work performed by the authors?
     \item[] Answer: \answerYes{}.
    \item[] Justification: The paper explicitly discusses the remaining \(\sqrt K\) gap between the regret upper bound and the known lower bound, and identifies closing this gap---or proving it inherent under rare optimization updates---as an important future direction.


\item {\bf Theory assumptions and proofs}
    \item[] Question: For each theoretical result, does the paper provide the full set of assumptions and a complete (and correct) proof?
     \item[] Answer: \answerYes{}.
    \item[] Justification: The assumptions are stated in the problem formulation and theorem statements. Full proofs of the main regret and batch-update guarantees, as well as supporting lemmas, are provided in the appendix.


    \item {\bf Experimental result reproducibility}
    \item[] Question: Does the paper fully disclose all the information needed to reproduce the main experimental results of the paper to the extent that it affects the main claims and/or conclusions of the paper (regardless of whether the code and data are provided or not)?
    \item[] Answer: \answerYes{}.
    \item[] Justification: The experiments use synthetic data, and the paper specifies the data-generation process, problem dimensions, horizon, baselines, and algorithmic procedures. Baseline adaptation details and the additional scalability experiment are provided in the appendix.


\item {\bf Open access to data and code}
    \item[] Question: Does the paper provide open access to the data and code, with sufficient instructions to faithfully reproduce the main experimental results, as described in supplemental material?
    \item[] Answer: \answerYes{} %
    \item[] Justification: The code is provided in the supplementary material.

\item {\bf Experimental setting/details}
    \item[] Question: Does the paper specify all the training and test details (e.g., data splits, hyperparameters, how they were chosen, type of optimizer) necessary to understand the results?
    \item[] Answer: \answerYes{}.
    \item[] Justification: Appendix~\ref{sec:exp} specifies the synthetic data distribution, values of \(N,K,r,T\), and the compared baselines.


\item {\bf Experiment statistical significance}
    \item[] Question: Does the paper report error bars suitably and correctly defined or other appropriate information about the statistical significance of the experiments?
    \item[] Answer: \answerYes{}.
    \item[] Justification: The figures report variability across independent synthetic runs using error bars. 


\item {\bf Experiments compute resources}
    \item[] Question: For each experiment, does the paper provide sufficient information on the computer resources (type of compute workers, memory, time of execution) needed to reproduce the experiments?
    \item[] Answer: \answerNo{} %
    \item[] Justification: The work is primarily theoretical and the experiments are simple synthetic simulations that do not require GPUs. 
    
\item {\bf Code of ethics}
    \item[] Question: Does the research conducted in the paper conform, in every respect, with the NeurIPS Code of Ethics \url{https://neurips.cc/public/EthicsGuidelines}?
    \item[] Answer: \answerYes{}.
    \item[] Justification: The work is theoretical and uses only synthetic experiments. It does not involve private data, human subjects, or high-risk model deployment, and it preserves submission anonymity.


\item {\bf Broader impacts}
    \item[] Question: Does the paper discuss both potential positive societal impacts and negative societal impacts of the work performed?
    \item[] Answer: \answerYes{}.
      \item[] Justification: The work may positively impact online matching and recommendation systems by reducing computational burden. Since the paper is primarily theoretical and uses only synthetic experiments, we do not identify direct negative societal impacts (see Appendix A.14).

    
\item {\bf Safeguards}
    \item[] Question: Does the paper describe safeguards that have been put in place for responsible release of data or models that have a high risk for misuse (e.g., pre-trained language models, image generators, or scraped datasets)?
    \item[] Answer: \answerNA{}.
    \item[] Justification: The paper does not release high-risk models, scraped datasets, or generative systems. The experiments use synthetic data.
\item

\item {\bf Licenses for existing assets}
    \item[] Question: Are the creators or original owners of assets (e.g., code, data, models), used in the paper, properly credited and are the license and terms of use explicitly mentioned and properly respected?
      \item[] Answer: \answerNA{}.
    \item[] Justification: The paper does not use external datasets, pretrained models, or existing code assets. Prior algorithms and related works are properly cited.


\item {\bf New assets}
    \item[] Question: Are new assets introduced in the paper well documented and is the documentation provided alongside the assets?
    \item[] Answer: \answerNA{}.
    \item[] Justification: The paper does not introduce or release a new dataset, benchmark, or model asset. The experiments are based on synthetic data generated according to the procedure described in the paper.


\item {\bf Crowdsourcing and research with human subjects}
    \item[] Question: For crowdsourcing experiments and research with human subjects, does the paper include the full text of instructions given to participants and screenshots, if applicable, as well as details about compensation (if any)? 
      \item[] Answer: \answerNA{}.
    \item[] Justification: The paper does not involve crowdsourcing or research with human subjects.


\item {\bf Institutional review board (IRB) approvals or equivalent for research with human subjects}
    \item[] Question: Does the paper describe potential risks incurred by study participants, whether such risks were disclosed to the subjects, and whether Institutional Review Board (IRB) approvals (or an equivalent approval/review based on the requirements of your country or institution) were obtained?
    \item[] Answer: \answerNA{}.
    \item[] Justification: The paper does not involve human subjects or crowdsourcing experiments, so IRB approval is not applicable.


\item {\bf Declaration of LLM usage}
    \item[] Question: Does the paper describe the usage of LLMs if it is an important, original, or non-standard component of the core methods in this research? Note that if the LLM is used only for writing, editing, or formatting purposes and does \emph{not} impact the core methodology, scientific rigor, or originality of the research, declaration is not required.
       \item[] Answer: \answerNA{}.
    \item[] Justification: The core methodology, theory, algorithms, and experiments do not use LLMs as an important, original, or non-standard component.


\end{enumerate}

\end{document}